\pdfoutput=1
\documentclass[11pt]{article}

\usepackage[a4paper,margin=1in]{geometry}
\usepackage{natbib}
\usepackage{authblk}

\usepackage[utf8]{inputenc}
\usepackage[T1]{fontenc}
\usepackage{microtype}
\usepackage{graphicx}
\usepackage{subcaption}
\usepackage{booktabs}
\usepackage{hyperref}
\hypersetup{
  colorlinks=true,
  linkcolor=blue!55!black,
  citecolor=blue!55!black,
  urlcolor=blue!55!black
}
\usepackage{url}
\usepackage{xcolor}
\usepackage{algorithm}
\usepackage{algorithmic}

\usepackage{amsmath}
\usepackage{amssymb}
\usepackage{amsfonts}
\usepackage{nicefrac}
\usepackage{mathtools}
\usepackage{amsthm}
\usepackage{tikz}

\usepackage[capitalize,noabbrev]{cleveref}

\theoremstyle{plain}
\newtheorem{theorem}{Theorem}[section]

\newtheorem{lemma}{Lemma}[section]

\theoremstyle{definition}

\theoremstyle{remark}

\crefname{assumption}{Assumption}{Assumptions}
\crefname{figure}{Fig{.}}{Figs{.}}
\crefname{table}{Tab{.}}{Tabs{.}}
\crefname{definition}{Definition}{Definitions}
\crefname{theorem}{Theorem}{Theorems}
\crefname{lemma}{Lemma}{Lemmas}
\crefname{proposition}{Proposition}{Propositions}
\crefname{corollary}{Corollary}{Corollaries}
\crefname{problem}{Problem}{Problems}
\crefname{example}{Example}{Examples}
\crefname{fact}{Fact}{Facts}
\crefname{conjecture}{Conjecture}{Conjectures}
\crefname{remark}{Remark}{Remarks}
\crefname{condition}{Condition}{Conditions}
\crefname{requirement}{Requirement}{Requirements}
\crefname{enumi}{}{}
\crefname{equation}{Eq{.}}{Eqs{.}}
\crefformat{section}{\S#2#1#3}
\crefmultiformat{section}{\S#2#1#3}{ and~\S#2#1#3}{, \S#2#1#3}{, and~\S#2#1#3}

\newcommand*{\twoimagewidth}{0.9\textwidth}

\newcommand*{\R}{\mathbb{R}}

\newenvironment{ack}{\section*{Acknowledgements}}{}

\title{Implicit Neural Representations \\ for Variational Problems on Graphons}

\author[1,2]{Taeyoung~Kim}
\author[3]{Jineon~Baek}
\author[2]{Joonkyung~Lee}
\author[1,2]{Hongseok~Yang}
\affil[1]{School of Computing, KAIST}
\affil[2]{School of Computational Sciences, KIAS}
\affil[3]{June E Huh Center for Mathematical Challenges, KIAS}
\affil[ ]{\newline\texttt{taeyoungkim21@kaist.ac.kr}, \texttt{jineon.kias@gmail.com}, \texttt{joonkyunglee@kias.re.kr}, \texttt{hongseokyang@kias.re.kr}}
\date{}

\begin{document}

\maketitle

\allowdisplaybreaks

\begin{abstract}
We propose a neural exploratory framework for \emph{variational problems over graphons}\footnote{A variational problem is an optimisation problem over a function space, in which the objective and constraints are expressed by means of functionals, i.e., mappings from the function space to $\mathbb{R}$.}, i.e., symmetric measurable functions $[0,1]^2 \to [0,1]$, that arise as limits of dense graph sequences. Such problems are central to two parts of the mathematical literature: extremal graph theory, which studies graph parameter optimisation under given restrictions, often for those graphs with a large number of vertices, and the large deviation theory of dense random graphs, which characterises the structure of rare events through constrained graphon optimisation. We represent graphons by implicit neural networks and optimise graphon objectives by gradient descent. Our design combines three ingredients: a multi-scale sinusoidal residual architecture biased toward sharp, step-like graphons; an embedded solver that enforces a single empirical density constraint and is differentiated by implicit differentiation; and symmetry-aware Monte Carlo estimators. On generalised Tur\'an problems that previously required substantial human effort, our framework rediscovers known optimal graphons without human intervention. Applied to open instances, it produces candidate optima, including a previously unreported family of extremal structures. Applied to the variational problems from the large deviation theory of Erd\H{o}s--R\'enyi random graphs, the framework produces new candidate optimal graphons across both upper- and lower-tail regimes for multiple pattern graphs; for the case of the triangle graph and the upper-tail regime, these candidates improve upon the best-known reference construction of Lubetzky and Zhao. The structures found by our framework have also led to new theorems on the optimal graphons for both families of problems.
\end{abstract}

\section{Introduction} \label{sec:introduction}

Machine learning is opening new directions in mathematical discovery. Recent work spans automatic provers and autoformalisation for Olympiad-style or research-level problems~\citep{polu2023formal, trinh2024solving, jiang2023draft, wu2022autoformalization, hubert2026olympiad, huang2025winning,alexeev2026primitive}, neural search for combinatorial constructions~\citep{wagner2021constructions, charton2024patternboost, romera2024mathematical, novikov2025alphaevolve}, and pattern-driven progress on long-standing conjectures in pure mathematics~\citep{davies2021advancing, he2025murmurations}. In this paper, we extend this line of research to a class of \emph{continuous optimisation problems over graphons}, with two motivating problem families drawn from extremal graph theory and the large deviation theory of dense random graphs.

A central object in extremal graph theory is the \emph{homomorphism density} $t(H, G)$ between finite undirected simple graphs $H$ and $G$: the probability that a uniformly random map from $V(H)$ to $V(G)$ is a graph homomorphism. As a basic example, the constrained optimisation problem
\begin{equation}
  \max_G\, t(K_2, G)
  \quad\text{ subject to }\quad
  t(K_3, G) = 0 \;\land\; |V(G)| = n
  \label{eqn:mantel}
\end{equation}
asks for an $n$-vertex triangle-free graph that maximises the edge density. Here $K_m$ denotes the complete graph on $m$ vertices, so
$K_2$ and $K_3$ represent an edge and a triangle, respectively. Although the answer is given by the classical theorem of Mantel~\citep{mantel1907vraagstuk}, a large family of similar problems involving homomorphism densities remains open~\citep{sidorenko1993correlation, borgs2006counting}.

A standard approach to such problems is to study their \emph{asymptotic} counterparts, taking $n \to \infty$ and analysing the limit of the optimum. These asymptotic problems can often be reformulated as continuous optimisation over \emph{graphons}, i.e., symmetric measurable functions $[0,1]^2 \to [0,1]$, that arise as limit objects of dense graph sequences~\citep{lovasz2006limits, lovasz2012large}. Another source of graphon optimisation problems comes from the large deviation theory of dense random graphs~\citep{chatterjee2011large}, where the limiting structure of an Erd\H{o}s--R\'enyi graph conditioned on an unusually large or small homomorphism density is itself characterised by a constrained optimisation over graphons. In both settings, the objective and constraints are integrals against the graphon $W$ (homomorphism densities or pointwise relative entropy), and known or conjectured optima are step-like.

We attack these graphon-optimisation problems directly with deep learning, representing a graphon $W$ as a neural network in the spirit of implicit neural representations~\citep{park2019deepsdf,mescheder2019occupancy,chen2019learning} and minimising the objective by gradient descent on the network parameters. Three recurring challenges shape our design. First, known and conjectured optima are typically step-like, and standard architectures struggle with such targets because of their bias toward smooth functions. We address this with a \emph{multi-scale sinusoidal residual architecture} that injects sinusoidal encodings of growing frequency into successive layers, biasing the model toward sharp, step-like solutions. Second, the problems often carry a single scalar density constraint (a fixed homomorphism density) that is poorly handled by penalty terms; we embed a numerical solver for an empirical version of the constraint into the forward pass and backpropagate through it via implicit differentiation, as in deep equilibrium models~\citep{bai2019deep}. Third, both the objective and the constraints are graphon integrals that must be estimated from samples, including the high-dimensional homomorphism-density integrals; we use symmetry-aware Monte Carlo estimators.

We demonstrate our framework on two families of graphon optimisation problems: \textbf{(P1)} generalised Tur\'an problems from extremal graph theory, minimising a homomorphism density under a fixed edge-density constraint; and \textbf{(P2)} the variational problems from the large deviation theory of dense random graphs \citep{chatterjee2011large}, minimising a relative-entropy functional under a fixed homomorphism-density constraint.
On \textbf{(P1)}, our framework rediscovers nontrivial known optima that previously required substantial human effort~\citep{goodman1959sets,bollobas1976relations,lovasz1983number,fisher1989lower,razborov2008minimal}, produces candidate minimisers for open instances such as minimisation of $C_7$-homomorphism density, and through a systematic sweep over small graphs, uncovers, for a six-vertex graph, a previously unreported family of extremal structures that do not match the complete $k$-partite pattern of classical examples.
Applied to \textbf{(P2)}, the framework produces new candidate optimal graphons, in the cases where the optimal structure is not fully understood, for four pattern graphs in the upper-tail regime and $K_3$ in the lower-tail regime. For $K_3$, these candidates consistently improve upon the reference construction of \citet{lubetzky2015replica} in the upper-tail regime.
We view the framework as a flexible exploratory tool that complements rigorous lower-bound techniques such as flag algebras~\citep{razborov2007flag,falgas2013applications}: it produces candidate upper bounds, concrete candidate minimisers, and structural conjectures that can be inspected, refined, or subsequently certified. The structures found by our method have also led to new theorems (\Cref{sec:proven}). For problems of \textbf{(P1)}, the balanced complete $k$-partite graphon minimises the density of every odd cycle at $p = 1 - 1/k$ for every $k \ge 3$, and that of every connected chordal graph whose maximal cliques have a common size $r \ge 3$ at $p = 1 - 1/k$ for every $k \ge r$. For problems of \textbf{(P2)}, close to the boundary of the region where the constant graphon is optimal, the upper-tail problem for regular graphs has a unique optimal graphon, which is bipodal.

\section{Background} \label{sec:background}

For a positive integer $n$, write $[n] \coloneqq \{1,2,\ldots,n\}$. Throughout the paper, a \emph{graph} is an undirected, finite, simple graph (no self-loops and no multiple edges). For a graph $G$, we write $V(G)$ and $E(G)$ for its vertex and edge sets, and let $v(G) \coloneqq |V(G)|$ and $e(G) \coloneqq |E(G)|$. We use standard notation for familiar graphs: $K_n$ denotes the complete graph on $n$ vertices (i.e., $V(K_n) = [n]$ and $E(K_n) = \{\{i,j\} : i, j \in [n],\, i \neq j\}$), $K_{1,n}$ denotes the complete bipartite graph with part sizes $1$ and $n$ (i.e., $V(K_{1,n}) = [n+1]$ and $E(K_{1,n}) = \{\{1,i\} : 2\le i\le n+1\}$), and $C_n$ denotes the cycle on $n$ vertices (i.e., $V(C_n) = [n]$ and $E(C_n) = \{\{i,i+1\} : 1\le i\le n-1\} \cup \{\{n,1\}\}$).

Let $H$ and $G$ be graphs. A \emph{homomorphism} from $H$ to $G$ is a map $\phi : V(H) \to V(G)$ that preserves adjacency, i.e., $\{\phi(u), \phi(v)\} \in E(G)$ for every $\{u,v\} \in E(H)$. The \emph{homomorphism density} from $H$ to $G$, denoted $t(H, G)$, is the probability that a uniformly random map $\phi : V(H) \to V(G)$ is a homomorphism. Equivalently,
\begin{align}
t(H, G) \,\coloneqq\, {}
\label{eqn:homomorphismdensity}
\frac{1}{v(G)^{v(H)}}
\sum_{\phi : V(H) \to V(G)}
\prod_{\{u,v\} \in E(H)} \mathbb{I}\bigl[\{\phi(u), \phi(v)\} \in E(G)\bigr],
\end{align}
where $\mathbb{I}[\cdot]$ is the indicator function.
Homomorphism densities are a natural language for extremal graph problems: they are size-normalised statistics of fixed patterns, so objectives and constraints are comparable across graphs of different sizes, and they are central to dense graph limit theory, quasirandomness, and regularity-based approximation~\citep{lovasz2012large}.

A \emph{graphon} $W$ is a symmetric measurable function $W : [0,1]^2 \to [0,1]$. Intuitively, $W$ can be viewed as the (probabilistic) adjacency matrix of a continuum graph with vertex set $[0,1]$: $W(x,y)$ is the probability of an edge between $x$ and $y$. Graphons arise as limits of dense graphs: along any graph sequence $(G_n)_{n \in \mathbb{N}}$ whose homomorphism densities $t(H, G_n)$ converge for every fixed graph $H$ with vertex set $[v(H)]$, there is a graphon $W$ with $t(H, G_n) \to t(H, W)$ for every $H$ with vertex set $[v(H)]$, where
$t(H,W)$ is the graphon homomorphism density of $H$ in $W$ defined by
\begin{align}
t(H, W) \,\coloneqq\, \mathbb{E}_{x_1, \ldots, x_{v(H)} \stackrel{\mathrm{iid}}{\sim} \mathrm{Unif}[0,1]}
\Big[ \prod_{\{i,j\} \in E(H)} W(x_i, x_j)\Big].
\label{eqn:homomorphismdensitygraphon}
\end{align}
Conversely, every graphon arises in this way~\citep{lovasz2006limits,lovasz2012large}. Familiar discrete structure is reflected in $W$: balanced complete bipartite graphs converge to the two-block graphon equal to $1$ on the off-diagonal blocks and $0$ on the diagonal blocks; an Erd\H{o}s--R\'enyi $G(n,q)$ converges to the constant graphon $W \equiv q$; a stochastic block model converges to a multi-block graphon whose blocks correspond to the communities.

Graphons provide a common language for both problem families studied in this paper: the objective and constraints in each involve graphon homomorphism densities $t(H,W)$ or closely related graphon integrals, and the structure of the optimal graphons is the central object of interest. The first family arises from extremal graph theory, where optimising over graphons is equivalent to optimising the corresponding asymptotic finite-graph problem~\citep{lovasz2012large}; the second arises from the large deviation theory of dense random graphs. We describe each in turn.

\noindent{\bf Constrained density minimisation.}\ $\,$
The first problem family is concerned with minimising the homomorphism density of a fixed graph $H$ subject to a fixed edge density $p$. To state the asymptotic problem formally, we phrase it in terms of graph sequences:
\begin{equation}
\label{eqn:extremalproblem:generalisedTuran}
\inf\Big\{\liminf_{n \to \infty} t(H, G_n) \,\Big|\, \text{$(G_n)$ is a graph sequence with } v(G_n)\to\infty \text{ and } t(K_2, G_n)\to p\Big\}.
\end{equation}
By the graph-limit theory~\citep{lovasz2012large}, this equals $\min_{W:\,t(K_2,W)=p}\, t(H, W)$, the graphon problem our method optimises directly.
It generalises Tur\'an's theorem~\citep{turan1941eine}
and the Erd\H{o}s--Stone--Simonovits theorem~\citep{erdos1946structure, erdos1966limit}. These classical results in extremal graph theory say
that for every $H$ with at least one edge and $p \leq 1 - 1/(\chi(H) - 1)$, where $\chi(H)$ is the chromatic number of $H$, the optimum is $0$. But for $p > 1 - 1/(\chi(H) - 1)$, the problem is much harder. It is completely solved only for cliques: $H = K_r$ for all $r \ge 3$~\citep{reiher2016clique}, with earlier results for $r = 3$~\citep{razborov2008minimal} and $r = 4$~\citep{nikiforov2011number}. For $C_5$, the minimiser is known only at the special values $p = 1 - 1/k$ ($k$ a positive integer)~\citep{bennett2020minimizing}.

\noindent{\bf Variational problem from large deviation theory.}\ $\,$
The large deviation theory of dense random graphs addresses two related questions: at what exponential rate does the Erd\H{o}s--R\'enyi random graph $G(n,q)$ have an atypically large or small homomorphism density, and what does the graph typically look like when such a rare event occurs?
These questions are of fundamental interest in probability theory and combinatorics, and they carry information-theoretic structure that is natural from an ML perspective: the Chatterjee--Varadhan framework reduces both questions to a single constrained optimisation over graphons whose objective is an averaged KL divergence~\citep{chatterjee2011large}.

Concretely, fix $q \in (0,1)$ and let $h_q(W) \coloneqq \int_{[0,1]^2} \mathrm{KL}(\mathrm{Bern}(W(x,y))\,\|\,\mathrm{Bern}(q))\,\mathrm{d}x\,\mathrm{d}y$, the integrated per-edge KL divergence of $W$ from the base Erd\H{o}s--R\'enyi graphon $W \equiv q$.
For a fixed graph $H$ and target level $r \in (0,1)$, the large deviation principle yields a variational characterisation of both tails.
In the \emph{upper tail} ($r > q$), $\binom{n}{2}^{-1}\log\Pr[t(H, G(n,q)) \ge r^{e(H)}]$ converges to $-\min_{W: t(H,W) \ge r^{e(H)}} h_q(W)$.
In the \emph{lower tail} ($r < q$), $\binom{n}{2}^{-1}\log\Pr[t(H, G(n,q)) \le r^{e(H)}]$ converges to $-\min_{W: t(H,W) \le r^{e(H)}} h_q(W)$.
In either case, the typical structure of $G(n,q)$ conditioned on the rare event is described by near-minimisers of the corresponding variational problem~\citep{lubetzky2015replica}.
Thus the minimiser is the graphon closest in KL divergence to $W \equiv q$ that explains the observed \emph{excess} (upper tail) or \emph{deficit} (lower tail) in the $H$-homomorphism density, in the same spirit as maximum-entropy modelling in ML\@.
Even in the triangle case $H = K_3$, it is not fully understood when the deviation localises (e.g. through a dense clique) versus spreads more uniformly, or what the minimisers look like across $(q,r)$.

Our method applies directly to the boundary form
\begin{equation}
\label{eqn:large-deviation-problem}
  \min_{W : t(H,W) = r^{e(H)}} h_q(W)
\end{equation}
which is equivalent to the inequality-constrained minimisations above (with $\ge$ in the upper tail $r>q$ and $\le$ in the lower tail $r<q$).

\section{Implicit neural representations of graphons} \label{sec:method}

We tackle the two graphon optimisation problems introduced in \Cref{sec:background}:
\begin{equation}
\textbf{(P1)} \min_{W : t(K_2, W) = p}  t(H, W),
\qquad\qquad
\textbf{(P2)} \min_{W : t(H, W) = r^{e(H)}} h_q(W),
\label{eqn:targetproblems}
\end{equation}
for a graph $H$, reals $p,q,r \in (0,1)$, where $W$ ranges over graphons. Both problems carry a single scalar equality constraint.

We represent $W$ as a neural network $f_\theta$, recast these graphon optimisations as optimisations over the parameters $\theta$, and apply gradient-based training. A naive instantiation of this idea fails for three reasons. First, known optimal graphons are typically step functions, which standard architectures such as MLPs struggle to represent due to their bias toward smooth functions. Second, the scalar constraints in \textbf{(P1)} and \textbf{(P2)} are poorly handled by penalty terms, so the architecture must enforce them more directly. Third, each occurrence of $t(\cdot, W)$ involves an intractable integral, and estimating these expectations efficiently requires problem-specific sampling. The remainder of this section describes the architecture and training procedure that address these three challenges.

\noindent{\bf Architecture.}\ $\,$
We parameterise graphons by a residual MLP that injects a layer-specific sinusoidal encoding of the input at every layer. To handle constraints in \textbf{(P1)} and \textbf{(P2)}, the architecture additionally includes a constraint solver.

Fix the number of layers $L \ge 2$, the hidden-layer width $d$, and a differentiable activation $\phi : \R \to \R$. Also, fix an increasing function $s : [L] \to (0, \infty)$, which determines the scale of weight initialisation at each layer; in our experiments, we use $s(\ell) = 2^{\ell-1}$. The network $f_\theta[c] : [0,1]^2 \to [0,1]$ has a distinguished scalar bias $c \in \R$ and remaining parameters $\theta$, and is defined for $2 \le \ell \le L$ by
\begin{align*}
f_{\theta}[c](x_1, x_2) & \,\coloneqq\, \mathit{sm}\bigl(V g^{(L)}_\theta(\min(x_1, x_2), \max(x_1, x_2)) \,+\, c\bigr),
\\
g^{(1)}_\theta(\mathbf{x}) & \,\coloneqq\, \sin\bigl(U^{(1)} \mathbf{x} \,+\, b^{(1)}\bigr),
\\
g^{(\ell)}_\theta(\mathbf{x}) & \,\coloneqq\, g_\theta^{(\ell-1)}(\mathbf{x}) \,+\, \phi\bigl(V^{(\ell-1)} g_\theta^{(\ell-1)}(\mathbf{x}) \,+\, W^{(\ell)} \sin(U^{(\ell)} \mathbf{x}) \,+\, b^{(\ell)}\bigr).
\end{align*}
Here $\mathit{sm}$ is the sigmoid, and $g^{(1)}_\theta, g^{(\ell)}_\theta$ are maps from $\{\mathbf{x} = (x_1, x_2) \in [0,1]^2 : x_1 \le x_2\}$ to $\R^d$. The network parameters are
$c \in \R$, $b^{(\ell')} \in \R^d$, $V \in \R^{1 \times d}$, $U^{(\ell')} \in \R^{d \times 2}$, and $V^{(\ell'')}, W^{(\ell')} \in \R^{d \times d}$, for $\ell' \in [L]$, $\ell'' \in [L-1]$. These parameters are initialised independently by
\begin{align*}
c,\, b^{(\ell')}_i,\, V_{1,i},\, V^{(\ell'')}_{i,j},\, W^{(\ell')}_{i,j}
&\,\stackrel{\mathrm{iid}}{\sim}\, \mathrm{Unif}\bigl[-1/\sqrt{2d},\, 1/\sqrt{2d}\bigr],
&
\quad
U^{(\ell')}_{i,k}
&\,\stackrel{\mathrm{iid}}{\sim}\, \mathrm{Unif}\bigl[-s(\ell')/\sqrt{2},\, s(\ell')/\sqrt{2}\bigr],
\end{align*}
for all $i, j \in [d]$, $k \in [2]$. We collect every parameter except $c$ into $\theta$. Using $\min$ and $\max$ in the definition of $f_\theta[c]$ ensures that $f_\theta[c]$ is symmetric and represents a graphon.

The architecture has two notable features. First, it repeatedly injects a sinusoidal encoding of the input $(x_1, x_2)$ at every layer, with an encoding scale that grows with depth. Increasing $s(\ell)$ with depth shifts higher-frequency content into deeper layers~\citep{tancik2020fourier} and encourages late layers to capture sharp boundaries, while early layers stabilise the coarse, low-frequency structure, an inductive bias that fits step-like graphons.

Second, in constrained problems, we determine $c$ via a numerical solver rather than learning it directly. To enforce $t(K_2, f_\theta[c]) = p$ in \textbf{(P1)}, we use a batch of $K$ i.i.d.\ samples
\[
\mathcal{E} = \bigl\{(x_1^{(k)}, x_2^{(k)})\bigr\}_{k \in [K]},
\quad
(x_1^{(k)}, x_2^{(k)}) \stackrel{\mathrm{iid}}{\sim} \mathrm{Unif}([0,1]^2),
\]
and define
\begin{align}
\textsc{Solve}_\theta(\mathcal{E}, p)
&\,\coloneqq\, \text{the unique } c \in \R \text{ such that }
\frac{1}{|\mathcal{E}|} \sum_{(x_1, x_2) \in \mathcal{E}} f_\theta[c](x_1, x_2) \,=\, p.
\label{eqn:solver-specification}
\end{align}
The left-hand side of the equation in \eqref{eqn:solver-specification} is an unbiased estimator of $t(K_2, f_\theta[c])$, so the equation is the empirical counterpart of the constraint. The solution exists and is unique, since the left-hand side is continuous and strictly increasing in $c$ and tends to $0$ and $1$ as $c \to -\infty$ and $c \to +\infty$, respectively. We solve~\eqref{eqn:solver-specification} with a clipped variant
of Newton's method (\Cref{alg:cap}) and use the resulting $c^*$ when evaluating $f_\theta[c^*]$ on other inputs.
The same idea applies to any single homomorphism-density constraint of the form $t(H, W)=\tau$: one replaces the edge-density estimator in~\eqref{eqn:solver-specification} by an estimator of $t(H, f_\theta[c])$. Because increasing the scalar bias $c$ increases every edge probability, this estimator remains monotone in $c$. For the large-deviation application \textbf{(P2)} with $H = K_3$, for instance, this means sampling triples $(x^{(k)}_1,x^{(k)}_2,x^{(k)}_3)$, storing those samples in $\mathcal{T}$, and solving the empirical equation
\[
  \frac{1}{|\mathcal{T}|}
  \sum_{(x_1,x_2,x_3)\in\mathcal{T}}
  f_\theta[c](x_1,x_2)f_\theta[c](x_2,x_3)f_\theta[c](x_3,x_1)
  = r^3.
\]
The same construction applies to any other pattern graph $H$ by replacing the triangle-density estimator with the corresponding $H$-density estimator from~\eqref{eqn:estimator}.

\noindent{\bf Estimating homomorphism density.}\ $\,$
Let $H$ be a graph with vertex set $[v(H)]$. By~\eqref{eqn:homomorphismdensitygraphon}, $t(H, f_\theta[c])$ is an expectation that we estimate by Monte Carlo sampling. Our estimator is parameterised by a set $\mathcal{S}$ of permutations of $[v(H)]$. It draws $(N \cdot v(H))$ i.i.d.\ samples $x_i^{(k)} \stackrel{\mathrm{iid}}{\sim} \mathrm{Unif}[0,1]$ for $(k, i) \in [N] \times [v(H)]$, and then computes
\begin{equation}
\label{eqn:estimator}
  \widehat{t}(H, f_\theta[c])
  \,\coloneqq\,
  \frac{1}{N \, |\mathcal{S}|}
  \sum_{\substack{k \in [N], \pi \in \mathcal{S}}}\,
  \prod_{\{i, j\} \in E(H)}
  f_\theta[c]\bigl(x_{\pi(i)}^{(k)}, x_{\pi(j)}^{(k)}\bigr).
\end{equation}
The estimator is unbiased for any choice of $\mathcal{S}$ (\Cref{lem:unbiased}); the choice of $\mathcal{S}$ controls variance and computational cost.
When $\mathcal{S} = \{\mathrm{id}\}$, our estimator in \eqref{eqn:estimator} is the standard Monte Carlo estimator using $N$ independent samples of $(x_1, \ldots, x_{v(H)})$. With multiple permutations, the same $N$ tuples are reused across the $|\mathcal{S}|$ permuted copies, often producing lower variance for the same number of independent tuples.
Different permutations also induce repeated evaluations of the network $f_\theta[c]$ on the same unordered pair $(x_i^{(k)}, x_j^{(k)})$ across summands, so caching the evaluations reduces the number of network evaluations from $N \cdot |\mathcal{S}| \cdot e(H)$ to $N \cdot C \cdot e(H)$ for some $C \le |\mathcal{S}|$ that depends on $\mathcal{S}$ and $H$. We exploit this for larger graphs (e.g., the Petersen graph in \Cref{sec:addexperiments}).

\noindent{\bf Training algorithm.}\ $\,$
For \textbf{(P1)}, fix a graph $H$ on $[v(H)]$ and $p \in (0, 1)$. At each iteration, we sample $(N \cdot v(H))$ i.i.d.\ uniform values $\{x_i^{(k)}\}_{(i,k) \in [N] \times [v(H)]}$ on $[0, 1]$, and form
$\mathcal{D} = \{(x_1^{(k)}, \ldots, x_{v(H)}^{(k)}) : k \in [N]\}$ and the unordered-pair set $\mathcal{E} = \{(x^{(k)}_i, x^{(k)}_j) : k \in [N],\ 1 \le i < j \le v(H)\}$. Then, we set $c_\theta \coloneqq \textsc{Solve}_\theta(\mathcal{E}, p)$ to enforce the empirical constraint, and update
\[
\theta_{t+1} \,\leftarrow\, \theta_t - \eta_t\, \nabla_\theta \widehat{t}_{\mathcal{D}}(H, f_\theta[c_\theta])\big|_{\theta = \theta_t},
\]
where $\widehat{t}_{\mathcal{D}}$ denotes the estimator in \eqref{eqn:estimator} evaluated on the tuples in $\mathcal{D}$. Reusing the $x_i^{(k)}$ across $\mathcal{D}$ and $\mathcal{E}$ saves forward evaluations of $f_\theta[c]$. We resample $(\mathcal{D}, \mathcal{E})$ at every iteration to mitigate overfitting to particular Monte Carlo samples.
We set $\eta_t$ via a cosine-annealing schedule~\citep{loshchilov2017sgdr}, which alternates between discovering coarse structure (high learning rate) and refining sharp boundaries (low learning rate).

\begin{figure}[t]
  \centering
  \begin{minipage}{0.40\textwidth}
    \centering
    \captionof{table}{$K_3$ densities in graphons found by our method and optimal construction for $K_3$ density minimisation under $t(K_2, W)=p$.\label{tab:exp1table}}
    {\setlength{\tabcolsep}{4pt}\small
    \begin{tabular}{cll}
      \toprule
      $\mathbf{p}$ & \textbf{Optimum} & \textbf{Ours} \\
      \midrule
      $1/2$ & $0$          & $0.00008130$  \\
      $2/3$ & $0.22222222$ & $0.22232803$ ($+0.05\%$) \\
      $4/5$ & $0.48$       & $0.48010290$ ($+0.02\%$) \\
      $4/7$ & $0.10298638$ & $0.10299835$ ($+0.01\%$) \\
      $5/8$ & $0.17361111$ & $0.17362659$ ($+0.01\%$) \\
      $5/7$ & $0.31363508$ & $0.31370832$ ($+0.02\%$) \\
      \bottomrule
    \end{tabular}
    }
  \end{minipage}\hfill
  \begin{minipage}{0.55\textwidth}
    \centering
    \includegraphics[width=0.9\linewidth]{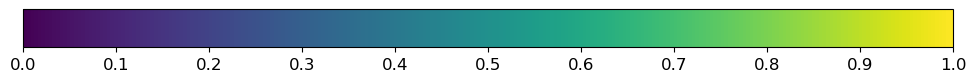}

    \begin{subfigure}{0.30\linewidth}
      \includegraphics[width=\linewidth]{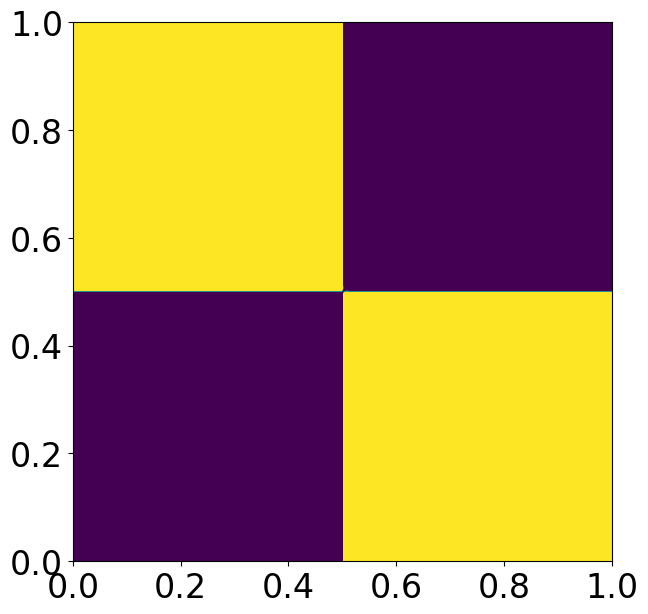}
      \caption{$p{=}1/2$.}\label{fig:k3-50}
    \end{subfigure}\hfill
    \begin{subfigure}{0.30\linewidth}
      \includegraphics[width=\linewidth]{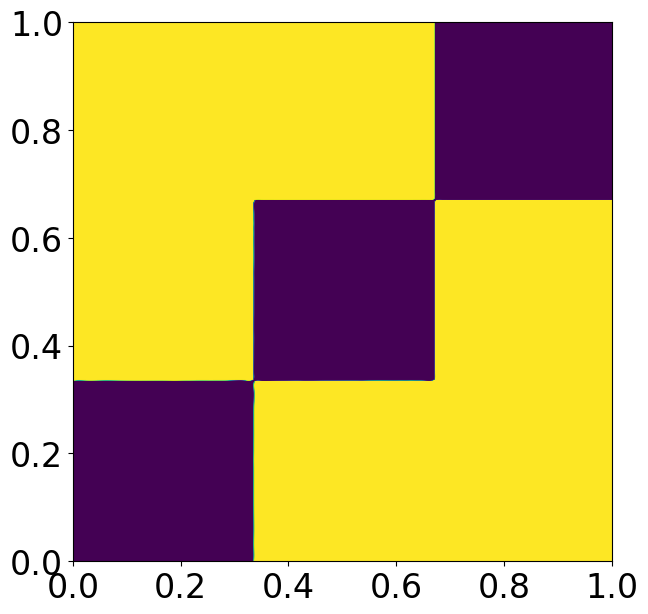}
      \caption{$p{=}2/3$.}\label{fig:k3-67}
    \end{subfigure}\hfill
    \begin{subfigure}{0.30\linewidth}
      \includegraphics[width=\linewidth]{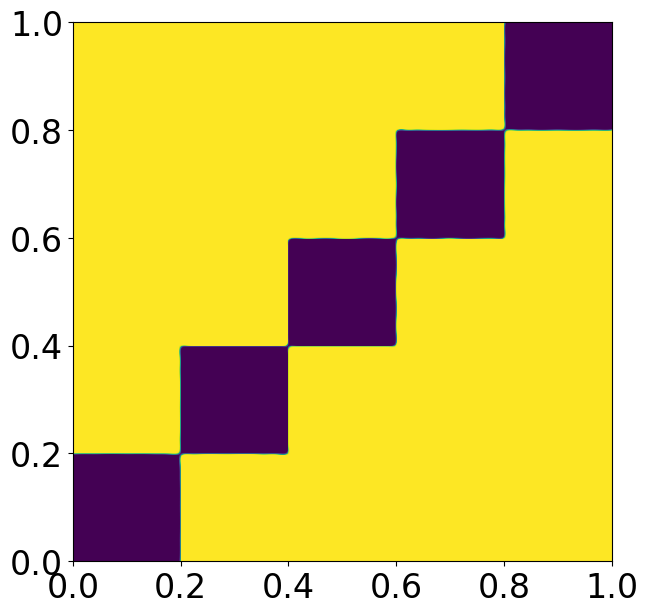}
      \caption{$p{=}4/5$.}\label{fig:k3-80}
    \end{subfigure}\\[0.3em]
    \begin{subfigure}{0.30\linewidth}
      \includegraphics[width=\linewidth]{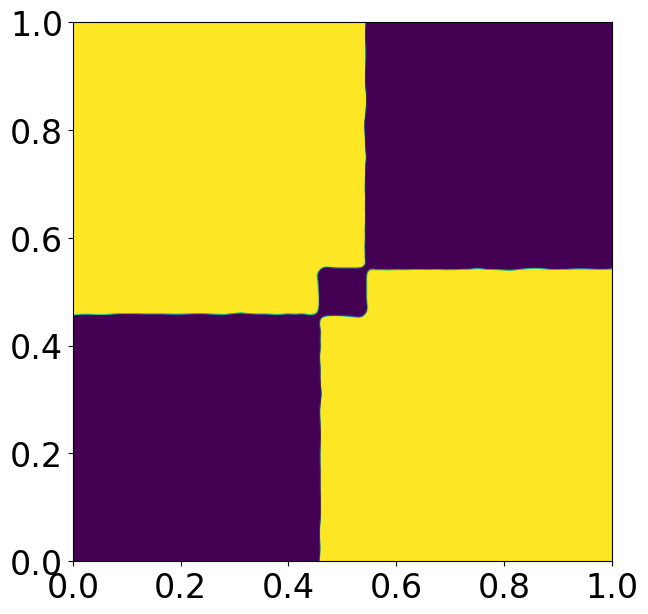}
      \caption{$p{=}4/7$.}\label{fig:k3-57}
    \end{subfigure}\hfill
    \begin{subfigure}{0.30\linewidth}
      \includegraphics[width=\linewidth]{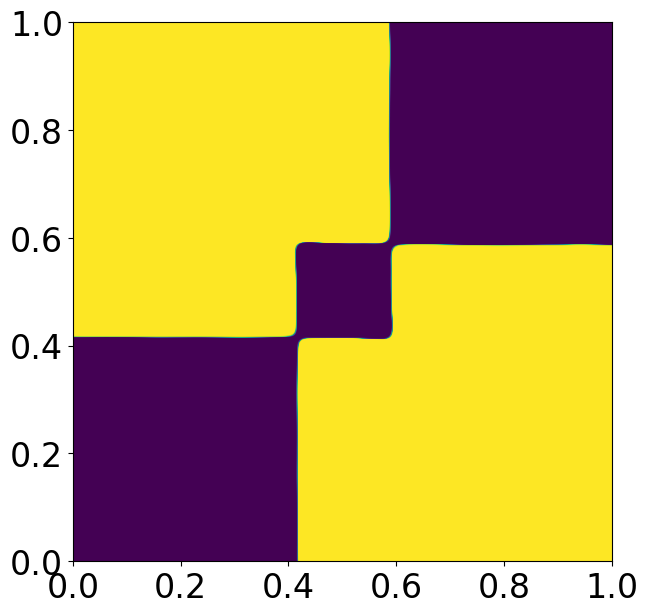}
      \caption{$p{=}5/8$.}\label{fig:k3-62}
    \end{subfigure}\hfill
    \begin{subfigure}{0.30\linewidth}
      \includegraphics[width=\linewidth]{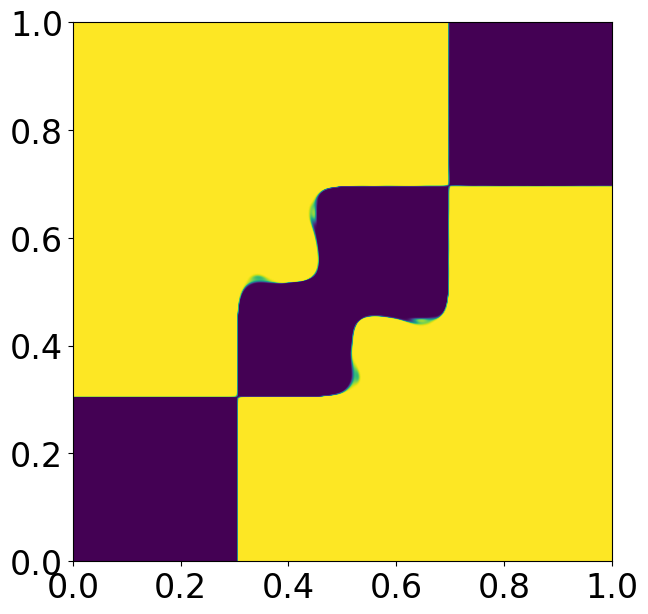}
      \caption{$p{=}5/7$.}\label{fig:k3-71}
    \end{subfigure}
    \captionof{figure}{Graphons found by our method for $K_3$ density minimisation under $t(K_2, W)=p$.}
    \label{fig:triangledensity-minimisers}
  \end{minipage}
\end{figure}

\begin{figure}[t]
  \centering
  \begin{minipage}{0.55\textwidth}
    \centering
    \captionof{table}{$K_k$ densities in graphons found by our method and the optimal construction for $K_k$ density minimisation under $t(K_2, W) {=} 1 {-} 1/k$. \label{tab:exp2table}}
    {\setlength{\tabcolsep}{4pt}\small
    \begin{tabular}{cll}
       \toprule
       $\mathbf{k}$& \textbf{Optimum} & \textbf{Ours (Estimated)} \\
       \midrule
       $k=4$ & $0.09375$    & $(\mathit{0.09403727}, \mathit{0.09411154})$ \\
       $k=5$ & $0.0384$     & $(\mathit{0.03868024}, \mathit{0.03872081})$ \\
       $k=6$ & $0.01543210$ & $(\mathit{0.01565511}, \mathit{0.01568170})$ \\
       \bottomrule
    \end{tabular}
    }
  \end{minipage}\hfill
  \begin{minipage}{0.40\textwidth}
    \centering
    \begin{subfigure}{0.30\linewidth}
      \includegraphics[width=\linewidth]{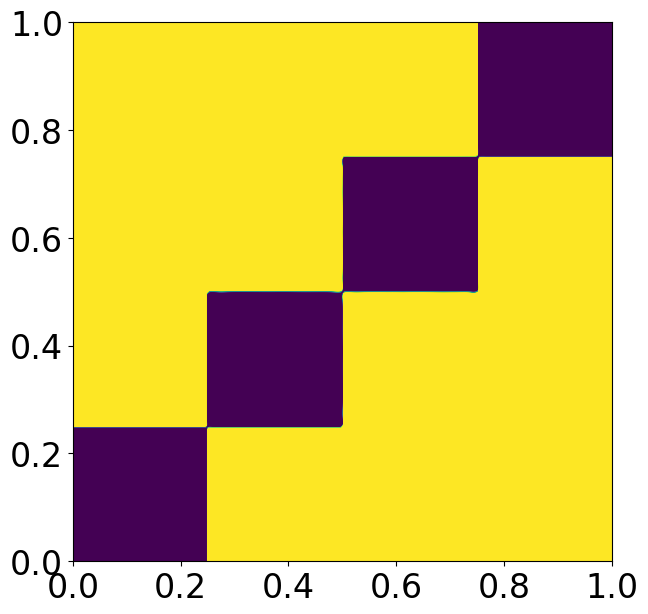}
      \caption{$k{=}4$.}\label{fig:k4-75}
    \end{subfigure}\hfill
    \begin{subfigure}{0.30\linewidth}
      \includegraphics[width=\linewidth]{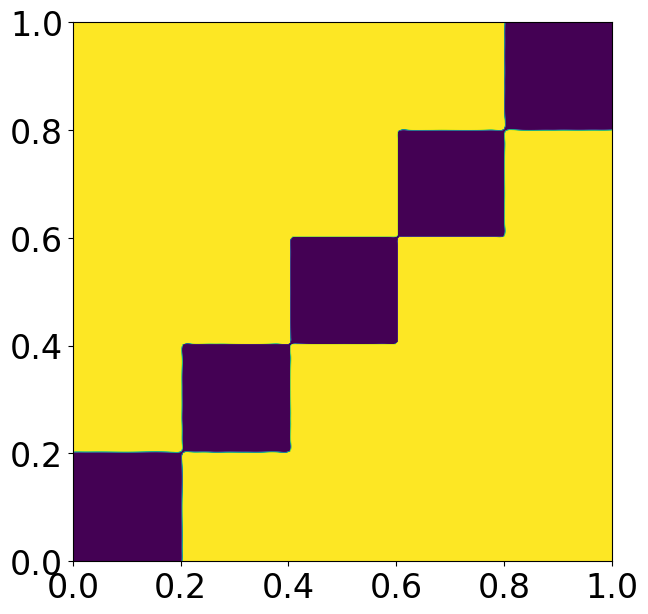}
      \caption{$k{=}5$.}\label{fig:k5-80}
    \end{subfigure}\hfill
    \begin{subfigure}{0.30\linewidth}
      \includegraphics[width=\linewidth]{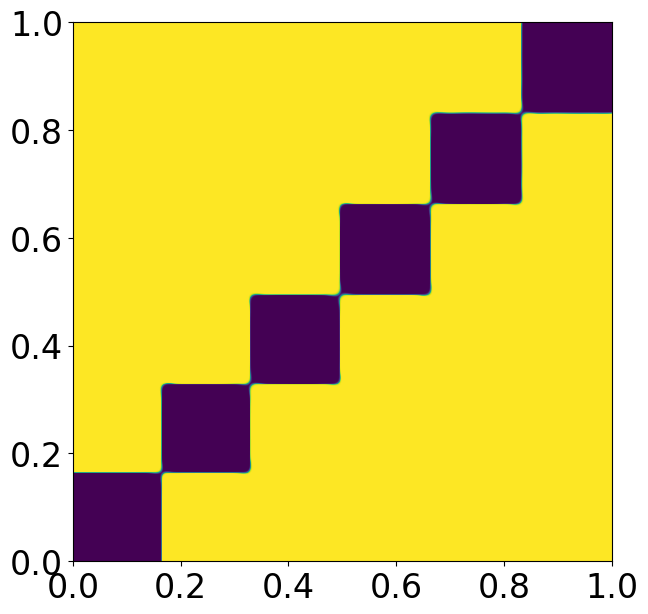}
      \caption{$k{=}6$.}\label{fig:k6-833}
    \end{subfigure}
    \captionof{figure}{Graphons found by our method for $K_k$ density minimisation under $t(K_2, W) = 1 - 1/k$.}
    \label{fig:clique}
  \end{minipage}
\end{figure}

Computing $\nabla_\theta \widehat{t}_{\mathcal{D}}(H, f_\theta[c_\theta])$ requires only the standard product rule once $\nabla_\theta f_\theta[c_\theta]$ is known; the only non-routine ingredient is the gradient of the solver output $c_\theta$ with respect to $\theta$, which we compute via implicit differentiation.
A direct alternative is to differentiate through an unrolled Newton iteration with autodiff, but this requires storing every intermediate iterate and is substantially more expensive in time and memory; the same trade-off motivates implicit-differentiation methods in deep equilibrium models~\citep{bai2019deep} and neural ODEs~\citep{chen2018neural}. The closed-form gradient is given by:

\begin{lemma}\label{lem:implicitdifferentiation}
For all $\mathbf{x} \in [0,1]^2$,
\begin{align*}
\nabla_\theta\, f_\theta[c_\theta](\mathbf{x})
\,&=\, \frac{\partial f_\theta[c](\mathbf{x})}{\partial \theta}\bigg|_{c = c_\theta}
\,-\, \mathit{sm}'\bigl(h_\theta(\mathbf{x}) + c_\theta\bigr)
\sum_{\mathbf{y} \in \mathcal{E}} w_\theta(\mathbf{y})\, \nabla_\theta h_\theta(\mathbf{y}),
\\
\text{where }
w_\theta(\mathbf{y}) \,&\coloneqq\, \frac{\mathit{sm}'(h_\theta(\mathbf{y}) + c_\theta)}{\sum_{\mathbf{z} \in \mathcal{E}} \mathit{sm}'(h_\theta(\mathbf{z}) + c_\theta)},
\quad
h_\theta(x_1, x_2) \,\coloneqq\, V g_\theta^{(L)}\!\bigl(\min(x_1, x_2), \max(x_1, x_2)\bigr).
\end{align*}
\end{lemma}

For \textbf{(P2)}, despite its different mathematical origins in large deviation theory, the same constrained loop applies with two substitutions: for a given pattern graph $H$ and $0<q<1$ with target $r \neq q$, the edge-density solver is replaced by the $H$-density solver targeting $t(H, f_\theta[c]) = r^{e(H)}$, and the objective is replaced by the Monte Carlo estimate of $h_q(f_\theta[c_\theta])=\int_{[0,1]^2} \mathrm{KL}(\mathrm{Bern}(f_\theta[c_\theta](x,y))\,\|\,\mathrm{Bern}(q))\,\mathrm{d}x\,\mathrm{d}y$, obtained by averaging $\mathrm{KL}(\mathrm{Bern}(f_\theta[c_\theta](x,y))\,\|\,\mathrm{Bern}(q))$ over a fresh batch of pairs $(x,y)$. The sigmoid output ensures the KL objective is differentiable in $\theta$, and an analogous implicit-differentiation formula applies to the $H$-density constraint.

\begin{figure}[t]
  \centering
  \begin{minipage}{0.43\textwidth}
    \centering
    \captionof{table}{$C_5$ densities in graphons found by the known or conjectured optimal construction and our method for $C_5$ density minimisation under $t(K_2, W)=p$.\label{tab:c5result}}
    {\setlength{\tabcolsep}{4pt}\small
    \begin{tabular}{cll}
       \toprule
       $\mathbf{p}$
       &
       \begin{tabular}{@{}l@{}}
       \textbf{Known} or \\
       \textbf{Conjectured}
       \end{tabular}
       &
       \textbf{Ours}
       \\
       \midrule
       $1/2$ & $0$ & $0.00004086$ \\
       $2/3$ & $0.12345679$ & $0.12350364$ ($+0.04\%$) \\
       $4/5$ & $0.3264$ & $0.32644549$ ($+0.01\%$) \\
       $4/7$ & $0.04587149$ & $0.04590020$ ($+0.06\%$) \\
       $5/8$ & $0.08653616$ & $0.08657084$ ($+0.04\%$) \\
       $5/7$ & $0.18200809$ & $0.18206173$ ($+0.03\%$) \\
       \bottomrule
    \end{tabular}
    }
  \end{minipage}\hfill
  \begin{minipage}{0.55\textwidth}
    \centering
    \begin{subfigure}{0.30\linewidth}
      \includegraphics[width=\linewidth]{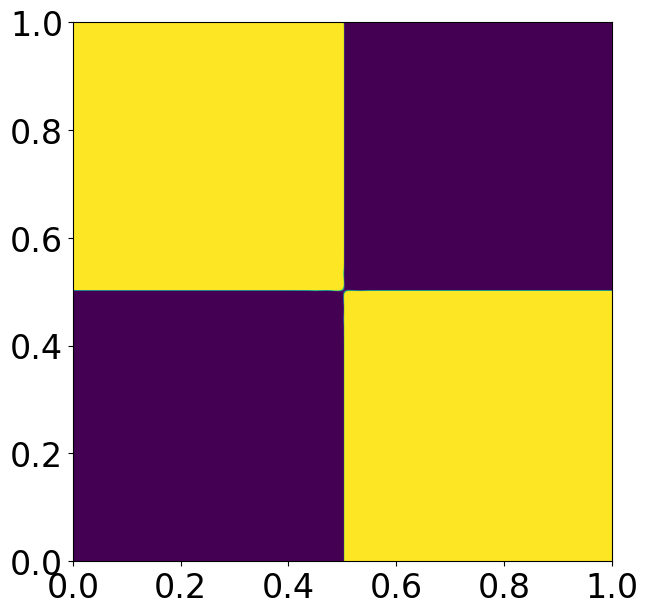}
      \caption{$p{=}1/2$.}\label{fig:opt-c5-50}
    \end{subfigure}\hfill
    \begin{subfigure}{0.30\linewidth}
      \includegraphics[width=\linewidth]{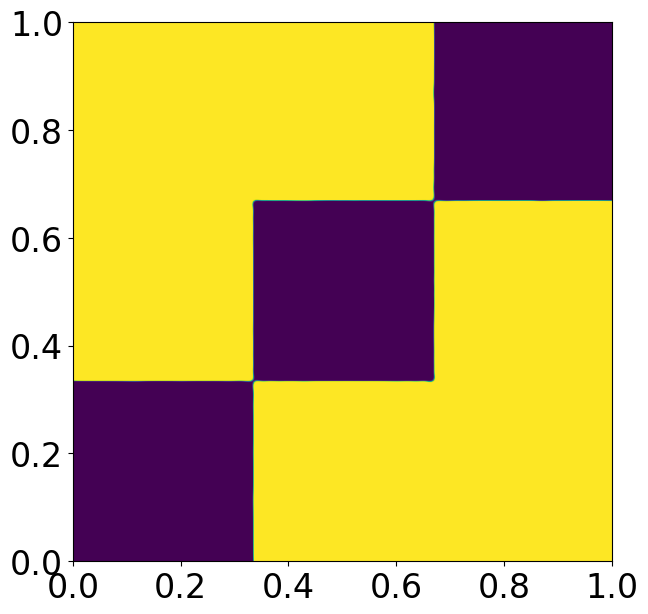}
      \caption{$p{=}2/3$.}\label{fig:opt-c5-67}
    \end{subfigure}\hfill
    \begin{subfigure}{0.30\linewidth}
      \includegraphics[width=\linewidth]{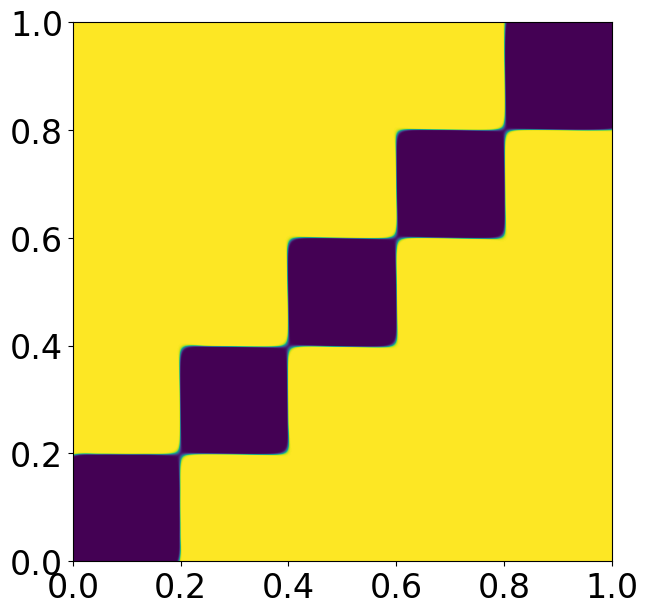}
      \caption{$p{=}4/5$.}\label{fig:opt-c5-80}
    \end{subfigure}\\[0.3em]
    \begin{subfigure}{0.30\linewidth}
      \includegraphics[width=\linewidth]{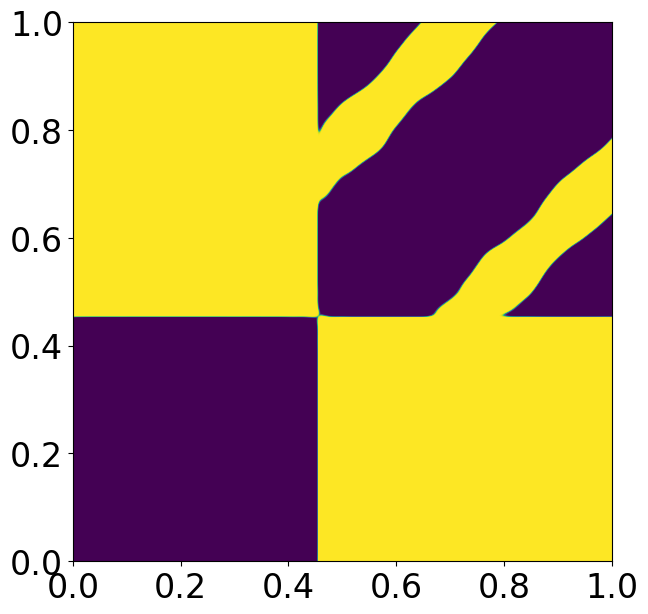}
      \caption{$p{=}4/7$.}\label{fig:opt-c5-57}
    \end{subfigure}\hfill
    \begin{subfigure}{0.30\linewidth}
      \includegraphics[width=\linewidth]{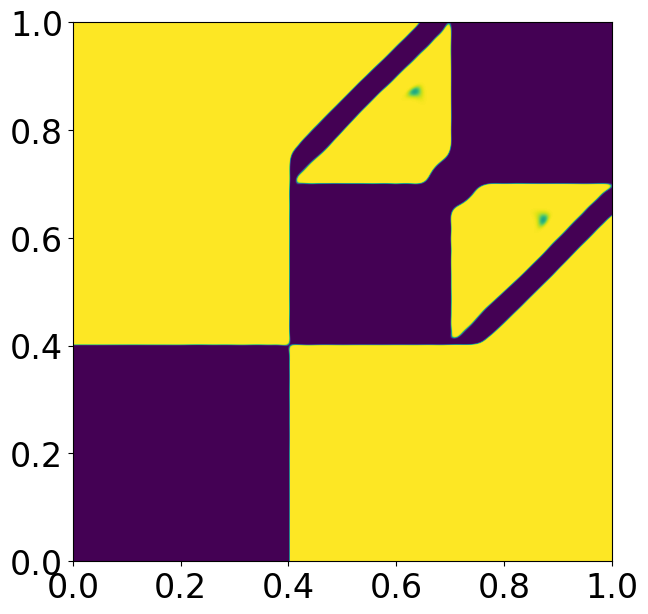}
      \caption{$p{=}5/8$.}\label{fig:opt-c5-62}
    \end{subfigure}\hfill
    \begin{subfigure}{0.30\linewidth}
      \includegraphics[width=\linewidth]{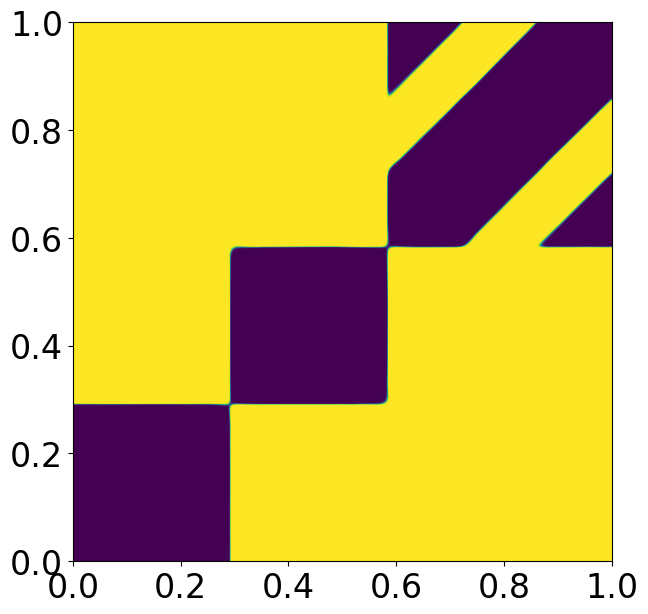}
      \caption{$p{=}5/7$.}\label{fig:opt-c5-71}
    \end{subfigure}
    \captionof{figure}{Graphons found by our method for $C_5$ density minimisation under $t(K_2, W)=p$.}
    \label{fig:5cycleminimisation}
  \end{minipage}
\end{figure}

\section{Experimental results}\label{sec:experiments}

We evaluate our framework on instances of \textbf{(P1)} and \textbf{(P2)} drawn from the extremal-graph-theory and large-deviation literatures. The same architecture, solver, and training loop are used throughout; the only changes between problem families are which constraint is enforced and how the objective is estimated. The experiments are organised to demonstrate two capabilities: (i) rediscovering classical optima that previously required substantial human effort, and (ii) producing mathematically interesting candidates for open instances. An ablation study at the end isolates the contribution of each architectural component.

We visualise graphons as Viridis heatmaps (blue $=0$, yellow $=1$) (\Cref{fig:triangledensity-minimisers}). Reported homomorphism densities are either numerically computed by discretising the learnt graphon into an $M\times M$ weighted adjacency matrix, or estimated using $2^{28}$ samples with a 95\% bootstrap confidence interval (shown in \textit{italic} font) when discretisation is infeasible. Across training runs over the $8$ configurations of activation and learning rate $(\phi, \eta) \in \{\mathrm{GELU}, \sin\} \times \{10^{-3}, 5 \times 10^{-4}, 10^{-4}, 5 \times 10^{-5}\}$, we report the lowest-objective run; reported values should be read as best-found candidate upper bounds, and the bootstrap intervals quantify Monte Carlo estimation error rather than variability across runs. See \Cref{sec:experimentdetails} for details.

\noindent{\bf Constrained homomorphism-density minimisation.}\ $\,$
We apply our method to \textbf{(P1)} across three classes of instances: clique densities ($H = K_n$, fully solved~\citep{razborov2008minimal,reiher2016clique}), odd-cycle densities ($H = C_5, C_7$, partly open), and a systematic sweep over 117 non-bipartite graphs with $v(H) \le 6$. The clique experiments validate the method against known optima; the latter two produce new candidate minimisers and, through the exhaustive sweep, a previously unreported family of extremal structures. We describe each in turn.

\begin{figure}[t]
  \centering
  \begin{minipage}{0.43\textwidth}
    \centering
    \captionof{table}{$C_7$ densities in graphons found by the conjectured optimal construction and our method for $C_7$ density minimisation under $t(K_2, W)=p$.\label{tab:c7result}}
    {\setlength{\tabcolsep}{4pt}\small
    \begin{tabular}{cll}
       \toprule
       $\mathbf{p}$ & \textbf{Conjectured} & \textbf{Ours} \\
       \midrule
       $1/2$ & $0$ & $0.00001517$ \\
       $2/3$ & $0.05761317$ & $0.05763195$ ($+0.03\%$) \\
       $4/5$ & $0.209664$   & $0.20967621$ ($+0.01\%$) \\
       $4/7$ & $0.01713510$ & $0.01715017$ ($+0.09\%$) \\
       $5/8$ & $0.03610272$ & $0.03611957$ ($+0.05\%$) \\
       $5/7$ & $0.09455118$ & $0.09457182$ ($+0.02\%$) \\
       \bottomrule
    \end{tabular}
    }
  \end{minipage}\hfill
  \begin{minipage}{0.55\textwidth}
    \centering
    \begin{subfigure}{0.30\linewidth}
      \includegraphics[width=\linewidth]{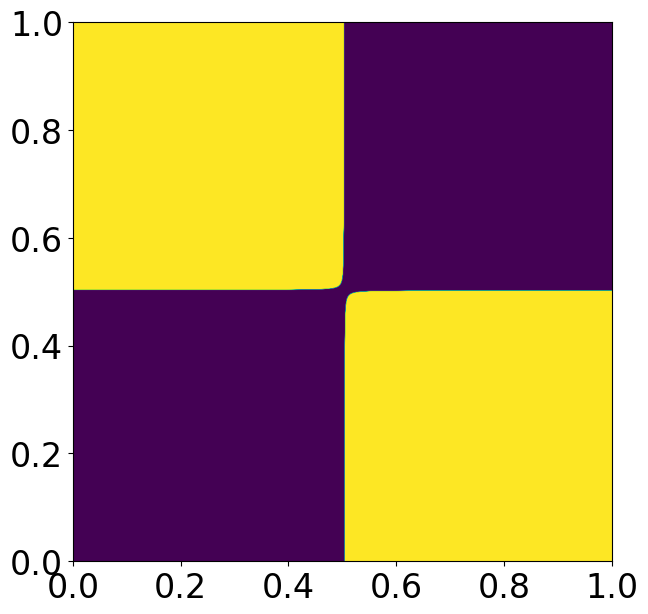}
      \caption{$p{=}1/2$.}\label{fig:opt-c7-50}
    \end{subfigure}\hfill
    \begin{subfigure}{0.30\linewidth}
      \includegraphics[width=\linewidth]{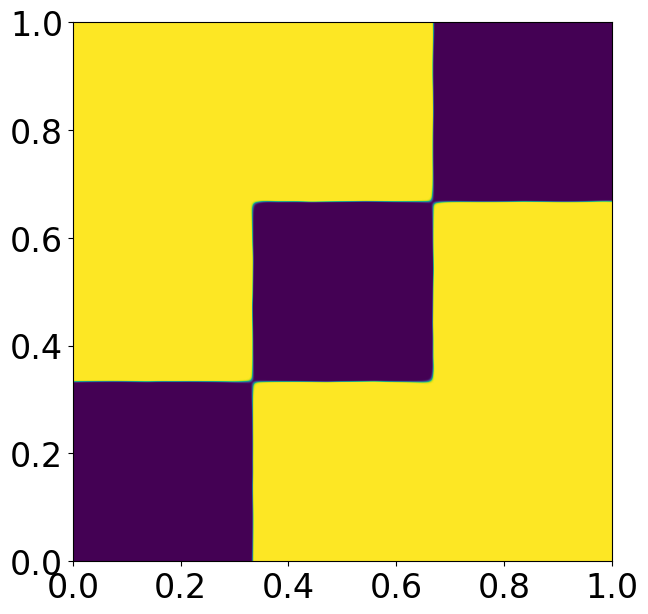}
      \caption{$p{=}2/3$.}\label{fig:opt-c7-67}
    \end{subfigure}\hfill
    \begin{subfigure}{0.30\linewidth}
      \includegraphics[width=\linewidth]{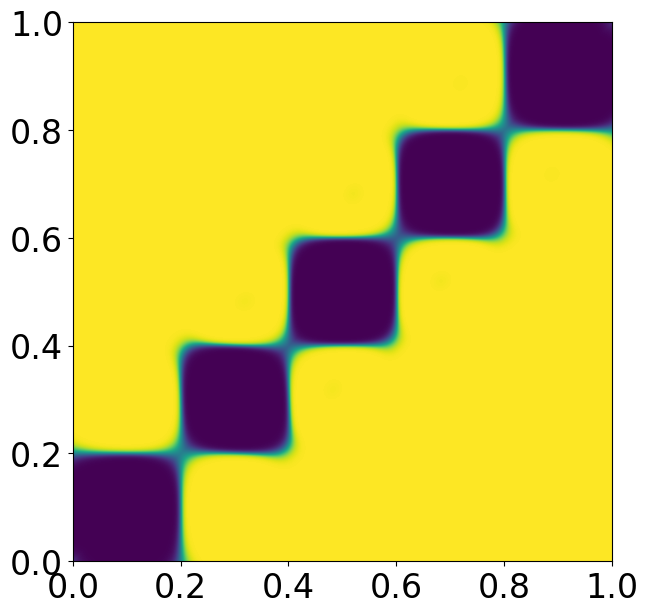}
      \caption{$p{=}4/5$.}\label{fig:opt-c7-80}
    \end{subfigure}\\[0.3em]
    \begin{subfigure}{0.30\linewidth}
      \includegraphics[width=\linewidth]{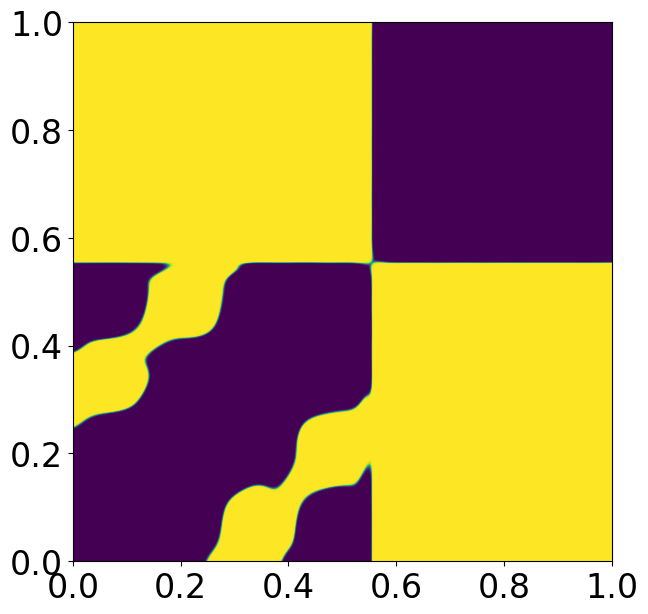}
      \caption{$p{=}4/7$.}\label{fig:opt-c7-57}
    \end{subfigure}\hfill
    \begin{subfigure}{0.30\linewidth}
      \includegraphics[width=\linewidth]{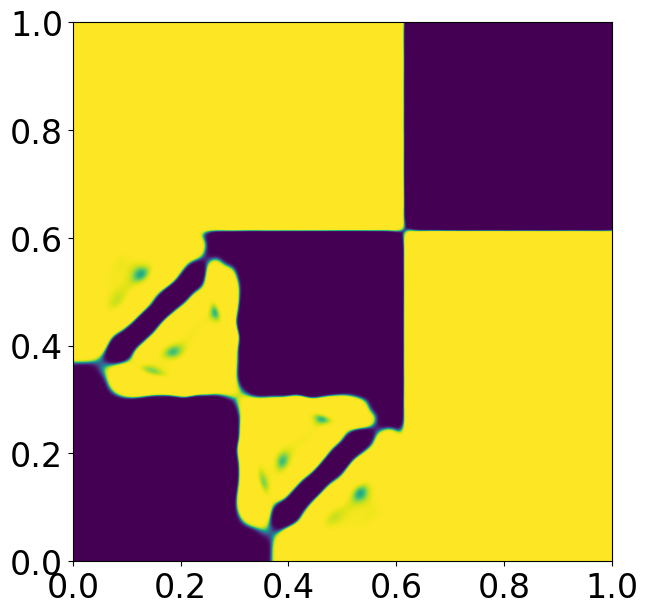}
      \caption{$p{=}5/8$.}\label{fig:opt-c7-62}
    \end{subfigure}\hfill
    \begin{subfigure}{0.30\linewidth}
      \includegraphics[width=\linewidth]{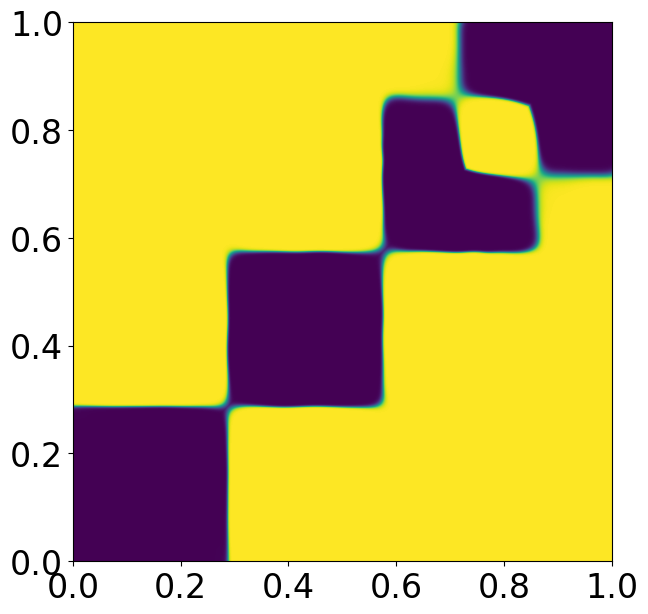}
      \caption{$p{=}5/7$.}\label{fig:opt-c7-71}
    \end{subfigure}
    \captionof{figure}{Graphons found by our method for $C_7$ density minimisation under $t(K_2, W)=p$.}
    \label{fig:7cycleminimisation}
  \end{minipage}
\end{figure}

\noindent{\it Clique densities ($H = K_n$).}\ $\,$
We begin with the case $n\,{=}\,3$. When $p \,{=}\, 1 \,{-}\, 1/k$ for some integer $k \ge 2$, the optimal graphon is the complete $k$-partite graphon or any of its measure-preserving permutations.
For $i \in [k]$, let $X_i = ((i-1)/k, i/k]$. The complete $k$-partite graphon $W_k$ is defined as follows:
$W_k(x_1,x_2) =
1$ if $x_1 \in X_i$ and $x_2 \in X_j$ with $i \neq j$; otherwise,
$W_k(x_1,x_2) = 0$.
For a measure-preserving $\phi$ on $[0,1]$, its permutation is
$W_{k,\phi}(x_1,x_2) = W_k(\phi(x_1), \phi(x_2))$. Every optimum is of this form in this case of $p = 1 - 1/k$.
When $1 - 1/k < p < 1 - 1/(k+1)$, the minimisation admits multiple types of optima.
One example is a variant $W_p$ of $W_{k+1}$,
where the first $k$ parts are equal in length but the last part $X_{k+1}$ is shorter.
Again, its permutations $W_{p,\phi}$ by measure-preserving $\phi$ on $[0,1]$ are also optimal, though
the $W_{p,\phi}$ do not exhaust all optima.

We considered $p \in \{1/2, 2/3, 4/5, 4/7, 5/8, 5/7\}$. For the three cases with $p=1-1/k$ ($p\in\{1/2,2/3,4/5\}$), our method produced near-optimal graphons (top row of \Cref{fig:triangledensity-minimisers}) within $0.1\%$ of the theoretical optima, closely matching the complete $k$-partite graphons $W_k$; among the equivalent measure-preserving permutations, our method consistently converged to these simpler representatives.
For $p\in\{4/7,5/8\}$, which satisfy $1-1/k<p<1-1/(k+1)$, the learnt graphons (\Cref{fig:triangledensity-minimisers}d and~\Cref{fig:triangledensity-minimisers}e) capture the core structure of the known optimal $W_p$.
For $p=5/7$, the learnt graphon (\Cref{fig:triangledensity-minimisers}f) does not match $W_p$, but its bipartite middle-block structure matches the characterisation of all optimal graphons in~\citep{pikhurko2017asymptotic}.

\begin{figure}[t]
  \centering
  \begin{minipage}{0.42\textwidth}
    \centering
    \captionof{table}{$H_6$ densities in graphons found by our method, the $k$-partite construction, and the constant graphon, for $H_6$ density minimisation under $t(K_2, W){=}p$.\label{tab:noveloptimum}}
    {\setlength{\tabcolsep}{4pt}\small
    \begin{tabular}{clll}
      \toprule
      $\mathbf{p}$ & \textbf{Ours} & \textbf{$k$-partite} & \textbf{Constant} \\
      \midrule
      $2/3$ & $0.02476414$ & $0.03292181$ & $0.02601229$ \\
      $3/4$ & $0.07218717$ & $0.07617188$ & $0.07508469$ \\
      $4/5$ & $0.12985603$ & $0.13056$    & $0.13421773$ \\
      \bottomrule
    \end{tabular}
    }
  \end{minipage}\hfill
  \begin{minipage}{0.55\textwidth}
    \centering
    \begin{subfigure}{0.23\linewidth}
      \centering
      \begin{tikzpicture}[scale=0.55]
        \node[circle, draw, fill=black, minimum size=3pt, inner sep=0pt] (v1) at (-1.5, 1.5) {};
        \node[circle, draw, fill=black, minimum size=3pt, inner sep=0pt] (v2) at (-1.5, 0) {};
        \node[circle, draw, fill=black, minimum size=3pt, inner sep=0pt] (v3) at (0, 1.5) {};
        \node[circle, draw, fill=black, minimum size=3pt, inner sep=0pt] (v4) at (0, 0) {};
        \node[circle, draw, fill=black, minimum size=3pt, inner sep=0pt] (v5) at (1.5, 1.5) {};
        \node[circle, draw, fill=black, minimum size=3pt, inner sep=0pt] (v6) at (1.5, 0) {};
        \draw (v5) -- (v4) -- (v1) -- (v3) -- (v5) -- (v6) -- (v3) -- (v2) -- (v4) -- (v6);
      \end{tikzpicture}
      \caption{$H_6$.}
      \label{fig:graph-222}
    \end{subfigure}\hfill
    \begin{subfigure}{0.23\linewidth}
      \includegraphics[width=\linewidth]{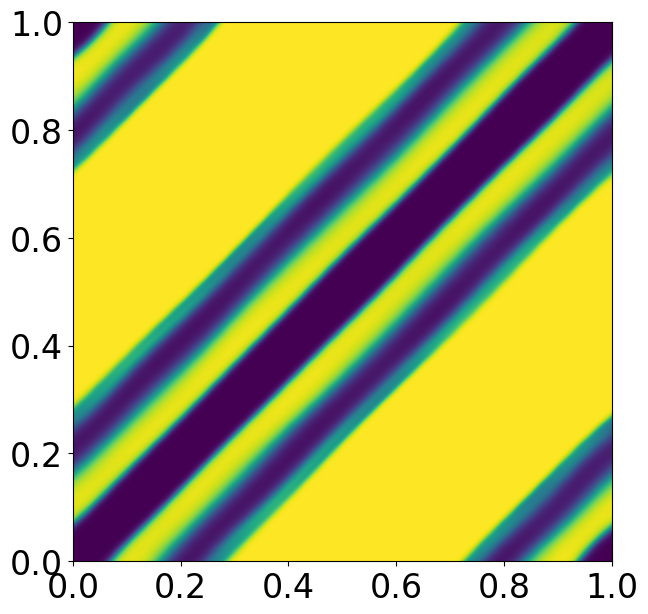}
      \caption{$p{=}2/3$.}
      \label{fig:opt-222-67}
    \end{subfigure}\hfill
    \begin{subfigure}{0.23\linewidth}
      \includegraphics[width=\linewidth]{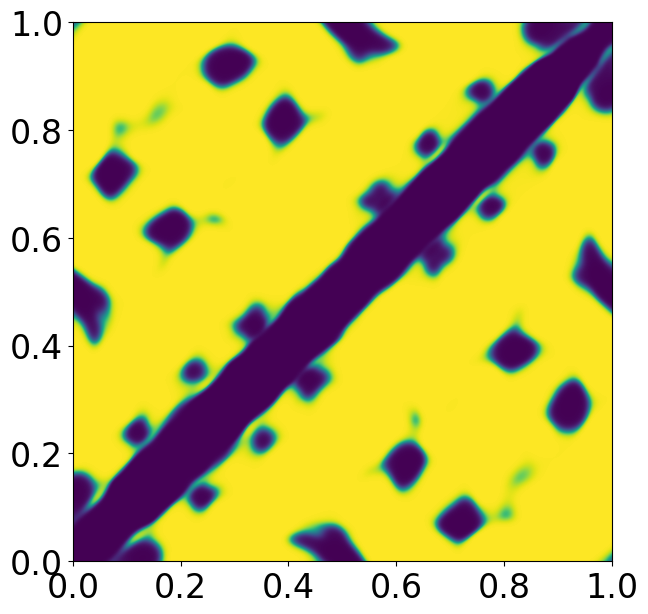}
      \caption{$p{=}3/4$.}
      \label{fig:opt-222-75}
    \end{subfigure}\hfill
    \begin{subfigure}{0.23\linewidth}
      \includegraphics[width=\linewidth]{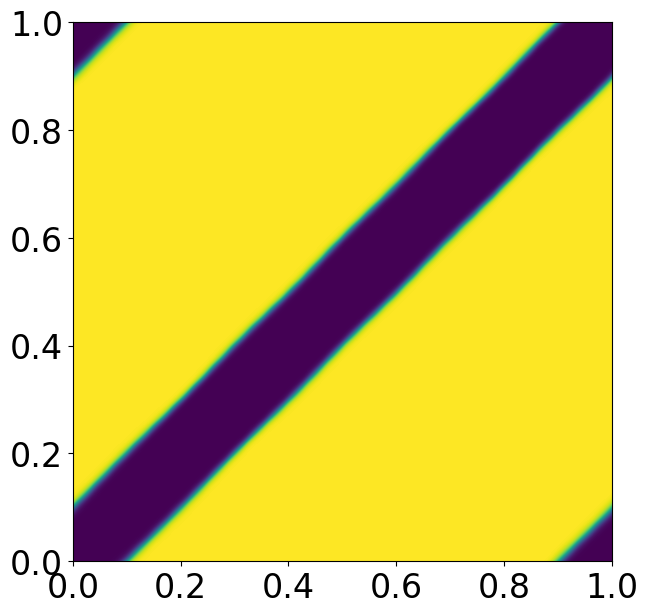}
      \caption{$p{=}4/5$.}
      \label{fig:opt-222-80}
    \end{subfigure}
    \captionof{figure}{The graph $H_6$ (six vertices, nine edges) and graphons found by our method for $H_6$ density minimisation under $t(K_2, W) = p$.}
    \label{fig:noveloptimum}
  \end{minipage}
\end{figure}

\begin{figure}
  \centering
  \begin{subfigure}{0.24\linewidth}
    \includegraphics[width=\linewidth]{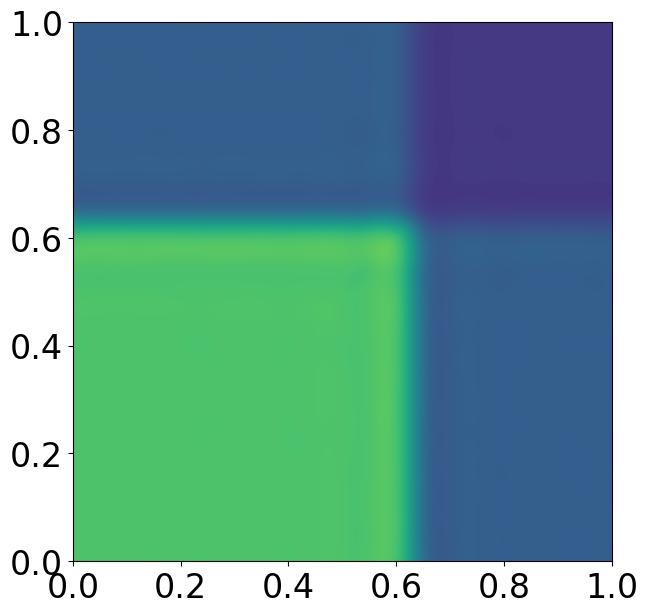}
    \label{fig:upperdeviation-k3-10-50}
  \end{subfigure}\hfill
  \begin{subfigure}{0.24\linewidth}
    \includegraphics[width=\linewidth]{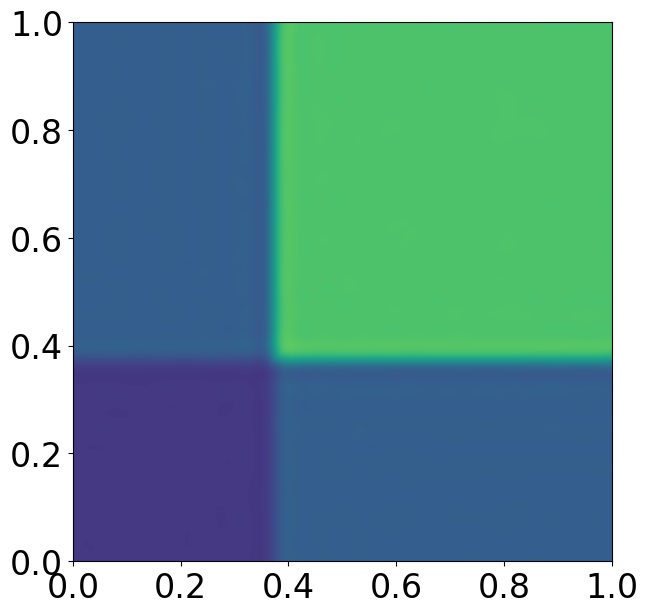}
    \label{fig:upperdeviation-c4-10-50}
  \end{subfigure}\hfill
  \begin{subfigure}{0.24\linewidth}
    \includegraphics[width=\linewidth]{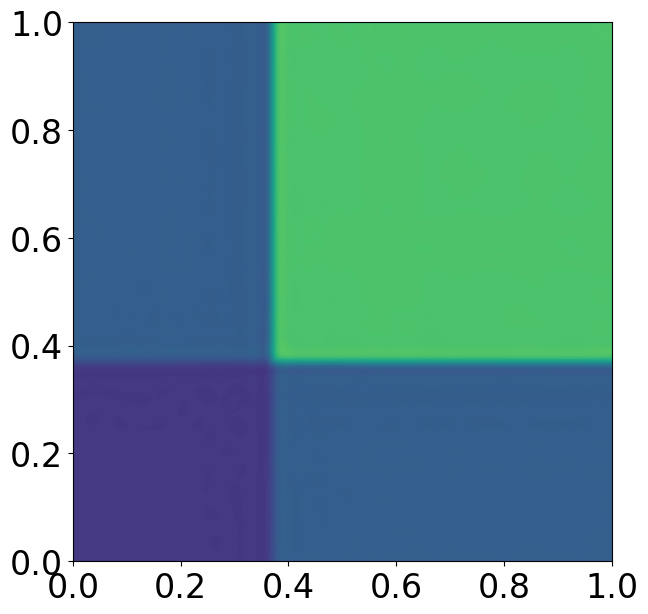}
    \label{fig:upperdeviation-c5-10-50}
  \end{subfigure}\hfill
  \begin{subfigure}{0.24\linewidth}
    \includegraphics[width=\linewidth]{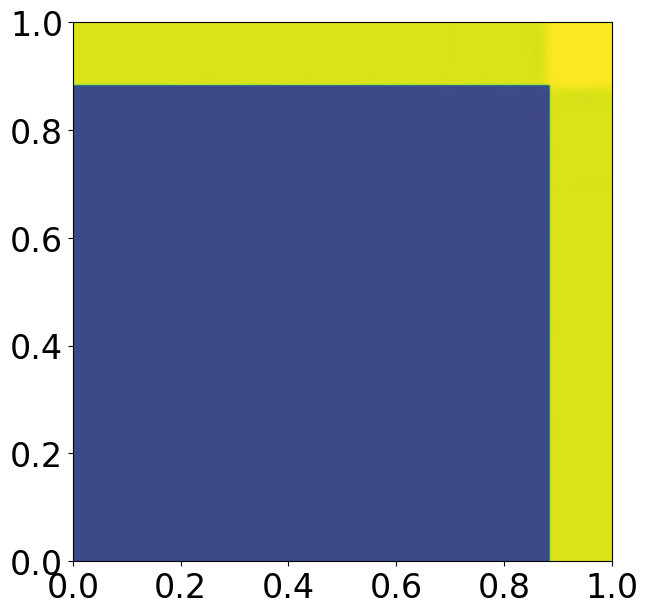}
    \label{fig:upperdeviation-k13-10-50}
  \end{subfigure}
  \caption{Learnt graphons for upper-tail variational problems with $q=0.1, r=0.5$ for $H = K_3, C_4, C_5, K_{1,3}$ (left to right).}
  \label{fig:upperdeviation}
\end{figure}

\begin{table}[t]
  \centering
  \captionof{table}{Relative-entropy $h_q(W)$ at our learnt graphon ($W_{\mathrm{ours}}$) and at the tripartite construction $W_{\mathrm{LZ}}$ of \citet{lubetzky2015replica}, for \textbf{(P2)} with $H = K_3$ on $(q,r) \in \{0.05, 0.1\} \times \{0.4, 0.5, 0.6\}$, together with a lower bound on the optimum (\Cref{sec:addexperiments}). Lower is better.\label{tab:largedeviation}}
  {\small
  \begin{tabular}{lcccccc}
    \toprule
    $(q,r)$ & $(0.05,0.4)$ & $(0.05,0.5)$ & $(0.05,0.6)$ & $(0.10,0.4)$ & $(0.10,0.5)$ & $(0.10,0.6)$ \\
    \midrule
    $h_q(W_{\mathrm{LZ}})$    & $0.556057$ & $0.830366$ & $1.144945$ & $0.311239$ & $0.510826$ & $0.750684$ \\
    $h_q(W_{\mathrm{ours}})$  & $0.464208$ & $0.727881$ & $1.051203$ & $0.309109$ & $0.506060$ & $0.747694$ \\
    Lower bound               & $0.460284$ & $0.725284$ & $1.049172$ & $0.306561$ & $0.504311$ & $0.746006$ \\
    \bottomrule
  \end{tabular}
  }
\end{table}

Turning to larger cliques ($k > 3$), we ran $\min_W t(K_k, W)$ s.t.\ $t(K_2, W){=}1{-}1/k$ for $k\in\{4,5,6\}$. Even though the larger motif size increases the variance of our estimator and the optimal graphon involves more blocks, our method recovers the complete $k$-partite structure (\Cref{fig:clique}).

\noindent{\it Open problems: $C_5$ and $C_7$.}\ $\,$
We next apply our method to two graphon optimisation problems that, to the best of our knowledge, remain open in the tested regimes: minimising the homomorphism densities of $C_5$ and $C_7$, each under a fixed edge density~$p$. In \Cref{sec:addexperiments}, we also report results for the Petersen graph, for which no prior optima or conjectured optima are known to us.

For $C_5$, \citet{bennett2020minimizing} characterised the minimisers when $p=1-1/k$ and conjectured them for other $p$. We applied our method for $p\in\{1/2, 2/3, 4/5, 4/7, 5/8, 5/7\}$ (\Cref{fig:5cycleminimisation}); the resulting values agree with the conjectured optima within $0.1\%$ (\Cref{tab:c5result}). When $p=1-1/k$, the learnt graphon recovers the complete $k$-partite structure (top row); for other $p$, it exhibits pronounced diagonal (banded) structure (bottom row), suggesting a qualitatively different extremal configuration away from the breakpoints $p = 1-1/k$.
For $C_7$, no prior results are known to us. Motivated by the $C_5$ results of \citet{bennett2020minimizing}, we propose a conjecture for the structure and optimum value of $C_7$-minimising graphons under a fixed edge density; see \Cref{sec:oddcycleconjecture}. Empirically, our learnt graphons match our conjectured optima within $0.1\%$ across all tested $p$ (\Cref{tab:c7result}). For $p\in\{1/2,2/3,4/5\}$ ($p=1-1/k$), they recover the complete $k$-partite structure (top row of \Cref{fig:7cycleminimisation}); for $p\in\{4/7,5/8\}$, they show diagonal patterns similar to $C_5$, suggesting that this is a common feature of odd-cycle minimisation; for $p=5/7$, the structure is qualitatively different, with block-like rather than diagonal patterns (bottom-right of \Cref{fig:7cycleminimisation}).

\noindent{\it Exhaustive study over small graphs.}\ $\,$
We now describe a systematic study enabled by our method.
We applied our method to the problem \textbf{(P1)} in \eqref{eqn:targetproblems} for all non-bipartite graphs $H$ with $v(H) \le 6$ and with no isolated vertices or isolated edges, for a total of 117 graphs; isolated vertices and isolated edges are excluded because they only multiply $t(H, W)$ by $1$ and $p$, respectively.
For each $H$, we trained our network under a single fixed hyperparameter configuration (GELU activation, learning rate $10^{-3}$) with three random seeds to minimise $t(H,W)$ subject to $t(K_2,W)=p$, for the six densities $p = 1 - 1/(k-1+j)$ with $j \in \{1/3, 1/2, 2/3, 1, 2, 3\}$, where $k$ is the chromatic number of $H$.
These values of $p$ are chosen to probe the structure of the curve
$x \mapsto \min_{W:t(K_2,W)=x} t(H,W)$:
the values with $j \in \{1, 2, 3\}$ correspond to breakpoints of piecewise-defined curves established for several classes of graphs~\citep{razborov2008minimal,reiher2016clique,bennett2020minimizing}, while those with $j < 1$ lie inside the segment $(k-2)/(k-1) < p < (k-1)/k$.
To assess the optimality of our solutions, we computed lower bounds on the minimum $H$-density using flag algebras~\citep{razborov2007flag} on $7$-vertex graphs.

Taking the best of the three seeds, the $H$-density of the learnt graphon is within $1\%$ of the flag-algebra lower bound in 72.6\% of the 702 $(H,p)$ instances, and a graphon whose $H$-density is more than $1\%$ smaller is known in 20.8\% of them; neither holds for the remaining 6.6\%.

Using the lower bounds, we also examined whether the minimum is convex or concave on the segment containing the three densities with $j < 1$. For 104 of the 117 graphs, including every graph with $k = 3$ except $K_3$, the minimum appears to be convex on this segment, as conjectured for odd cycles~\citep{bennett2020minimizing}; for 7 graphs, it appears to be concave, as for cliques~\citep{reiher2016clique}; and for 5 graphs, including $K_6$ minus an edge, it is neither convex nor concave.
Our experiments also suggest a novel family of extremal graphon structures. When $p \,{=}\, 1 \,{-}\, 1/k$, the minimisers for cliques and $C_5$ are known to be complete $k$-partite~\citep{reiher2016clique,bennett2020minimizing}, and we conjecture the same for $C_7$; in all three cases, our method recovers this structure (\Cref{fig:triangledensity-minimisers,fig:5cycleminimisation,fig:7cycleminimisation}). However, this is not universal: for the graph $H_6$ in Fig.~\ref{fig:noveloptimum}a at the same $p$, our method finds a candidate minimiser that is neither complete $k$-partite nor constant (\Cref{fig:noveloptimum}); the learnt graphons exhibit a geometric pattern in which edge probabilities appear to depend on the circular distance when $[0,1]$ is interpreted as a circle.

\noindent{\bf Variational problem from large deviation theory.}\ $\,$
We apply our framework to \textbf{(P2)}, the variational problem of \citet{chatterjee2011large}.
We test $(q, r) \in \{0.05, 0.1\} \times \{0.4, 0.5, 0.6\}$ for upper-tail variational problems with $H \in \{K_3, C_4, C_5, K_{1,3}\}$, and $(q, r) \in \{0.7, 0.8, 0.9\} \times \{0.05, 0.1\}$ for lower-tail variational problems with $H = K_3$; in all these regimes, the structure of optimisers is, to our knowledge, not fully understood. As a reference for $H = K_3$, we use the tripartite construction $W_{\mathrm{LZ}}$ proposed by Lemma 3.4 of \citet{lubetzky2015replica}; \Cref{tab:largedeviation} compares $h_q(W_{\mathrm{LZ}})$ to $h_q(W_{\mathrm{ours}})$ at our learnt graphon $W_{\mathrm{ours}}$ in the upper-tail regime.
Qualitatively, our method consistently finds two-block structures for all graphs and in both upper- and lower-tail regimes.
For the upper-tail problems, the 2-regular graphs $K_3$, $C_4$, $C_5$ share close edge densities and block sizes in their respective learnt graphons, while the non-2-regular graph $K_{1,3}$ shows a qualitatively different structure with a much smaller block size; see \Cref{fig:upperdeviation} for the learnt graphons with $q=0.1, r=0.5$ and \Cref{sec:addexperiments} for all upper-tail regimes.
For the lower-tail problems, we obtain a qualitatively different structure of nearly-bipartite graphons.
Quantitatively, in the upper-tail problems for $H = K_3$, our method consistently finds graphons with lower relative entropy than the tripartite construction $W_{\mathrm{LZ}}$, showing that the latter is suboptimal in these regimes; see \Cref{sec:addexperiments} for details.

\noindent{\bf Ablation study, baselines, and multi-task variant.}\  $\,$
To assess the contribution of each architectural component and to compare the neural approach with simpler alternatives, we ran an ablation study and four non-neural baselines on triangle-density minimisation under $t(K_2, W) = 7/9$. Replacing the sinusoidal features with GELU features, the constraint solver with a regulariser, or the cosine-annealing schedule with a constant learning rate increases the triangle density, and using the sinusoidal encoding only in the first layer degrades the recovered block structure; the best constant encoding scale reaches a similar density. The four non-neural baselines tailored to step graphons (a fixed-block SBM with $256$ blocks; SBMs with trainable block sizes; a decision tree; a tree ensemble) all reach higher triangle densities than our method. To explore whether a single model can handle multiple problem instances simultaneously, we also trained a FiLM-conditioned multi-task variant over $p \in [1/2, 2/3]$ for $H = K_3$ \citep{perez2018film}, which demonstrates that the architecture extends naturally to this setting. Full details are deferred to \Cref{sec:ablation}.

\begin{ack}
We thank Sangho Lim, Gyeongwon Jung, Seonghun Park, Seonghyuk Im, and Kyeongsik Nam for helpful discussions.
T.~Kim and H.~Yang were supported by the National Research Foundation of Korea (NRF) grant funded by the Korean Government (MSIT) (No.~RS-2023-00279680). H.~Yang was also supported by a KIAS Individual Grant (CG109601) at the Korea Institute for Advanced Study. J.~Lee was supported by Samsung STF Grant (SSTF-BA2201-02) and the National Research Foundation of Korea (NRF) grant funded by the Korean Government (MSIT) (No.~NRF-2022R1C1C1010300).
J.~Baek was supported by KIAS Individual Grant (HP105001).
\end{ack}

\bibliography{references}
\bibliographystyle{plainnat}

\appendix

\newpage

\allowdisplaybreaks[0]

\section{Related work}

Implicit neural representations of graphons were introduced by \citet{xia2023implicit} for \emph{graphon learning}, and \citet{azizpour2025scalable} proposed a scalable variant. In graphon learning, one observes a collection of graphs assumed to be sampled from a common unknown graphon and estimates this graphon from the samples. Recovering sharp details of the true graphon is not the focus of these methods, and such details are often lost in the learnt graphons; see, e.g., \citet[Figs.~3(b) and~9]{azizpour2025scalable}. Our setting is different: there is neither a graph dataset nor a target graphon, and we search directly over graphons to minimise a homomorphism-density or relative-entropy functional under a density constraint. Applying graphon learning to these problems would first require near-optimal finite graphs, which is itself a hard discrete optimisation problem.

Razborov's \emph{flag algebra} framework~\citep{razborov2007flag} enables proofs of inequalities between graph homomorphism densities by means of sums of squares (ultimately, applications of the Cauchy--Schwarz inequality) in an algebra of graphs, where each finite graph can be interpreted as a functional on the space of graphons. This viewpoint also enables the automation of Cauchy--Schwarz-style arguments via semidefinite programming (SDP), and has been implemented in the {\sc Flagmatic} package~\citep{falgas2013applications}.
Flag algebras and Flagmatic have resolved many open problems in extremal combinatorics over the past decade (for a survey, see~\citet{razborov2013flag}), but their computational cost can be prohibitive. \footnote{For example, to derive a nontrivial inequality for a graph $G$, {\sc Flagmatic} typically enumerates all graphs $H$ with $v(H) = k \geq v(G)$ and solves an SDP over vectors indexed by such graphs. Graphs $H$ with more than eight vertices are currently unsupported due to computational complexity.}
This limits applicability to some instances we consider, such as problems involving the Petersen graph.
We view our method as \emph{complementary} to flag algebras rather than a replacement: flag algebras provide rigorous \emph{lower} bounds together with certificates, whereas our method searches directly for candidate graphons and yields candidate \emph{upper} bounds and concrete candidate minimisers. When an upper bound from our method nearly matches a rigorous lower bound, this gives strong evidence for the correct optimum and produces a graphon whose structure can be inspected and turned into mathematical conjectures (as in our exhaustive sweep in \Cref{sec:experiments}).

For fixed values of $n$ (i.e., without taking the limit in \Cref{eqn:extremalproblem:generalisedTuran} or related formulations), one may instead attempt to solve the discrete optimisation problem directly. A classical approach is local (e.g., Tabu) search, which remains useful in combination with reinforcement learning; see, e.g.,~\citet{parczyk2025new,mehrabian2024finding}.
As a direct reinforcement-learning approach, \citet{mehrabian2024finding} used {\sc AlphaZero} to search for $n$-vertex graphs with as many edges as possible and no copies of $K_3$ and $C_4$.
The so-called \emph{deep cross entropy method} was used by~\citet{wagner2021constructions} for analysing discrete problems in combinatorics. It iteratively trains a neural network to learn a policy that promotes the generation of high-quality examples, which succeeded in disproving conjectures regarding the largest eigenvalue of the adjacency matrix and the matching number.
Finally, {\sc PatternBoost}~\citep{charton2024patternboost} takes a hybrid approach that combines local search with a learned generative model; one illustrative example is generating lower-bound examples related to Mantel's theorem by solving finite-$n$ instances of \Cref{eqn:extremalproblem:generalisedTuran} for $H=K_3$ (for small $n$) up to edge density $p \le 1/2$.
For discrete objectives, these approaches can be more versatile than our method because they can optimise non-differentiable quantities (e.g., the matching number). In contrast, they operate at a fixed graph size, whereas our method targets asymptotic problems in the dense-graph limit.
Because many combinatorial objects are highly symmetric, solutions produced by such discrete approaches can be difficult to interpret or distil into human-understandable structure.
In contrast, as demonstrated by visualisations of learnt graphons, our method often captures structure that is amenable to inspection and interpretation.

From the perspective of implicit neural representations, the work most closely related to ours is \citet{mundinger2025neural}, which applies implicit neural representations to the Hadwiger--Nelson problem: colouring the plane $\mathbb{R}^2$ with the minimum number of colours such that no two points at unit distance share a colour. Analogously to our approach, they represent a colouring by a neural network and recover a 6-colouring by minimising a monochromatic edge-density objective. With human post-processing, they construct two new 6-colourings that expand the known range of realisable distances from $[0.415, 0.447]$ to $[0.354, 0.657]$. In contrast, we focus on asymptotic extremal problems in graph theory that can be phrased as constrained optimisation over graphons. Moreover, our approach integrates problem structure into the learning procedure, including an architectural bias toward discontinuous, step-like solutions, explicit constraint enforcement via an embedded solver, and variance reduction in homomorphism density estimation, with the aim of supporting systematic conjecture generation and exploration across diverse extremal problems in graph theory.

\section{Societal impacts and limitations}

Our work presents a method for finding structured candidate optima of graphon variational problems, which is of interest to mathematicians working on extremal graph theory and the large deviation theory of dense random graphs.
Compared with existing baselines, our method produces candidate optima that are markedly more interpretable, offering insight into the structure of optimal graphons and thereby supporting mathematical conjecture generation and discovery.
We do not foresee any major negative societal impacts arising from this work.

Our method has several limitations.
First, it is tailored to graphon optimisation and therefore does not apply to combinatorial problems on sparse graphs, for which the graphon limit is degenerate.
Second, although the candidate optima it produces are highly interpretable, the method itself does not provide rigorous guarantees of optimality; it should therefore be used in conjunction with rigorous lower-bound techniques such as flag algebras~\citep{razborov2007flag,falgas2013applications} to verify optimality and provide certificates.
Third, evaluating our method requires estimating homomorphism densities, for which we use either discretisation-based computation or Monte Carlo estimation; both can be computationally expensive for large motifs.
Discretisation-based computation of $t(H, W)$ on an $M \times M$ grid becomes infeasible for motifs with many vertices, such as the Petersen graph; Monte Carlo estimation can suffer from large relative error when the motif density is very small.

\section{Background on graphons, odd cycles, and large deviations}\label{sec:graphon-background}

In this section, we give an informal introduction to the graphon formulation of \Cref{sec:background} and state the conjectured optima for odd cycles used in \Cref{sec:experiments}.

\paragraph{Graphons as limits of graphs.}
Extremal graph theory and the large deviation theory of random graphs ask questions about very large graphs, such as how few triangles a graph with a given edge density can have. Graphons replace a sequence of growing graphs by a single function on the unit square, and our method optimises over these functions.

A graph $G$ with vertex set $[n]$ can be drawn as an image on the unit square through its adjacency matrix:
\[
W_G(x, y) \,\coloneqq\,
\begin{cases}
1 & \text{if } x \in \bigl(\frac{i-1}{n}, \frac{i}{n}\bigr],\ y \in \bigl(\frac{j-1}{n}, \frac{j}{n}\bigr], \text{ and } \{i, j\} \in E(G), \\
0 & \text{otherwise}.
\end{cases}
\] Note that isomorphic graphs give images that differ only by a relabelling, i.e., a permutation of rows and columns. The heatmaps in our figures show graphons in the same way.

Consider the limit of such images when the number of vertices grows. For the complete graph $K_n$, the image takes the value $1$ except on the $n$ diagonal squares, whose total area is $1/n$, so it converges to the constant function $1$ as $n \to \infty$. For the Erd\H{o}s--R\'enyi random graph $G(n, q)$, the image is noise with entries $0$ and $1$, which never approaches the constant function $q$ pointwise. Nevertheless, the constant graphon $W \equiv q$ is the natural limit of $G(n, q)$. This limit is characterised by the so-called cut distance. For a symmetric measurable function $W : [0,1]^2 \to [0,1]$, we consider a measure-preserving bijection $\sigma : [0,1] \to [0,1]$, which is a continuous counterpart of a relabelling of vertices, and write $W^\sigma(x, y) \coloneqq W(\sigma(x), \sigma(y))$. Then the cut distance is the infimum of the difference between two graphons at a cut $S \times T$ over such transformed graphons:
\[
\delta_\square(U, W) \,\coloneqq\, \inf_{\sigma}\, \sup_{S, T \subseteq [0,1]} \biggl| \int_{S \times T} \bigl(U(x, y) - W^\sigma(x, y)\bigr)\, \mathrm{d}x\, \mathrm{d}y \biggr|.
\]
A sequence of graphs $(G_n)$ converges to a graphon $W$ if $\delta_\square(W_{G_n}, W) \to 0$. The space of graphons, considered up to measure-preserving transformations (more precisely, with graphons at cut distance zero identified), is compact in the cut distance, and the images $W_G$ of graphs are dense in it~\citep{lovasz2012large}. In particular, every sequence of graphs has a subsequence that converges to a graphon.

Homomorphism densities are compatible with this limit. For every graph $H$ and graphons $U$ and $W$, we have
\[
|t(H, U) - t(H, W)| \,\le\, e(H)\, \delta_\square(U, W),
\]
so the $H$-density of a large graph is close to that of any graphon near it in the cut distance. Conversely, for a sequence of graphs $(G_n)$, if $t(H, W_{G_n}) \to t(H, W)$ for every graph $H$, then $W_{G_n} \to W$ in the cut distance~\citep{lovasz2012large}. By the continuity of homomorphism densities and the compactness of the space of graphons, an optimisation problem over large graphs whose objective and constraints are continuous in the cut distance, such as homomorphism densities, becomes an optimisation problem over graphons: the asymptotic optimum over graph sequences equals the optimum over graphons, and the latter is attained. For example, the asymptotic minimum in~\eqref{eqn:extremalproblem:generalisedTuran} equals the minimum in \textbf{(P1)}. The optimum is also attained for objectives that are only lower semicontinuous in the cut distance, such as the relative entropy $h_q$ in \textbf{(P2)}~\citep{chatterjee2011large}.

\paragraph{Conjectured optima for odd cycles.}\label{sec:oddcycleconjecture}
In \citet{bennett2020minimizing}, the conjectured optimal graphon is constructed as follows.
Let $k$ be the integer satisfying $1 - \frac{1}{k} < p < 1-\frac{1}{k+1}$.
First, we partition the interval $[0,1]$ into $(k+1)$ sets, $X_1, \ldots, X_{k-1}, Y_1, Y_2$.
We set the lengths of $X_1,\ldots,X_{k-1}$ to be $x$, and the lengths of $Y_1$ and $Y_2$ to be $\frac{y}{2}$, where $y=1-(k-1)x$.
We then define the graphon by
\[
  W(x_1, x_2) = \begin{cases}
    1 & \text{if } x_1 \in X_m,\ x_2 \in X_n,\ m \neq n
    \\
    1 & \text{if } x_1 \in X_m,\ x_2 \in Y_n,\ \text{or vice versa}
    \\
    0 & \text{if } x_1 \in X_m,\ x_2 \in X_m
    \\
    \rho & \text{if } x_1 \in Y_1,\ x_2 \in Y_2\ \text{or vice versa}
    \\
    0 & \text{if } x_1 \in Y_m,\ x_2 \in Y_m
  \end{cases}
\]
where $\rho$ is chosen to satisfy the edge-density constraint.

One can compute the $C_5$ density of this graphon, which can be written as a polynomial $f_5(x, y, \rho)$ given by
\begin{align*}
    &f_5(x, y, \rho)
    \\
    {} = &\left(\frac{(k-1)_5}{10} + \frac{(k-1)_4}{2} + \frac{(k-1)_3}{2} \right)x^5
    \\
    {} + &\left( \frac{(k-1)_4}{2} + \frac{3(k-1)_3}{2} + \frac{(k-1)_2}{2}\right)x^4 y
    \\
    {} + &\left(\left(\frac{1}{2} + \frac{\rho}{2}\right) (k-1)_3 + \left(1 + \frac{\rho}{2}\right) (k-1)_2\right) x^3 y^2
    \\
    {} + &\left( \left(\frac{\rho}{2} + \frac{\rho^2}{2}\right) (k-1)_2 + \frac{\rho}{2}(k-1) \right) x^2 y^3
    \\
    {} + &\frac{\rho^3}{2} (k-1) xy^4,
\end{align*}
where $(k-1)_r$ is a falling factorial, $(k-1)_r = (k-1) \cdot (k-2) \cdots (k-r)$.
With these definitions, the conjectured optimal value is obtained by optimising over $x$ and $\rho$ subject to the constraints below (with $y$ determined by $x$ via $(k-1)x + y = 1$).
The contribution of the $Y_1$--$Y_2$ rectangles to the ordered edge density is $\rho y^2/2$, since the two off-diagonal rectangles have total area $y^2/2$.
Equivalently, one can compute the conjectured optimum by solving
\begin{align*}
    &\min f_5(x, y, \rho)
    \\
  \text{subject to } & (k-1) x + y = 1,
    \\
    & (k-1)_2 x^2 + 2 (k-1) xy + \frac{\rho y^2}{2} = p,
    \\
    & x, y \ge 0, \quad 0 \le \rho \le 1.
\end{align*}
Using the same structure, we can generalise the conjecture to arbitrary odd cycles $C_{2\ell+1}$ by replacing $f_5$ with the corresponding $f_{2\ell+1}$.
For instance, $f_7$ is given as follows.

\begin{align*}
  &f_7(x, y, \rho)
  \\
  {} = &\left(\frac{(k-1)_7}{14} + (k-1)_6 + 4 (k-1)_5 + 5(k-1)_4 + \frac{3(k-1)_3}{2}\right) x^7
  \\
  {} + &\left( \frac{(k-1)_6}{2} + 5 (k-1)_5 + \frac{25(k-1)_4}{2} + \frac{15 (k-1)_3}{2} + \frac{(k-1)_2}{2} \right) x^6 y
  \\
  {} + &\left( (k-1)_5 + (3\rho + 7)(k-1)_4 + \left( \frac{7\rho}{2} + 10 \right) (k-1)_3 + \left( \frac{\rho}{2} + 2 \right) (k-1)_2 \right) x^5 y^2
  \\
  {} + &\left( \left( \frac{3 \rho}{2} + \frac{1}{2} \right) (k-1)_4 + \left( \frac{3\rho^2}{2} + 6 \rho + \frac{5}{2} \right)(k-1)_3 + \left( \frac{\rho^2}{2} + 3\rho + 2 \right)(k-1)_2\right) x^4 y^3
  \\
  {} + &\left( \left( \frac{\rho^3}{2} + \frac{3\rho^2}{2} + \frac{\rho}{2} \right)(k-1)_3 + \left( \frac{\rho^3}{2} + 3\rho^2 + \frac{3\rho}{2} \right)(k-1)_2 + \frac{\rho}{2} (k-1) \right) x^3 y^4
  \\
  {} + &\left( \left( \frac{\rho^4}{2} + \rho^3 \right) (k-1)_2 + \rho^3 (k-1) \right) x^2 y^5
  \\
  {} + & \frac{\rho^5}{2} (k-1) x y^6.
\end{align*}

\paragraph{Large deviations and the relative entropy.}
For $a, q \in (0,1)$, write
\[
\mathrm{KL}(a \,\|\, q) \,\coloneqq\, \mathrm{KL}\bigl(\mathrm{Bern}(a) \,\|\, \mathrm{Bern}(q)\bigr) \,=\, a \log\frac{a}{q} + (1-a) \log\frac{1-a}{1-q}.
\]
In $G(n, q)$, for each of the $N = \binom{n}{2}$ vertex pairs, the two vertices are adjacent independently with probability $q$. If $aN$ is an integer, Stirling's formula gives $\log\binom{N}{aN} = -N\bigl(a \log a + (1-a)\log(1-a)\bigr) + o(N)$. Hence,
\begin{align*}
\log \Pr_{G \sim G(n, q)}\bigl[e(G) = aN\bigr]
&= \log\Bigl(\binom{N}{aN} q^{aN} (1-q)^{(1-a)N}\Bigr)
\\
&= -N\Bigl(a \log\frac{a}{q} + (1-a) \log\frac{1-a}{1-q}\Bigr) + o(N)
\\
&= -N\, \mathrm{KL}(a \,\|\, q) + o(N).
\end{align*}
Thus the probability that $G(n, q)$ has edge density $a$ rather than the typical value $q$ decays exponentially in $N$, at rate $\mathrm{KL}(a \,\|\, q)$. The large deviation principle of \citet{chatterjee2011large} extends this from a single edge density to the whole profile of edge probabilities: informally, the probability that $G(n, q)$ is close in cut distance to a graphon $W$ is $\exp(-N h_q(W) + o(n^2))$, where $h_q(W) = \int_{[0,1]^2} \mathrm{KL}(W(x,y) \,\|\, q)\, \mathrm{d}x\, \mathrm{d}y$ as in \Cref{sec:background}. The probability of an event such as $t(H, G(n,q)) \ge r^{e(H)}$ is dominated by the graphons that satisfy this constraint with the smallest relative entropy, which gives the variational formulas in \Cref{sec:background}, and conditioned on the event, $G(n, q)$ is close to these minimisers~\citep{lubetzky2015replica}.

\paragraph{Optimal solutions of large deviation problems.}
Consider the upper-tail problem in \textbf{(P2)}, where $q < r$:
\[
\min_{W} h_q(W) \quad \text{s.t.} \quad t(H, W) = r^{e(H)}.
\]
The constant graphon $W \equiv r$ satisfies the constraint $t(H, W) = r^{e(H)}$, so it is a feasible solution, with objective value $h_q(W) = \mathrm{KL}(r \,\|\, q)$. For a $d$-regular graph $H$ with $d \ge 2$, \citet{lubetzky2015replica} characterised when this constant graphon is optimal: there is a curve $q_c(r)$ such that $W \equiv r$ is the unique optimum of \textbf{(P2)} if $q > q_c(r)$, and some graphon has a smaller objective value than $W \equiv r$ if $q < q_c(r)$. The region where the constant graphon $W \equiv r$ is optimal is called the \emph{replica-symmetric phase}, and the region where it is not optimal is called the \emph{symmetry-breaking phase}.

\section{Proven results on the learnt structures}\label{sec:proven}

The graphons found by our method motivated the following three theorems. We state them in the notation of this paper and explain how each relates to our experiments. Recall the balanced complete $k$-partite graphon $W_k$ from \Cref{sec:experiments}.

\begin{theorem}[\citealp{im2026goodman}]\label{thm:oddcycle}
Let $m \ge 3$ be odd, and let $W$ be a graphon with $t(K_2, W) = p \ge 2/3$. Then
\[
t(C_m, W) \,\ge\, p^m - p(1-p)^{m-1}.
\]
In particular, when $p = 1 - 1/k$ for some integer $k \ge 3$, the balanced complete $k$-partite graphon $W_k$ minimises the density of every odd cycle.
\end{theorem}
The second statement holds because $t(C_m, W_k) = p^m - p(1-p)^{m-1}$ at $p = 1 - 1/k$. Hence the learnt $C_7$ graphons at $p \in \{2/3, 4/5\}$ (top row of \Cref{fig:7cycleminimisation}), which recover the complete $k$-partite structure, are optimal.

A graph is \emph{chordal} if no cycle of length greater than $3$ appears as an induced subgraph of it.
\begin{theorem}\label{thm:chordal}
Let $H$ be a connected chordal graph whose maximal cliques all have the same size $r \ge 3$. Then, for every integer $k \ge r$, the graphon $W_k$ minimises $t(H, W)$ subject to $t(K_2, W) = 1 - 1/k$.
\end{theorem}
A proof of \Cref{thm:chordal}, formalised in Lean~4, is available at \url{https://github.com/mekty2012/PureChordalGoodman}. \Cref{thm:chordal} applies to 16 graphs of our exhaustive study in \Cref{sec:experiments} other than cliques, and settles $48$ of its instances.

A graphon $W$ is \emph{bipodal} if there is a measurable set $A \subseteq [0,1]$ such that, with $B = [0,1] \setminus A$, the graphon $W$ is almost everywhere constant on each of $A \times A$, $(A \times B) \cup (B \times A)$, and $B \times B$. Recall the curve $q_c(r)$ of \citet{lubetzky2015replica} from \Cref{sec:graphon-background}.
\begin{theorem}[{\citealp[Theorem~1.5]{lim2026bipodal}}]\label{thm:bipodal}
Let $H$ be a $d$-regular graph with $d \ge 2$, and let $r_0 \in (0,1) \setminus \{(d-1)/d\}$. There is an open neighbourhood $U$ of $(q_c(r_0), r_0)$ such that, for every $(q, r) \in U$ with $q < q_c(r)$, the minimiser of $h_q(W)$ subject to $t(H, W) \ge r^{e(H)}$ is unique up to relabelling, nonconstant, and bipodal.
\end{theorem}
This partially resolves the conjecture that the upper-tail large deviation problem for regular graphs has bipodal graphons as optimal solutions. The learnt upper-tail graphons for the $2$-regular graphs $K_3$, $C_4$, and $C_5$ have this two-block form (\Cref{fig:upperdeviation} and \Cref{sec:addexperiments}).

\section{Properties of the density estimator and the implicit gradient}

\begin{lemma} \label{lem:unbiased}
The estimator in~\Cref{eqn:estimator} is unbiased, i.e., $\mathbb{E}[\widehat{t}(H,f_\theta[c])] = t(H,f_\theta[c])$.
\end{lemma}
\begin{proof}[Proof]
We compute the claimed equality as follows:
\begin{align*}
\mathbb{E}\left[\widehat{t}(H, f_\theta{[c]})\right]
& {} =
\mathbb{E}\left[\frac{1}{N|\mathcal{S}|} \sum_{k \in [N], \pi \in \mathcal{S}} \left(\prod_{\{i, j\} \in E(H)} f_\theta[c]\Big(x_{\pi(i)}^{(k)}, x_{\pi(j)}^{(k)}\Big)\right)\right]
\\
& {} =
\frac{1}{N|\mathcal{S}|} \sum_{k \in [N], \pi \in \mathcal{S}}
\mathbb{E}\left[\prod_{\{i, j\} \in E(H)} f_\theta[c]\Big(x_{\pi(i)}^{(k)}, x_{\pi(j)}^{(k)}\Big)\right]
\\
& {} =
\frac{1}{N|\mathcal{S}|} \sum_{k \in [N], \pi \in \mathcal{S}}
\mathbb{E}_{x_1,\ldots,x_{v(H)} \stackrel{\mathrm{iid}}{\sim}\mathrm{Unif}[0,1]}\left[\prod_{\{i, j\} \in E(H)} f_\theta[c](x_i, x_j)\right]
\\
& {} =
\frac{1}{N{|\mathcal{S}|}}
\sum_{k \in [N], \pi \in \mathcal{S}} t(H,f_\theta{[c]}) = t(H,f_\theta{[c]}).
\end{align*}
\end{proof}

\begin{lemma}[Restatement of \Cref{lem:implicitdifferentiation}]
For all $\mathbf{x} \in [0,1]^2$,
\begin{align*}
\nabla_\theta\, f_\theta[c_\theta](\mathbf{x})
\,&=\, \frac{\partial f_\theta[c](\mathbf{x})}{\partial \theta}\bigg|_{c = c_\theta}
\,-\, \mathit{sm}'\bigl(h_\theta(\mathbf{x}) + c_\theta\bigr)
\sum_{\mathbf{y} \in \mathcal{E}} w_\theta(\mathbf{y})\, \nabla_\theta h_\theta(\mathbf{y}),
\\
\text{where }
w_\theta(\mathbf{y}) \,&\coloneqq\, \frac{\mathit{sm}'(h_\theta(\mathbf{y}) + c_\theta)}{\sum_{\mathbf{z} \in \mathcal{E}} \mathit{sm}'(h_\theta(\mathbf{z}) + c_\theta)}.
\end{align*}
\end{lemma}
\begin{proof}[Proof]
Let $\mathbf{x} \in [0,1]^2$. Note that
\begin{align*}
    \nabla_\theta\;f_\theta[c_\theta](\mathbf{x})
    & = \frac{\partial f_\theta[c](\mathbf{x})}{\partial\theta}\Big|_{c = c_\theta}
        + \frac{\partial f_\theta[c](\mathbf{x})}{\partial c}\Big|_{c = c_\theta} \nabla_\theta c_\theta
    \\
    & = \frac{\partial f_\theta[c](\mathbf{x})}{\partial\theta}\Big|_{c = c_\theta}
        + \mathit{sm}'(h_\theta(\mathbf{x}) + c_\theta) \nabla_\theta c_\theta
\end{align*}
Thus, it suffices to show that
\[
\nabla_\theta c_\theta =
 {-}  \sum_{\mathbf{y} \in \mathcal{E}}\frac{\mathit{sm}'(h_\theta(\mathbf{y}) + c_\theta)}{\sum_{\mathbf{z} \in \mathcal{E}} \mathit{sm}'(h_\theta(\mathbf{z}) + c_\theta)} \nabla_\theta\, h_\theta(\mathbf{y}).
\]
We prove the above equality by applying implicit differentiation to the following equation about $c_\theta$, which holds by the choice of $c_\theta$:
\[
\frac{1}{|\mathcal{E}|} \sum_{\mathbf{y} \in \mathcal{E}} f_{\theta}[c_\theta](\mathbf{y}) = p.
\]
First, we multiply both sides of the equation by $|\mathcal{E}|$, and differentiate them with respect to $\theta$:
\[
\left(\sum_{\mathbf{y} \in \mathcal{E}} \frac{\partial f_{\theta}[c](\mathbf{y})}{\partial \theta}\Big|_{c = c_\theta}\right)
+
\left(\sum_{\mathbf{z} \in \mathcal{E}} \frac{\partial f_{\theta}[c](\mathbf{z})}{\partial c}\Big|_{c = c_\theta} \nabla_\theta c_\theta\right)
= 0
\]
Thus,
\begin{align*}
    \nabla_\theta c_\theta
    & {} = {-} \left(\sum_{\mathbf{y} \in \mathcal{E}} \frac{\partial f_{\theta}[c](\mathbf{y})}{\partial \theta}\Big|_{c = c_\theta}\right) \Bigg/  \left(\sum_{\mathbf{z} \in \mathcal{E}} \frac{\partial f_{\theta}[c](\mathbf{z})}{\partial c}\Big|_{c = c_\theta}\right)
    \\
    & =  {-} \sum_{\mathbf{y} \in \mathcal{E}} \frac{\mathit{sm}'(h_\theta(\mathbf{y}) + c_\theta)}{\sum_{\mathbf{z} \in \mathcal{E}} \mathit{sm}'(h_\theta(\mathbf{z}) + c_\theta)} \nabla_\theta\, h_\theta(\mathbf{y}),
\end{align*}
as desired.
\end{proof}

The gradient in \Cref{lem:implicitdifferentiation} is the exact gradient for the empirical constraint~\eqref{eqn:solver-specification}. Since $\nabla_\theta c_\theta$ is a ratio of two sample averages, it is in general a biased estimator of the gradient of the offset that enforces the population constraint $t(K_2, f_\theta[c]) = p$. It is, however, a consistent estimator, i.e., it converges to this gradient in the limit of infinitely many samples.

\begin{lemma}[Consistency of the implicit gradient]\label{lem:consistency}
Assume that $h_\theta(\mathbf{y})$ and $\nabla_\theta h_\theta(\mathbf{y})$ are continuous in $(\theta, \mathbf{y})$. Fix $\theta$ and $p \in (0,1)$, and let $c^*_\theta$ be the unique $c \in \R$ with $t(K_2, f_\theta[c]) = p$. Let $\mathbf{y}_1, \mathbf{y}_2, \ldots \stackrel{\mathrm{iid}}{\sim} \mathrm{Unif}([0,1]^2)$, and for $K \ge 1$ let $c_{\theta,K} \coloneqq \textnormal{\textsc{Solve}}_\theta(\{\mathbf{y}_1, \ldots, \mathbf{y}_K\}, p)$. Then, almost surely, $c_{\theta,K} \to c^*_\theta$ and
\begin{align*}
\nabla_\theta c_{\theta,K}
&= -\frac{\sum_{k=1}^{K} \mathit{sm}'\bigl(h_\theta(\mathbf{y}_k) + c_{\theta,K}\bigr)\, \nabla_\theta h_\theta(\mathbf{y}_k)}{\sum_{k=1}^{K} \mathit{sm}'\bigl(h_\theta(\mathbf{y}_k) + c_{\theta,K}\bigr)}
\\
&\longrightarrow\;
\nabla_\theta c^*_\theta
\,=\, -\frac{\mathbb{E}\bigl[\mathit{sm}'\bigl(h_\theta(\mathbf{y}) + c^*_\theta\bigr)\, \nabla_\theta h_\theta(\mathbf{y})\bigr]}{\mathbb{E}\bigl[\mathit{sm}'\bigl(h_\theta(\mathbf{y}) + c^*_\theta\bigr)\bigr]}
\end{align*}
as $K \to \infty$, where $\mathbf{y} \sim \mathrm{Unif}([0,1]^2)$.
\end{lemma}
\begin{proof}
Let $M \coloneqq \max_{\mathbf{y} \in [0,1]^2} |h_\theta(\mathbf{y})|$, which is finite because $h_\theta$ is continuous on the compact set $[0,1]^2$, and let $\ell \coloneqq \mathit{sm}^{-1}(p)$. For $c \in \R$, define
\[
F_K(c) \,\coloneqq\, \frac{1}{K} \sum_{k=1}^{K} \mathit{sm}\bigl(h_\theta(\mathbf{y}_k) + c\bigr),
\qquad
F(c) \,\coloneqq\, \mathbb{E}\bigl[\mathit{sm}\bigl(h_\theta(\mathbf{y}) + c\bigr)\bigr] \,=\, t(K_2, f_\theta[c]).
\]
Both functions are continuous and strictly increasing in $c$. If $c < \ell - M$, then $\mathit{sm}(h_\theta(\mathbf{y}) + c) < \mathit{sm}(\ell) = p$ for every $\mathbf{y}$, and if $c > \ell + M$, then $\mathit{sm}(h_\theta(\mathbf{y}) + c) > p$ for every $\mathbf{y}$. Hence $c_{\theta,K}$ and $c^*_\theta$ lie in $[\ell - M, \ell + M]$.

\emph{Convergence of the offsets.} Let $\varepsilon > 0$. Since $F$ is strictly increasing, $F(c^*_\theta - \varepsilon) < p < F(c^*_\theta + \varepsilon)$. By the strong law of large numbers, $F_K(c^*_\theta \pm \varepsilon) \to F(c^*_\theta \pm \varepsilon)$ almost surely, so almost surely $F_K(c^*_\theta - \varepsilon) < p < F_K(c^*_\theta + \varepsilon)$ for all sufficiently large $K$, and then $|c_{\theta,K} - c^*_\theta| < \varepsilon$ because $F_K$ is increasing. Taking $\varepsilon = 1/m$ for every $m \in \mathbb{N}$ gives $c_{\theta,K} \to c^*_\theta$ almost surely.

\emph{Convergence of the gradients.} The formula for $\nabla_\theta c_{\theta,K}$ is the one derived in the proof of \Cref{lem:implicitdifferentiation}. Let $g : [0,1]^2 \to \R^{m'}$ be bounded and measurable. Since $|\mathit{sm}''| \le 1$, the function $\mathit{sm}'$ is $1$-Lipschitz, so
\begin{multline*}
\Bigl\| \frac{1}{K} \sum_{k=1}^{K} \mathit{sm}'\bigl(h_\theta(\mathbf{y}_k) + c_{\theta,K}\bigr)\, g(\mathbf{y}_k) - \mathbb{E}\bigl[\mathit{sm}'\bigl(h_\theta(\mathbf{y}) + c^*_\theta\bigr)\, g(\mathbf{y})\bigr] \Bigr\|
\\
\le\, \|g\|_\infty\, |c_{\theta,K} - c^*_\theta|
+ \Bigl\| \frac{1}{K} \sum_{k=1}^{K} \mathit{sm}'\bigl(h_\theta(\mathbf{y}_k) + c^*_\theta\bigr)\, g(\mathbf{y}_k) - \mathbb{E}\bigl[\mathit{sm}'\bigl(h_\theta(\mathbf{y}) + c^*_\theta\bigr)\, g(\mathbf{y})\bigr] \Bigr\|.
\end{multline*}
The first term tends to $0$ almost surely by the convergence of the offsets, and the second by the strong law of large numbers. We apply this with $g \equiv 1$ and with $g = \nabla_\theta h_\theta$, which is bounded because it is continuous on $[0,1]^2$. The limit of the denominator is positive: since $h_\theta(\mathbf{y}) + c^*_\theta \in [\ell - 2M, \ell + 2M]$ and $\mathit{sm}'(z)$ decreases in $|z|$,
\[
\mathbb{E}\bigl[\mathit{sm}'\bigl(h_\theta(\mathbf{y}) + c^*_\theta\bigr)\bigr]
\,\ge\, \min\bigl\{\mathit{sm}'(\ell - 2M),\, \mathit{sm}'(\ell + 2M)\bigr\} \,>\, 0.
\]
Hence the ratio converges almost surely to $-\mathbb{E}[\mathit{sm}'(h_\theta(\mathbf{y}) + c^*_\theta)\, \nabla_\theta h_\theta(\mathbf{y})] / \mathbb{E}[\mathit{sm}'(h_\theta(\mathbf{y}) + c^*_\theta)]$.

Finally, this limit equals $\nabla_\theta c^*_\theta$. The map $(\theta, c) \mapsto \mathbb{E}[\mathit{sm}(h_\theta(\mathbf{y}) + c)]$ is continuously differentiable, since its partial derivatives are expectations of functions continuous in $(\theta, c, \mathbf{y})$ over a compact domain, and its derivative in $c$ is the positive denominator above. The implicit function theorem applied to $\mathbb{E}[\mathit{sm}(h_\theta(\mathbf{y}) + c^*_\theta)] = p$ then gives the stated formula for $\nabla_\theta c^*_\theta$.
\end{proof}

When we extend our constraint solver beyond edge-density constraints to $H$-density constraints, we sample
$x_{i}^{(k)}$ for $k \in [N]$ and $i \in [v(H)]$ independently from $\mathrm{Unif}[0,1]$, and estimate the $H$-density constraint by the Monte Carlo estimator
\begin{align*}
\widehat{t}(H, f_\theta[c])
&= \frac{1}{N|\mathcal{S}|}
\sum_{k \in [N]}\sum_{\pi \in \mathcal{S}}
\prod_{\{i, j\} \in E(H)} f_\theta[c]\Big(x_{\pi(i)}^{(k)}, x_{\pi(j)}^{(k)}\Big),
\end{align*}
where $\mathcal{S}$ is a set of permutations of $V(H)$.

\begin{lemma}[Implicit differentiation for the $H$-density estimator]
\label{lem:implicitdifferentiation:Hdensity}
Fix the samples \(\{x_i^{(k)} : k\in[N],\ i\in[v(H)]\}\) and permutation set \(\mathcal{S}\). Suppose \(c_\theta\) is chosen so that
\(\widehat{t}(H,f_\theta[c_\theta]) = p\), and that \(\partial_c\widehat{t}(H,f_\theta[c])\big|_{c=c_\theta} \neq 0\).
Then for all \(x \in [0,1]^2\),
\begin{align*}
\nabla_\theta f_\theta[c_\theta](x)
&= \frac{\partial f_\theta[c](x)}{\partial \theta}\Big|_{c=c_\theta}
\;-
\frac{\partial f_\theta[c](x)}{\partial c}\Big|_{c=c_\theta}
\cdot
\frac{\nabla_\theta\widehat{t}(H,f_\theta[c])\big|_{c=c_\theta}}{\partial_c\widehat{t}(H,f_\theta[c])\big|_{c=c_\theta}},
\end{align*}
where
\begin{align*}
&\nabla_\theta\widehat{t}(H,f_\theta[c])
\\
{} = &\frac{1}{N|\mathcal{S}|}
\sum_{k\in[N],\,\pi\in\mathcal{S}}
\sum_{\{i,j\}\in E(H)}
\Bigg(\prod_{\substack{\{i',j'\}\in E(H)\\ \{i',j'\}\neq\{i,j\}}}
f_\theta[c]\bigl(x_{\pi(i')}^{(k)},x_{\pi(j')}^{(k)}\bigr)\Bigg)
\nabla_\theta f_\theta[c](x_{\pi(i)}^{(k)},x_{\pi(j)}^{(k)}),
\\
&\partial_c\widehat{t}(H,f_\theta[c])
\\
{} = &\frac{1}{N|\mathcal{S}|}
\sum_{k\in[N],\,\pi\in\mathcal{S}}
\sum_{\{i,j\}\in E(H)}
\Bigg(\prod_{\substack{\{i',j'\}\in E(H)\\ \{i',j'\}\neq\{i,j\}}}
f_\theta[c]\bigl(x_{\pi(i')}^{(k)},x_{\pi(j')}^{(k)}\bigr)\Bigg)
\frac{\partial f_\theta[c](x_{\pi(i)}^{(k)},x_{\pi(j)}^{(k)})}{\partial c}.
\end{align*}
\end{lemma}
\begin{proof}
Let \(x \in [0,1]^2\). By the chain rule,
\begin{align*}
\nabla_\theta\, f_\theta[c_\theta](x)
&= \frac{\partial f_\theta[c](x)}{\partial \theta}\Big|_{c=c_\theta}
 + \frac{\partial f_\theta[c](x)}{\partial c}\Big|_{c=c_\theta}\, \nabla_\theta c_\theta.
\end{align*}
Thus, it suffices to derive an expression for \(\nabla_\theta c_\theta\).

By assumption, \(c_\theta\) is chosen so that
\(\widehat{t}(H,f_\theta[c_\theta]) = p\).
Differentiating both sides with respect to \(\theta\) yields
\begin{align*}
\nabla_\theta\,\widehat{t}(H,f_\theta[c_\theta]) &= \nabla_\theta p = 0.
\end{align*}
Viewing \(\widehat{t}(H,f_\theta[c])\) as a function of \((\theta,c)\), another application of the chain rule gives
\begin{align*}
0
&= \nabla_\theta\,\widehat{t}(H,f_\theta[c])\Big|_{c=c_\theta}
 + \partial_c\widehat{t}(H,f_\theta[c])\Big|_{c=c_\theta}\, \nabla_\theta c_\theta.
\end{align*}
Since \(\partial_c\widehat{t}(H,f_\theta[c])\big|_{c=c_\theta} \neq 0\), we can solve for
\begin{align*}
\nabla_\theta c_\theta
&= {-}\,\frac{\nabla_\theta\widehat{t}(H,f_\theta[c])\big|_{c=c_\theta}}{\partial_c\widehat{t}(H,f_\theta[c])\big|_{c=c_\theta}}.
\end{align*}
Substituting this into the first display gives the claimed expression for
\(\nabla_\theta f_\theta[c_\theta](x)\).

Finally, we derive the displayed expressions for \(\nabla_\theta\widehat{t}(H,f_\theta[c])\) and
\(\partial_c\widehat{t}(H,f_\theta[c])\). For fixed \(k\in[N]\) and \(\pi\in\mathcal{S}\), write
\[
P_{k,\pi}(c) \coloneqq \prod_{\{i,j\}\in E(H)} f_\theta[c]\bigl(x_{\pi(i)}^{(k)},x_{\pi(j)}^{(k)}\bigr).
\]
Then \(\widehat{t}(H,f_\theta[c]) = \frac{1}{N|\mathcal{S}|}\sum_{k,\pi} P_{k,\pi}(c)\), and hence
\begin{align*}
&\nabla_\theta\widehat{t}(H,f_\theta[c])
\\
{} = &\frac{1}{N|\mathcal{S}|}\sum_{k\in[N],\,\pi\in\mathcal{S}} \nabla_\theta P_{k,\pi}(c)
\\
{} = &\frac{1}{N|\mathcal{S}|}\sum_{k\in[N],\,\pi\in\mathcal{S}}
\sum_{\{i,j\}\in E(H)}
\Bigg(\prod_{\substack{\{i',j'\}\in E(H)\\ \{i',j'\}\neq\{i,j\}}}
f_\theta[c]\bigl(x_{\pi(i')}^{(k)},x_{\pi(j')}^{(k)}\bigr)\Bigg)
\nabla_\theta f_\theta[c]\bigl(x_{\pi(i)}^{(k)},x_{\pi(j)}^{(k)}\bigr),
\end{align*}
where the second line is the product rule.
The expression for \(\partial_c\widehat{t}(H,f_\theta[c])\) is identical, replacing
\(\nabla_\theta f_\theta[c](\cdot)\) by \(\partial_c f_\theta[c](\cdot)\).
\end{proof}

\section{Details of experiments} \label{sec:experimentdetails}

All experiments were conducted on a single NVIDIA RTX A5000 GPU.
The execution time depends on the size of the graph, from 30 minutes for a triangle to a few hours for the Petersen graph.

In all experiments, we used a network of width $d=256$ and depth $L=6$, which was sufficient to represent sharp boundaries and structural properties.
The activation function $\phi$ was either $\sin$ or the GELU nonlinearity; we did not observe a consistent advantage of one over the other.
In all experiments, we used the increasing input-encoding scale $s(\ell)=2^{\ell-1}$. The constant-scale schedules in the ablation study are described in \Cref{sec:ablation}.

The set $\mathcal{S}$ in \Cref{eqn:estimator} was chosen according to the graph. For example, for complete graphs we used $\mathcal{S}=\{\mathrm{id}\}$, since any permutation yields the same summand.
For cycles, we used a set of representatives for the cosets $S_{V(C_k)} / \text{Aut}(C_k)$, where $\text{Aut}(C_k)$ is the automorphism group of $C_k$, which gives $|\mathcal{S}| = (k-1)! / 2$.
For the other graphs, we used a set $\mathcal{S}=\{\pi^j : j = 0, 1, \ldots, m-1\}$ generated by a permutation $\pi \in S_{V(H)}$, where $\pi$ and $m$ are chosen so that the set
\[
  \{\{\pi^j(v), \pi^j(w)\} : \{v, w\} \in E(H) \text{ and } j = 0, 1, \ldots, m-1\}
\]
equals the set of all unordered pairs of distinct vertices $\{\{v, w\} : v, w \in V(H) \text{ and } v \neq w\}$.

For training, we used the largest number of samples that fit in GPU memory, ranging from $2^{12}$ to $2^{16}$.
For large graphs such as the Petersen graph, we repeated multiple forward passes to accumulate gradients, yielding an effective sample size of $2^{16}$.
For evaluation, we used $2^{28}$ evaluation samples.
With caching of pairwise network evaluations, one training iteration takes $O(Nv(H)^2Ld^2 + N|\mathcal{S}|e(H))$ time and $O(Nv(H)^2Ld)$ memory. Larger motifs therefore reduce the feasible batch size under a fixed memory budget. The Monte Carlo estimator retains the usual $O(N^{-1/2})$ absolute-error rate for bounded integrands, but relative error becomes much harder to control when the true density is extremely small.

For optimisation, we used the Adam optimiser with $\beta_1=0.9$ and $\beta_2=0.999$.
We repeated the experiments with learning rates $\eta \in \{10^{-3}, 5 \times 10^{-4}, 10^{-4}, 5 \times 10^{-5}\}$, and set the minimum learning rate in the scheduler to 1\% of the initial learning rate.
We trained for 20,000 epochs, corresponding to five cosine-annealing cycles.

For the constraint solver, we found that applying Newton's method directly (without modification) can be unstable, owing to vanishing derivatives in many regimes.
In particular, since many optimal graphons take values close to 0 or 1, the sigmoid derivative $\mathit{sm}(x)\,(1-\mathit{sm}(x))$ is close to zero over much of the domain.
Consequently, the denominator in the Newton update can become very small, leading to excessively large updates and unstable convergence.
To address this, we use a clipped variant of Newton's method. To set the clipping range, we compute upper and lower bounds on the solution $c^*$ and use their difference as the initial clipping range.
We then gradually decrease the clipping range as the iteration proceeds, which improves stability.
The detailed algorithms are described in \Cref{alg:cap} (edge-density constraint) and \Cref{alg:cap-hom} (homomorphism-density constraint).

\begin{algorithm}
  \caption{Edge-Density Constraint Solver}\label{alg:cap}
  \begin{algorithmic}
    \REQUIRE $p \in (0, 1), y \in \R^N$
    \ENSURE $\frac{1}{N} \sum_{i=1}^N \mathit{sm}(y_i + c^*) \approxeq p$
    \STATE $c_1 \gets \mathit{sm}^{-1}(p) - \frac{1}{N} \sum_{i=1}^N y_i$ \hfill \COMMENT{Initialise with guess}
    \STATE $D \gets \max_i \{y_i\} - \min_i \{y_i\}$ \hfill \COMMENT{Compute range of $y$ values for clipping}
    \FOR{$t=1, 2, \ldots, T_{\max}$}
      \STATE $f(c_t) \gets \frac{1}{N} \sum_{i=1}^N \mathit{sm}(y_i + c_t) - p$
      \IF{$|f(c_t)| < \epsilon$}
        \STATE \textbf{break}
      \ENDIF
      \STATE $f'(c_t) \gets \frac{1}{N} \sum_{i=1}^N \mathit{sm}(y_i + c_t) (1 - \mathit{sm}(y_i + c_t))$
      \STATE $c_{t+1} \gets c_t - \mathrm{Clip}_{[-D/t, D/t]}\left(\frac{f(c_t)}{f'(c_t)}\right)$ \hfill \COMMENT{Clipping range $D/t$ shrinks with the iteration index $t$.}
    \ENDFOR
    \STATE \textbf{Return} $c_t$
  \end{algorithmic}
\end{algorithm}

\begin{algorithm}
  \caption{Homomorphism-Density Constraint Solver}\label{alg:cap-hom}
  \begin{algorithmic}
    \REQUIRE $p \in (0,1)$, $y \in \R^{B \times m}$
    \ENSURE $\frac{1}{B} \sum_{b=1}^B \prod_{j=1}^m \mathit{sm}(y_{b,j} + c^*) \approxeq p$
    \STATE $c_1 \gets \mathit{sm}^{-1}(p^{1/m}) - \frac{1}{Bm} \sum_{b=1}^B \sum_{j=1}^m y_{b,j}$ \hfill \COMMENT{Initialise with guess}
    \STATE $D \gets \max_{b,j}\{y_{b,j}\} - \min_{b,j}\{y_{b,j}\}$ \hfill \COMMENT{Compute range of $y$ values for clipping}
    \FOR{$t=1, 2, \ldots, T_{\max}$}
      \STATE $S \gets \bigl(\mathit{sm}(y_{b,j} + c_t)\bigr)_{b \in [B],\, j \in [m]}$
      \STATE $P \gets \Bigl(\prod_{j=1}^m S_{b,j}\Bigr)_{b \in [B]}$
      \STATE $f(c_t) \gets \frac{1}{B} \sum_{b=1}^B P_b - p$
      \IF{$|f(c_t)| < \epsilon$}
        \STATE \textbf{break}
      \ENDIF
      \STATE $f'(c_t) \gets \frac{1}{B} \sum_{b=1}^B P_b\,\sum_{j=1}^m (1 - S_{b,j})$
      \STATE $c_{t+1} \gets c_t - \mathrm{Clip}_{[-D/t,\, D/t]}\left(\frac{f(c_t)}{f'(c_t)}\right)$
    \ENDFOR
    \STATE \textbf{Return} $c_t$
  \end{algorithmic}
\end{algorithm}

For the variants in the ablation study, we performed hyperparameter searches for each variant.
When using constant per-layer input encoding scale $s$, we searched over $s(\ell) \in \{1, 2, 4, 8\}$.
For the strength of the regulariser, we searched over $\lambda \in \{0.01, 0.1, 1, 10, 100\}$.
For the constant learning rate, we used two values, $10^{-4}$ (low) and $10^{-3}$ (high).

When reporting the homomorphism density computed by discretisation, we use a $(2^{14} \times 2^{14})$ grid for $K_3, C_5, C_7$ and a $(2^{10} \times 2^{10})$ grid for $H_6$, and 64-bit floating point precision.

The code to replicate our experiments is available at \url{https://github.com/mekty2012/INR4GraphonVariationalProblem}.

\section{Additional experimental results}
\label{sec:addexperiments}

We consider the problem of minimising the homomorphism density of the Petersen graph under the edge-density constraint $t(K_2, W) = p$ for various values of $p$.
For minimising the Petersen homomorphism density, there are (to the best of our knowledge) no known results or conjectures, likely due to the graph size and the lack of an obvious extremal structure.
We applied our method for $p \in \{1/2, 2/3, 4/5, 4/7, 5/8, 5/7\}$; the learnt graphons are shown in \Cref{fig:petersenminimisation}.
As there is no known conjectured optimum, we compare our results to those obtained by a hand-crafted graphon motivated by~\citet{razborov2008minimal}; see \Cref{sec:handcraftedgraphons} for details.

\begin{table}[t]
    \centering
  \caption{Homomorphism densities of the Petersen graph in graphons under the edge-density constraint $t(K_2, W)=p$, found by our method and by a hand-crafted construction. \label{tab:petersenresult2}}
  {\small
    \begin{tabular}{cll}
         \toprule
         $\mathbf{p}$ & \textbf{Ours (Estimated)} & \textbf{Hand-crafted}
         \\ \midrule
         $1/2$ & $(\mathit{0.00000000}, \mathit{0.00000000})$ & $0$
         \\
         $2/3$ & $(\mathit{0.00184293}, \mathit{0.00184619})$ & $0.0020322$
         \\
         $4/5$ & $(\mathit{0.03412074}, \mathit{0.03413641})$ & $0.0340869$
         \\
         $4/7$ & $(\mathit{0.00006195}, \mathit{0.00006209})$ & $0.0001736$
         \\
         $5/8$ & $(\mathit{0.00054270}, \mathit{0.00054310})$ & $0.0007531$
         \\
         $5/7$ & $(\mathit{0.00571939}, \mathit{0.00572465})$ & $0.0059452$
         \\
         \bottomrule
    \end{tabular}
    }
\end{table}

\begin{figure}[t]
  \centering
  \begin{subfigure}{0.3\textwidth}
    \centering
    \includegraphics[width=\textwidth]{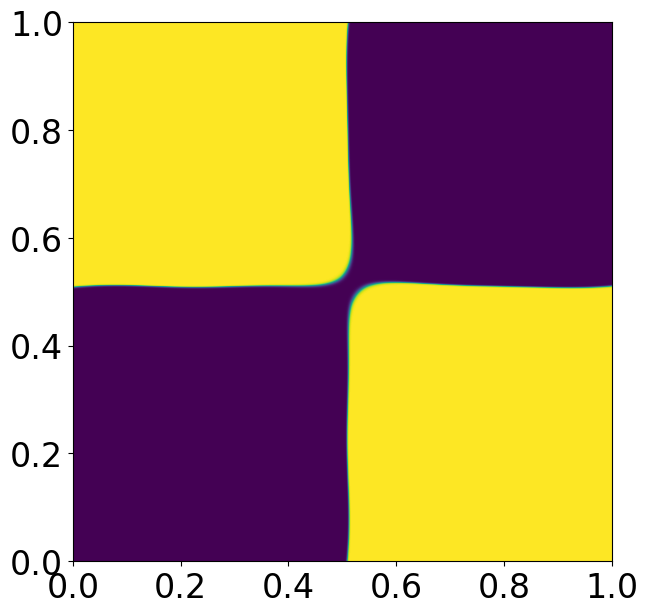}
    \caption{Found graphon, $p=1/2$.}
    \label{fig:opt-pet-50}
  \end{subfigure}
  \begin{subfigure}{0.3\textwidth}
    \centering
    \includegraphics[width=\textwidth]{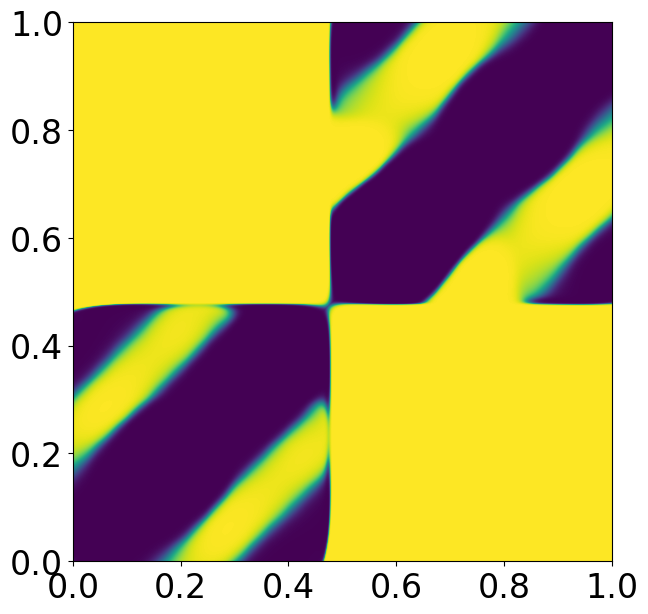}
    \caption{Found graphon, $p=2/3$.}
    \label{fig:opt-pet-67}
  \end{subfigure}
  \begin{subfigure}{0.3\textwidth}
    \centering
    \includegraphics[width=\textwidth]{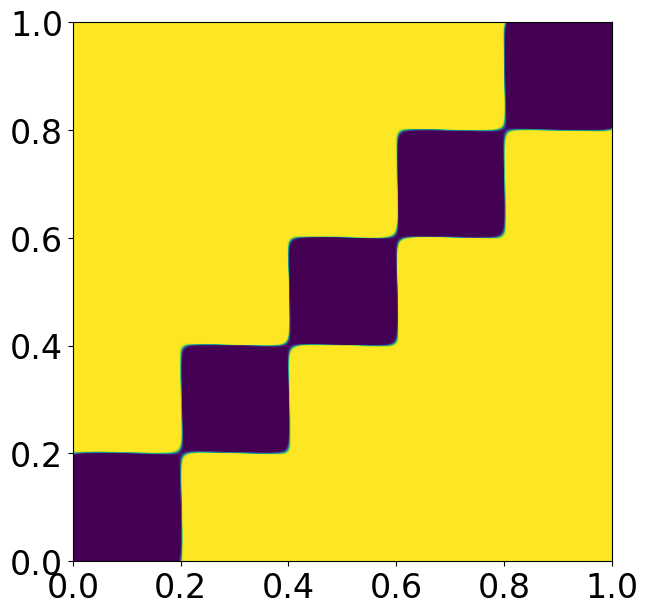}
    \caption{Found graphon, $p=4/5$.}
    \label{fig:opt-pet-80}
  \end{subfigure}

  \begin{subfigure}{0.3\textwidth}
    \centering
    \includegraphics[width=\textwidth]{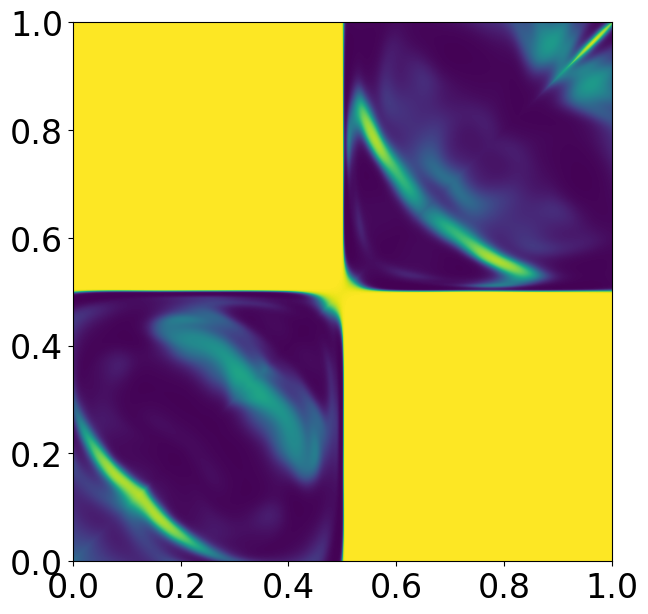}
    \caption{Found graphon, $p=4/7$.}
    \label{fig:opt-pet-57}
  \end{subfigure}
  \begin{subfigure}{0.3\textwidth}
    \centering
    \includegraphics[width=\textwidth]{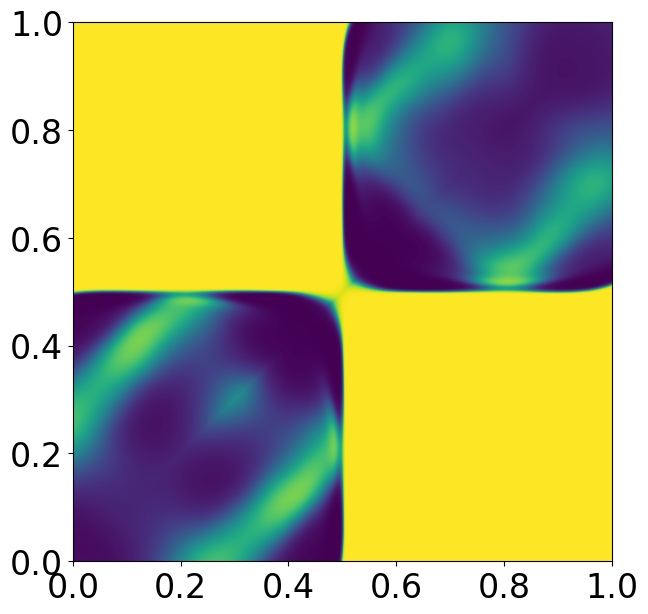}
    \caption{Found graphon, $p=5/8$.}
    \label{fig:opt-pet-62}
  \end{subfigure}
  \begin{subfigure}{0.3\textwidth}
    \centering
    \includegraphics[width=\textwidth]{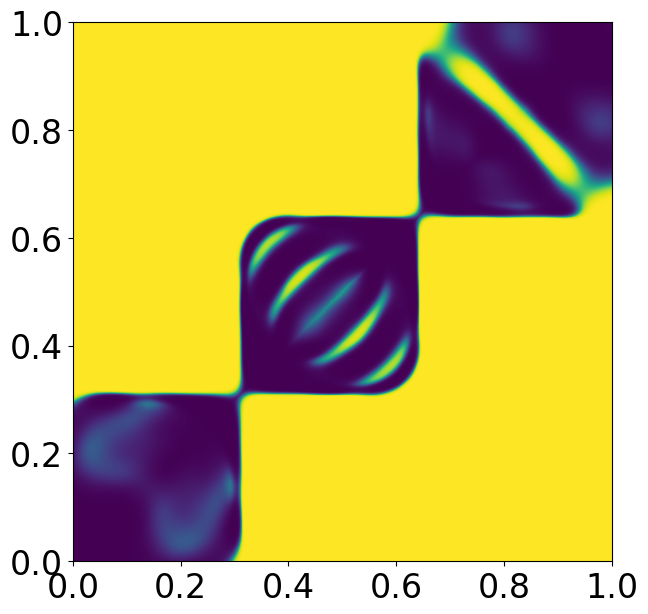}
    \caption{Found graphon, $p=5/7$.}
    \label{fig:opt-pet-71}
  \end{subfigure}
  \caption{Graphons found for the Petersen graph density minimisation under the edge-density constraint $t(K_2, W)=p$.}
  \label{fig:petersenminimisation}
\end{figure}

We also report the complete results of the variational problem discussed in \Cref{sec:experiments}; see \Cref{fig:k3upperdeviation,fig:k3lowerdeviation,fig:c4upperdeviation,fig:c5upperdeviation,fig:k13upperdeviation} for the found graphons.
We compute the results by solving the constraint after the discretisation of the learnt graphon, which gives a step graphon that satisfies the constraint exactly.
For the quantitative comparison, we include the tripartite construction $W_{\mathrm{LZ}}$ of \citet{lubetzky2015replica} only in the cases where this construction is available in the literature, namely the \emph{upper-tail} problem for regular graphs. Accordingly, \Cref{tab:upperdeviationcycle} compares $h_q(W_{\mathrm{ours}})$ with $h_q(W_{\mathrm{LZ}})$ for $H\in\{C_4,C_5\}$ under constraints $t(C_k,W)\ge r^k$. For the $K_{1,3}$ upper-tail problem, \Cref{tab:upperdeviationcycle} reports only our learnt solutions (no corresponding $W_{\mathrm{LZ}}$ construction is available in the literature for this case). For the \emph{lower-tail} problem for $K_3$, we likewise report only our learnt solutions; see \Cref{tab:lowerdeviation}.

When $H$ is $d$-regular, \citet{lubetzky2015replica} also prove a lower bound on the optimum of the upper-tail problem. Three graphs in our experiments, $K_3$, $C_4$, and $C_5$, satisfy this assumption, and the computed lower bound is reported in \Cref{tab:largedeviation,tab:upperdeviationcycle}; its values are the same for the three graphs, as the bound is determined by $d$.

\begin{table}[t]
  \centering
  \captionof{table}{Relative-entropy $h_q(W)$ at our learnt graphon ($W_{\mathrm{ours}}$) and, for $H\in\{C_4,C_5\}$, at the tripartite construction $W_{\mathrm{LZ}}$ of \citet{lubetzky2015replica}, for \textbf{(P2)} on $(q,r) \in \{0.05, 0.1\} \times \{0.4, 0.5, 0.6\}$. We also report a lower bound on the optimum for $H\in\{C_4,C_5\}$, and $W_{\mathrm{ours}}$ for $H=K_{1,3}$ (no corresponding $W_{\mathrm{LZ}}$ construction is available in the literature for this case). Lower is better.\label{tab:upperdeviationcycle}}
  {\small
  \begin{tabular}{lcccccc}
    \toprule
    & \multicolumn{3}{c}{$q=0.05$} & \multicolumn{3}{c}{$q=0.10$} \\
    \cmidrule(lr){2-4}\cmidrule(lr){5-7}
    $r$ & $0.4$ & $0.5$ & $0.6$ & $0.4$ & $0.5$ & $0.6$ \\
    \midrule
    $h_q(W_{\mathrm{LZ}})$, $H=C_4$    & $0.556057$ & $0.830366$ & $1.144945$ & $0.311239$ & $0.510826$ & $0.750684$ \\
    $h_q(W_{\mathrm{ours}})$, $H=C_4$  & $0.464162$ & $0.727364$ & $1.050351$ & $0.309184$ & $0.506141$ & $0.747042$ \\
    \midrule
    $h_q(W_{\mathrm{LZ}})$, $H=C_5$    & $0.556057$ & $0.830366$ & $1.144945$ & $0.311239$ & $0.510826$ & $0.750684$ \\
    $h_q(W_{\mathrm{ours}})$, $H=C_5$  & $0.465644$ & $0.727618$ & $1.052023$ & $0.308858$ & $0.506122$ & $0.746845$ \\
    \midrule
    Lower bound, $H\in\{C_4,C_5\}$    & $0.460284$ & $0.725284$ & $1.049172$ & $0.306561$ & $0.504311$ & $0.746006$ \\
    \midrule
    $h_q(W_{\mathrm{ours}})$, $H=K_{1,3}$    & $0.357819$ & $0.665939$ & $1.076113$ & $0.263666$ & $0.484864$ & $0.748724$ \\
    \bottomrule
  \end{tabular}
  }
\end{table}

\begin{table}[t]
  \centering
  \captionof{table}{Relative-entropy $h_q(W)$ at our learnt graphon ($W_{\mathrm{ours}}$) for \textbf{(P2)} with $H = K_3$ on $(q,r) \in \{0.7, 0.8, 0.9\} \times \{0.05, 0.1\}$. Lower is better.\label{tab:lowerdeviation}}
  {\small
  \begin{tabular}{lcccccc}
    \toprule
    & \multicolumn{3}{c}{$r=0.05$} & \multicolumn{3}{c}{$r=0.10$} \\
    \cmidrule(lr){2-4}\cmidrule(lr){5-7}
    $q$ & $0.7$ & $0.8$ & $0.9$ & $0.7$ & $0.8$ & $0.9$ \\
    \midrule
    $h_q(W_{\mathrm{ours}})$  & $0.600882$ & $0.591680$ & $0.804127$ & $0.796391$ & $1.151161$ & $1.144472$ \\
    \bottomrule
  \end{tabular}
  }
\end{table}

\begin{figure}
  \centering
  \begin{subfigure}{0.3\textwidth}
    \centering
    \includegraphics[width=\textwidth]{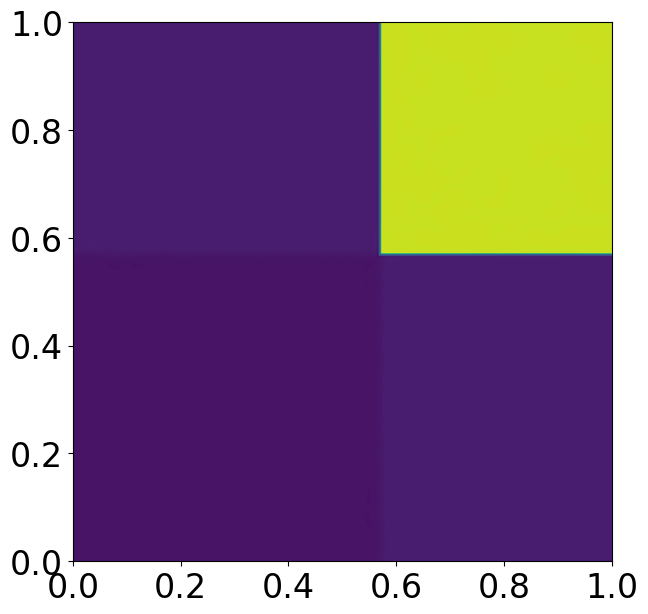}
  \end{subfigure}
  \begin{subfigure}{0.3\textwidth}
    \centering
    \includegraphics[width=\textwidth]{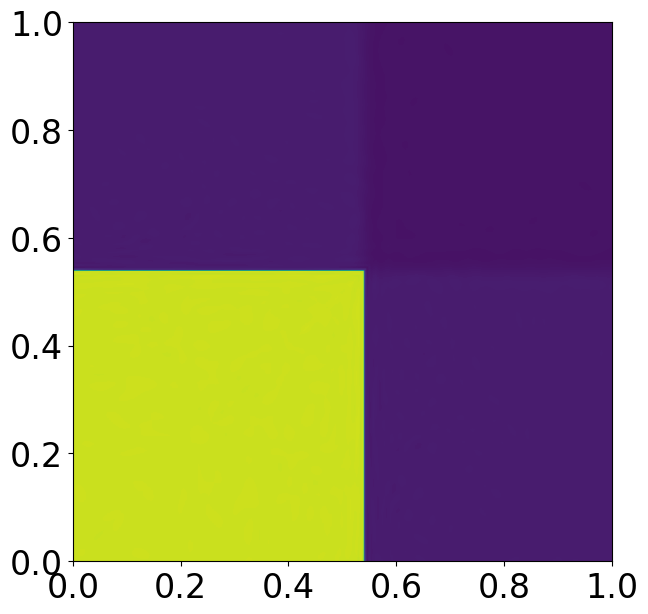}
  \end{subfigure}
  \begin{subfigure}{0.3\textwidth}
    \centering
    \includegraphics[width=\textwidth]{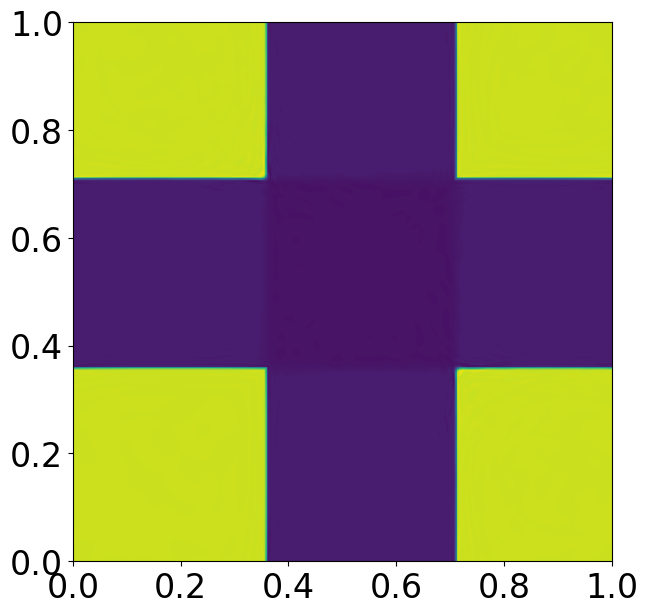}
  \end{subfigure}

  \begin{subfigure}{0.3\textwidth}
    \centering
    \includegraphics[width=\textwidth]{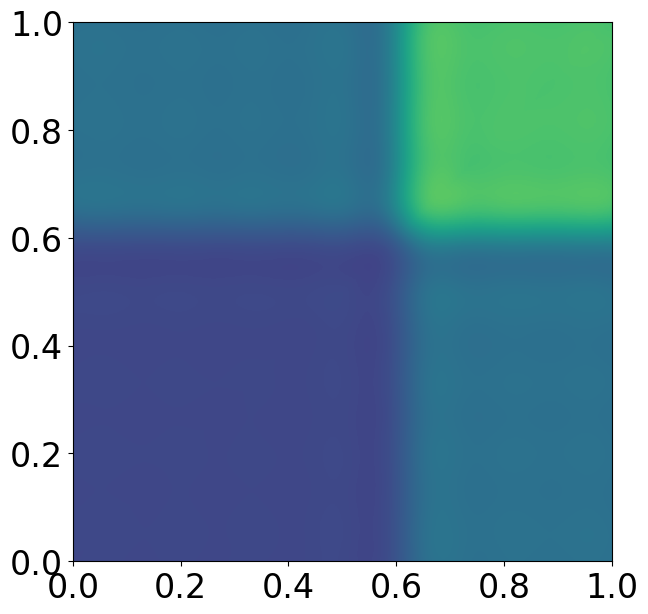}
  \end{subfigure}
  \begin{subfigure}{0.3\textwidth}
    \centering
    \includegraphics[width=\textwidth]{figures/upperdeviation_k3_10_50.png}
  \end{subfigure}
  \begin{subfigure}{0.3\textwidth}
    \centering
    \includegraphics[width=\textwidth]{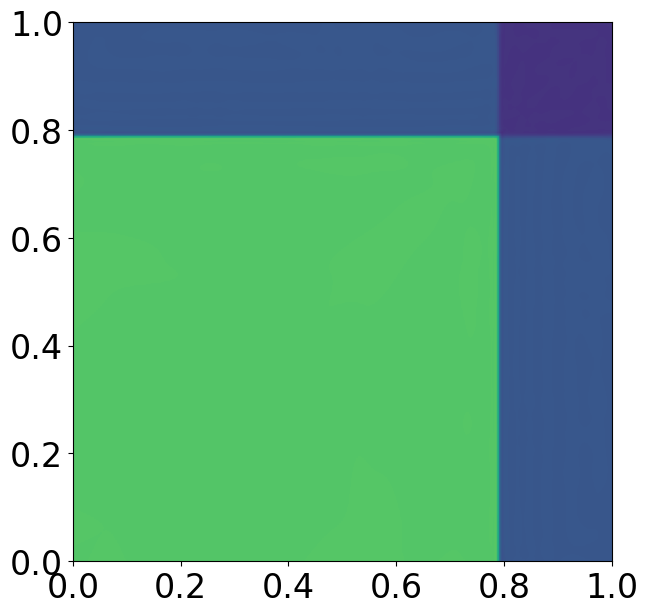}
  \end{subfigure}
  \caption{Found graphon for upper tail variational problems for $K_3$, for $q = 0.05$ (top) and $q = 0.1$ (bottom) and $r = 0.4, 0.5, 0.6$ (left to right).}
  \label{fig:k3upperdeviation}
\end{figure}

\begin{figure}
  \centering
  \begin{subfigure}{0.3\textwidth}
    \centering
    \includegraphics[width=\textwidth]{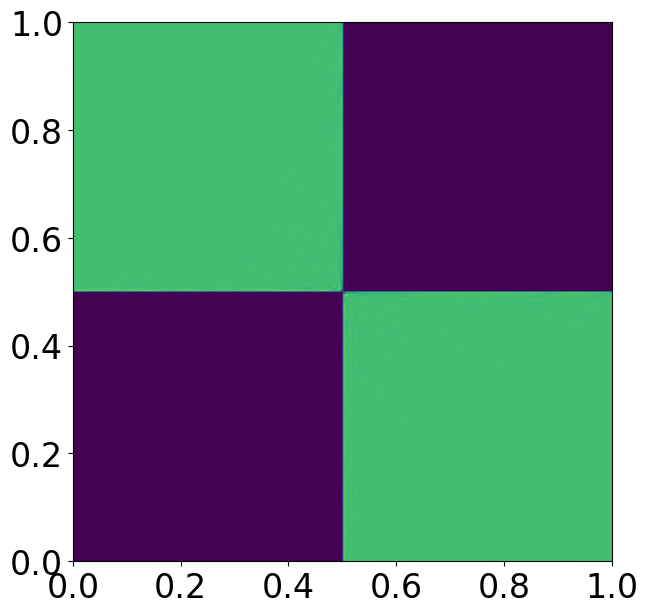}
  \end{subfigure}
  \begin{subfigure}{0.3\textwidth}
    \centering
    \includegraphics[width=\textwidth]{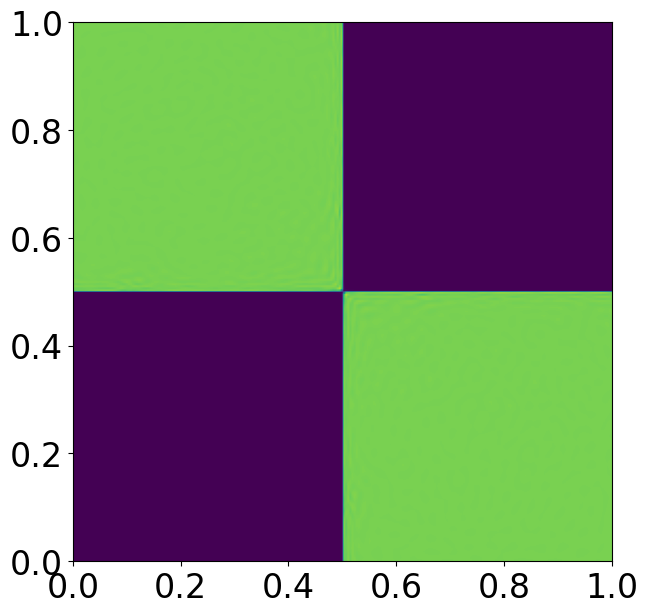}
  \end{subfigure}
  \begin{subfigure}{0.3\textwidth}
    \centering
    \includegraphics[width=\textwidth]{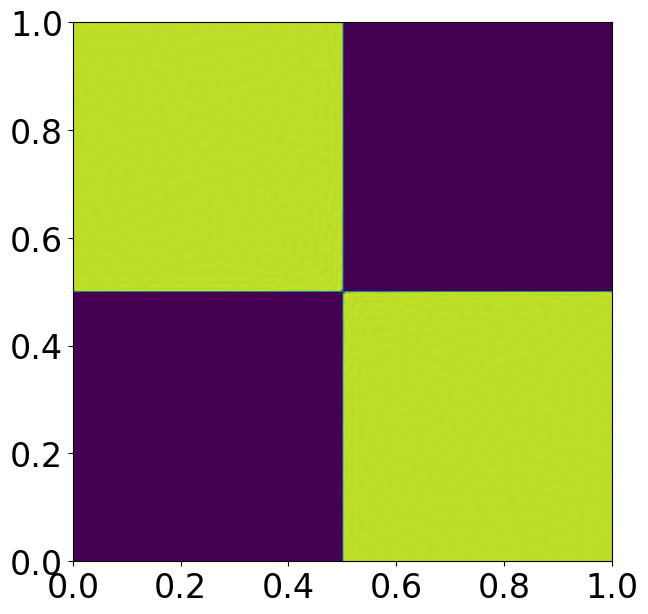}
  \end{subfigure}

  \begin{subfigure}{0.3\textwidth}
    \centering
    \includegraphics[width=\textwidth]{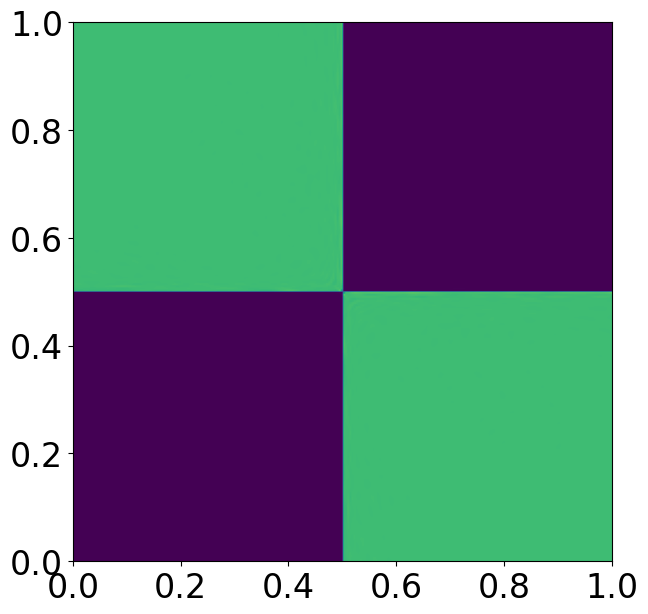}
  \end{subfigure}
  \begin{subfigure}{0.3\textwidth}
    \centering
    \includegraphics[width=\textwidth]{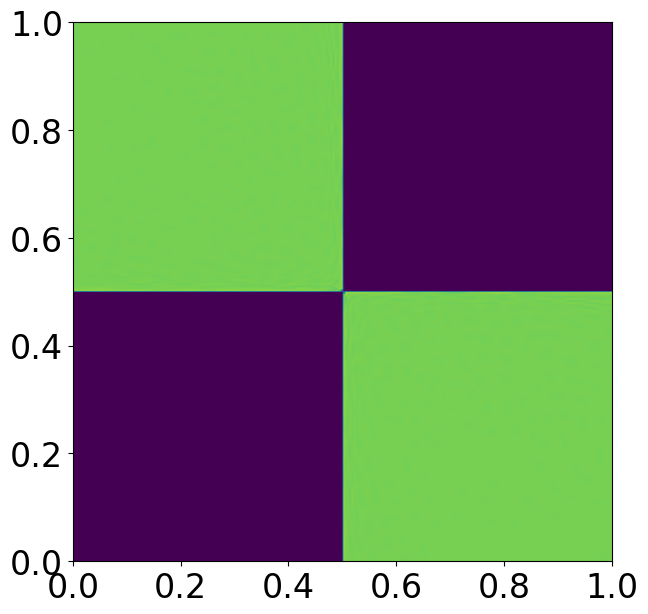}
  \end{subfigure}
  \begin{subfigure}{0.3\textwidth}
    \centering
    \includegraphics[width=\textwidth]{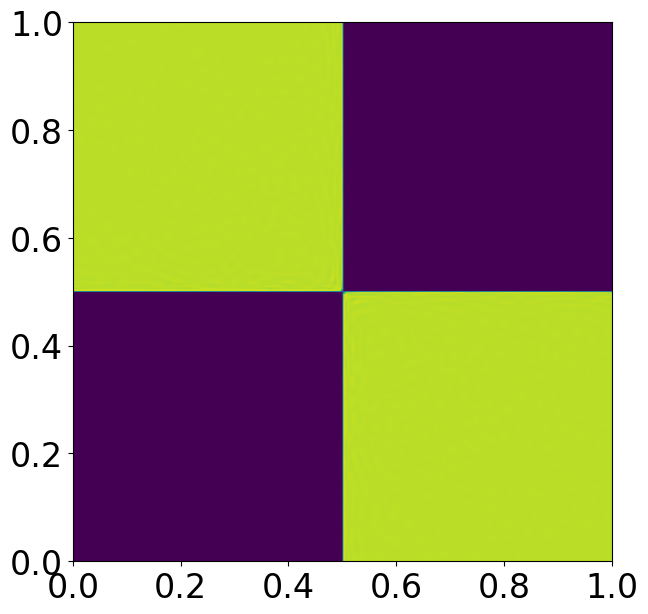}
  \end{subfigure}
  \caption{Found graphon for lower tail variational problems for $K_3$, for $r = 0.05$ (top) and $r = 0.1$ (bottom) and $q = 0.7, 0.8, 0.9$ (left to right).}
  \label{fig:k3lowerdeviation}
\end{figure}

\begin{figure}
  \centering
  \begin{subfigure}{0.3\textwidth}
    \centering
    \includegraphics[width=\textwidth]{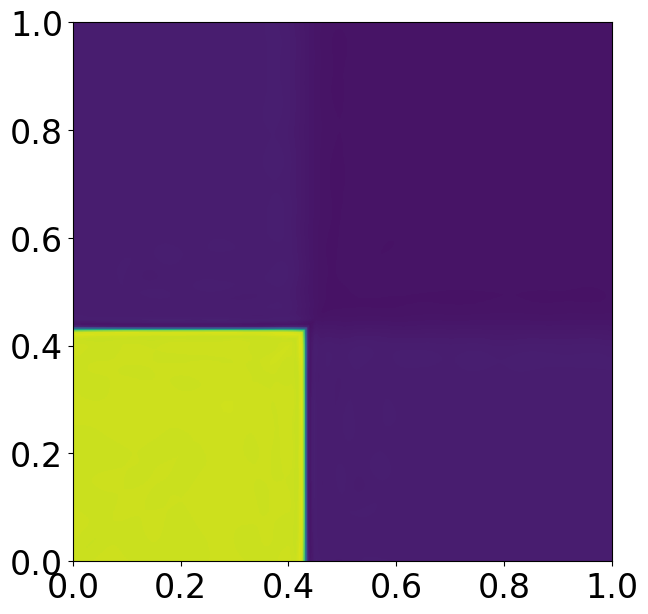}
  \end{subfigure}
  \begin{subfigure}{0.3\textwidth}
    \centering
    \includegraphics[width=\textwidth]{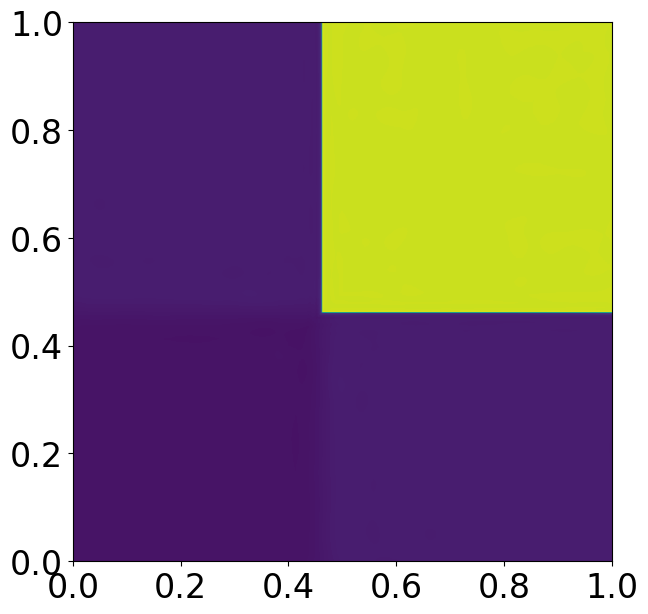}
  \end{subfigure}
  \begin{subfigure}{0.3\textwidth}
    \centering
    \includegraphics[width=\textwidth]{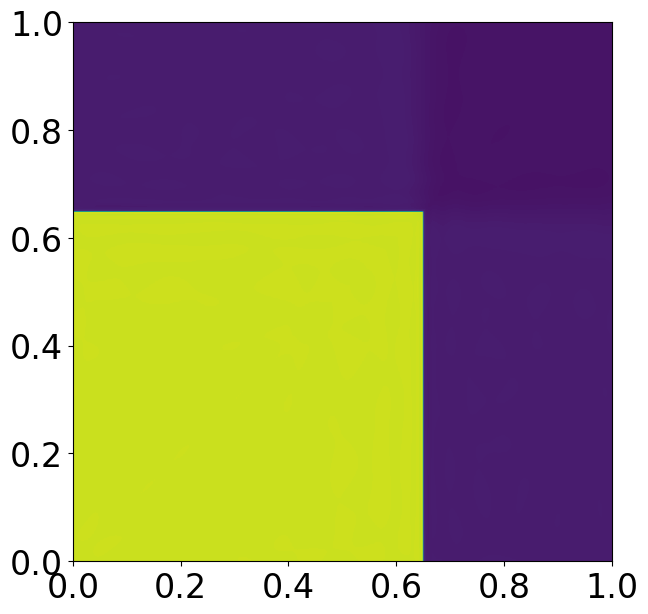}
  \end{subfigure}

  \begin{subfigure}{0.3\textwidth}
    \centering
    \includegraphics[width=\textwidth]{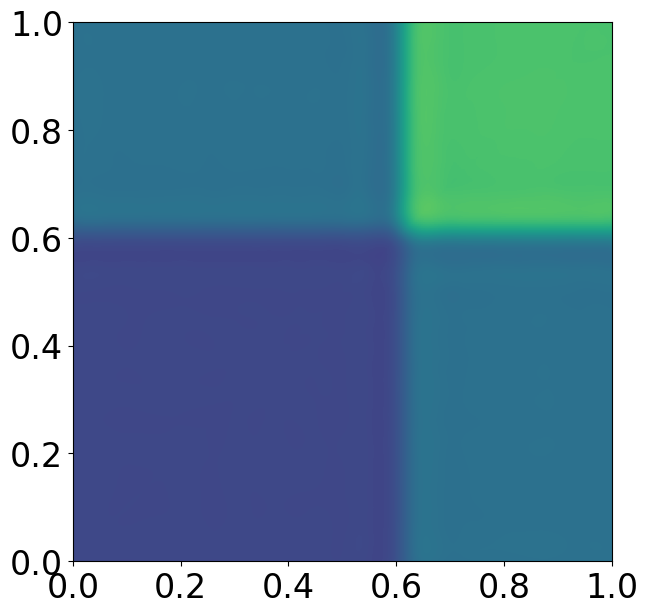}
  \end{subfigure}
  \begin{subfigure}{0.3\textwidth}
    \centering
    \includegraphics[width=\textwidth]{figures/upperdeviation_c4_10_50.png}
  \end{subfigure}
  \begin{subfigure}{0.3\textwidth}
    \centering
    \includegraphics[width=\textwidth]{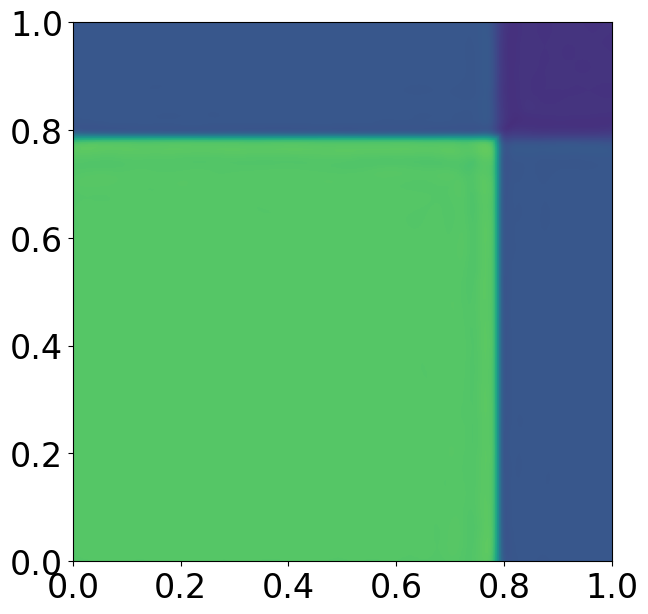}
  \end{subfigure}
  \caption{Found graphon for upper tail variational problems for $C_4$, for $q = 0.05$ (top) and $q = 0.1$ (bottom) and $r = 0.4, 0.5, 0.6$ (left to right).}
  \label{fig:c4upperdeviation}
\end{figure}

\begin{figure}
  \centering
  \begin{subfigure}{0.3\textwidth}
    \centering
    \includegraphics[width=\textwidth]{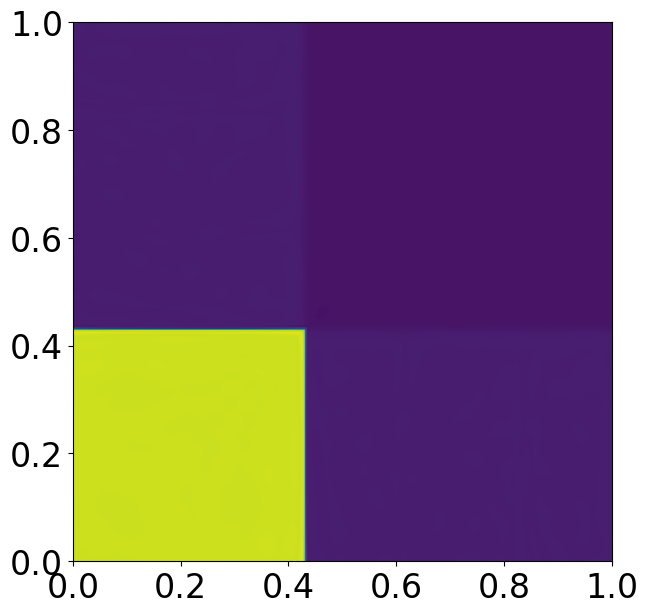}
  \end{subfigure}
  \begin{subfigure}{0.3\textwidth}
    \centering
    \includegraphics[width=\textwidth]{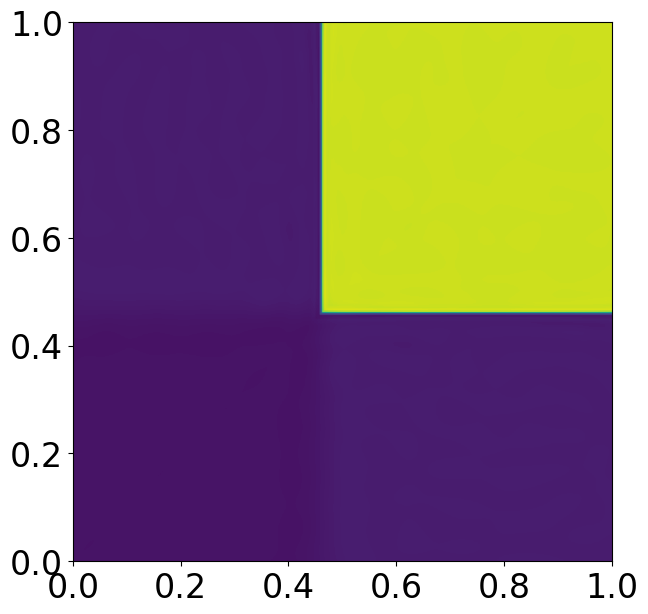}
  \end{subfigure}
  \begin{subfigure}{0.3\textwidth}
    \centering
    \includegraphics[width=\textwidth]{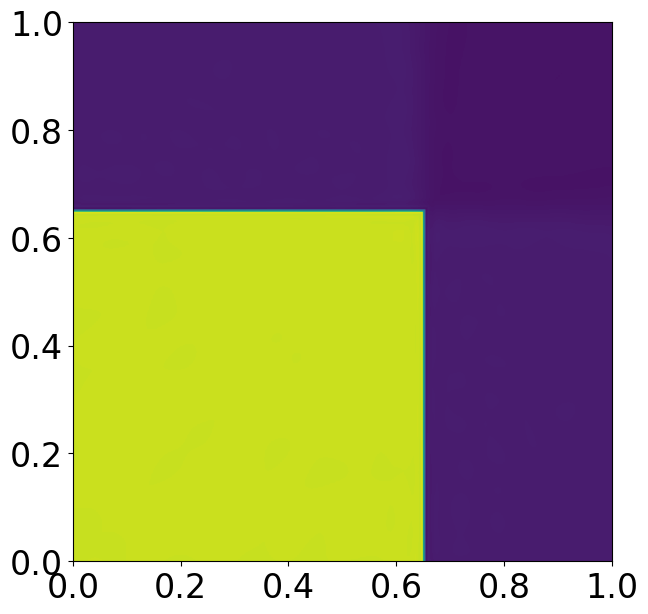}
  \end{subfigure}

  \begin{subfigure}{0.3\textwidth}
    \centering
    \includegraphics[width=\textwidth]{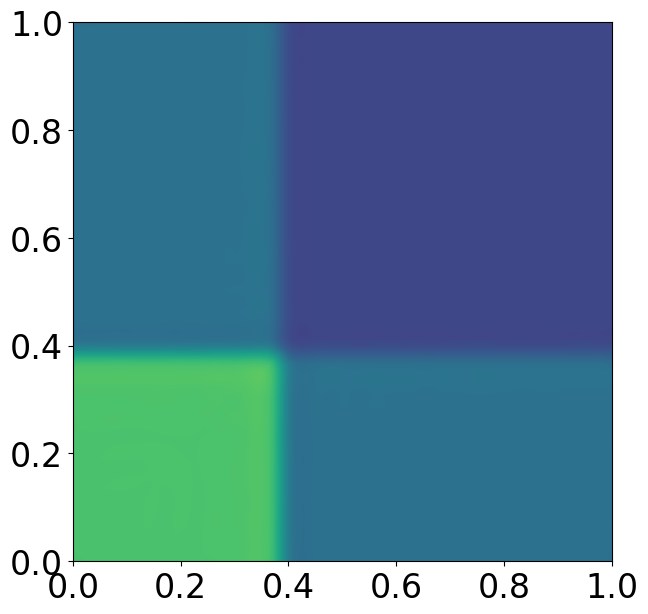}
  \end{subfigure}
  \begin{subfigure}{0.3\textwidth}
    \centering
    \includegraphics[width=\textwidth]{figures/upperdeviation_c5_10_50.png}
  \end{subfigure}
  \begin{subfigure}{0.3\textwidth}
    \centering
    \includegraphics[width=\textwidth]{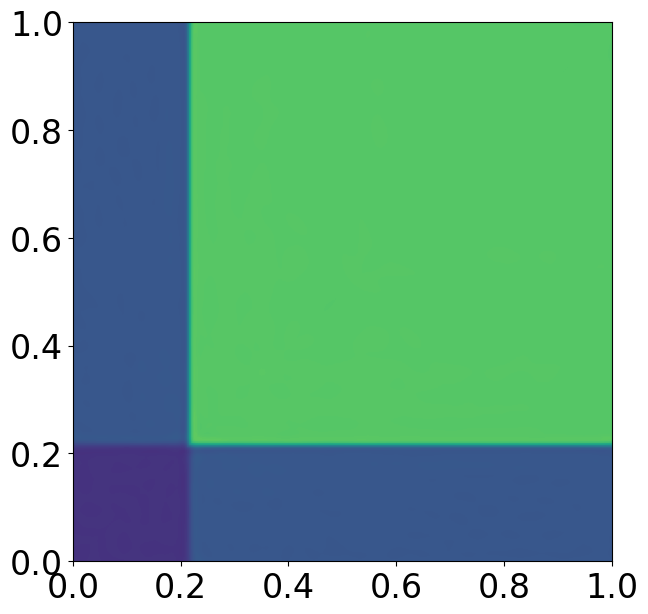}
  \end{subfigure}
  \caption{Found graphon for upper tail variational problems for $C_5$, for $q = 0.05$ (top) and $q = 0.1$ (bottom) and $r = 0.4, 0.5, 0.6$ (left to right).}
  \label{fig:c5upperdeviation}
\end{figure}

\begin{figure}
  \centering
  \begin{subfigure}{0.3\textwidth}
    \centering
    \includegraphics[width=\textwidth]{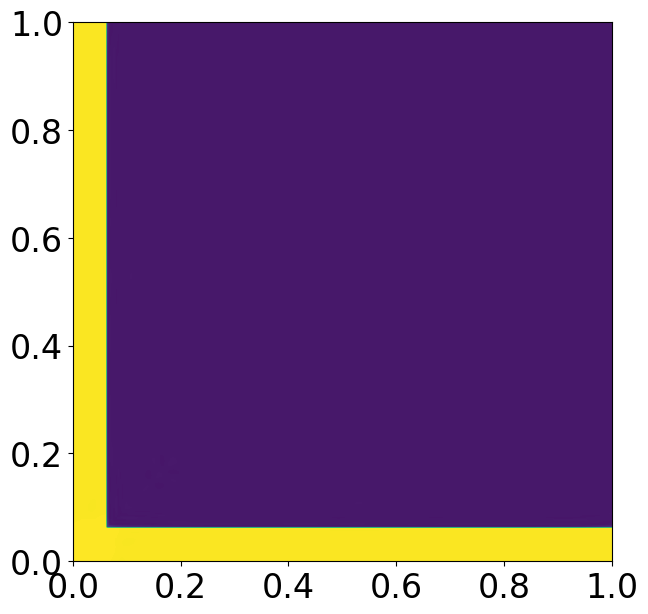}
  \end{subfigure}
  \begin{subfigure}{0.3\textwidth}
    \centering
    \includegraphics[width=\textwidth]{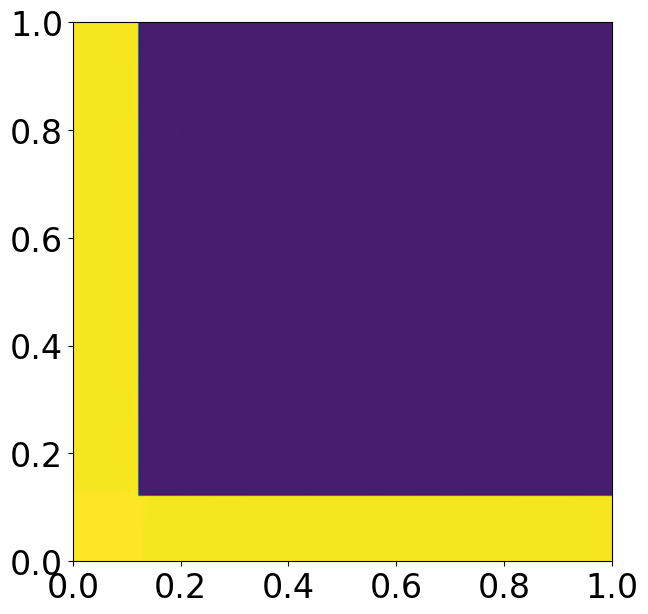}
  \end{subfigure}
  \begin{subfigure}{0.3\textwidth}
    \centering
    \includegraphics[width=\textwidth]{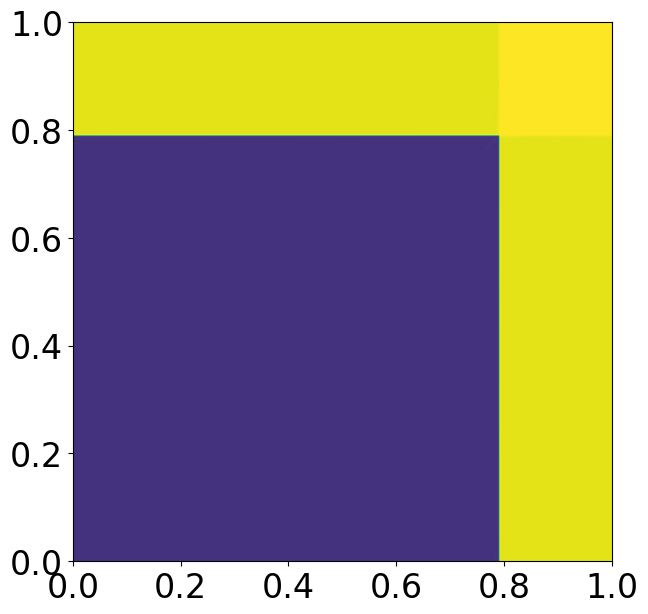}
  \end{subfigure}

  \begin{subfigure}{0.3\textwidth}
    \centering
    \includegraphics[width=\textwidth]{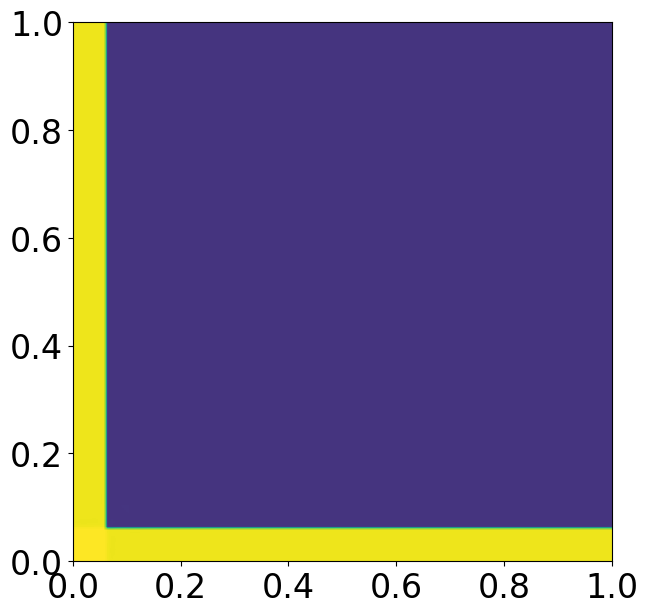}
  \end{subfigure}
  \begin{subfigure}{0.3\textwidth}
    \centering
    \includegraphics[width=\textwidth]{figures/upperdeviation_k13_10_50.png}
  \end{subfigure}
  \begin{subfigure}{0.3\textwidth}
    \centering
    \includegraphics[width=\textwidth]{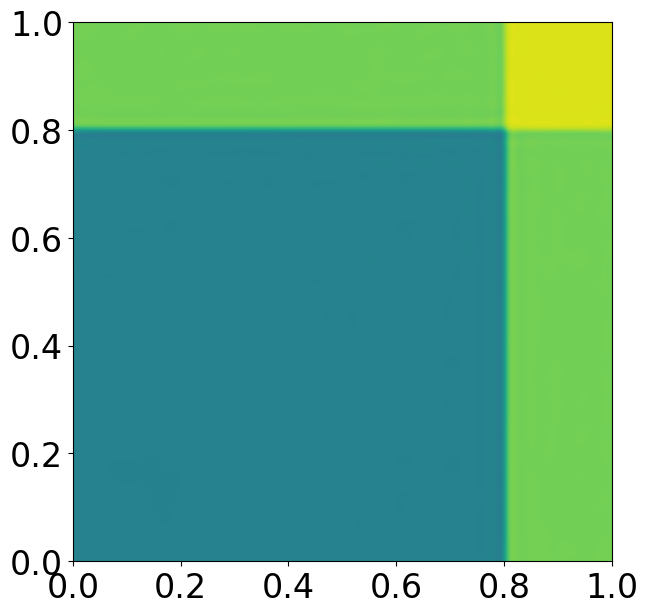}
  \end{subfigure}
  \caption{Found graphon for upper tail variational problems for $K_{1,3}$, for $q = 0.05$ (top) and $q = 0.1$ (bottom) and $r = 0.4, 0.5, 0.6$ (left to right).}
  \label{fig:k13upperdeviation}
\end{figure}

\section{Details of hand-crafted graphons}
\label{sec:handcraftedgraphons}

We consider edge densities satisfying $1 - 1/k < p < 1 - 1/(k+1)$, where $k$ is a positive integer.

The first two instances are graphons arising from the optimal solution to edge-density-constrained $K_3$ homomorphism-density minimisation.
We construct a complete $(k+1)$-partite graphon by partitioning $[0,1]$ into $(k+1)$ groups, $X_1, \ldots, X_{k+1}$, and defining
\[
  W_i(x_1, x_2) = \begin{cases}
    1 & \text{if } x_1 \in X_m,\ x_2 \in X_n,\ m \neq n
    \\
    0 & \text{if } x_1 \in X_m,\ x_2 \in X_m.
  \end{cases}
\]
We set $|X_1| = \cdots = |X_k| = x$ and $|X_{k+1}| = 1 - kx$, and choose $x$ to satisfy the edge-density constraint. This amounts to solving the quadratic equation
\[
  k(k-1)x^2 + 2k x (1-kx) = p.
\]
This quadratic has two solutions: one with $x > 1 - kx$ (denoted $W_1$) and one with $x < 1 - kx$ (denoted $W_2$). The first corresponds to the optimum for $K_3$ density minimisation.

The next too instance is constructed by deleting edges from the complete $(k+1)$-partite graphon.
We partition the interval $[0,1]$ into $k+1$ groups $X_1, \ldots, X_k, X_{k+1}$ with $|X_1| = \cdots = |X_k| = |X_{k+1}| = 1/(k+1)$, and define
\begin{align*}
  W_3(x_1, x_2) &= \begin{cases}
    \frac{(k+1)p}{k} & \text{if } x_1 \in X_m,\ x_2 \in X_n,\ m \neq n
    \\
    0 & \text{if } x_1 \in X_m,\ x_2 \in X_m,
  \end{cases}
  \\
  W_4(x_1, x_2) &= \begin{cases}
    \frac{p (k+1)^2 - k^2 - k + 2}{2} & \text{if } x_1 \in X_1,\ x_2 \in X_2\ \text{or vice versa}
    \\
    0 & \text{if } x_1 \in X_m, x_2 \in X_m,
    \\
    1 & \text{otherwise}.
  \end{cases}
\end{align*}
For $1 - 1/k < p < 1 - 1/(k+1)$, both nontrivial edge probabilities above lie in $[0,1]$: their non-negativity follows from the lower bound on $p$, while $(k+1)p/k < 1$ and $\bigl(p(k+1)^2-k^2-k+2\bigr)/2 < 1$ follow from the upper bound.

Finally, the last graphon is the constant graphon $W(x_1, x_2) = p$.

We visualised $W_1, \ldots, W_4$ for $p=0.7$ in \Cref{fig:handcraftedgraphons}.

\begin{figure}
  \centering
  \begin{subfigure}{0.22\textwidth}
      \centering
      \includegraphics[width=\twoimagewidth]{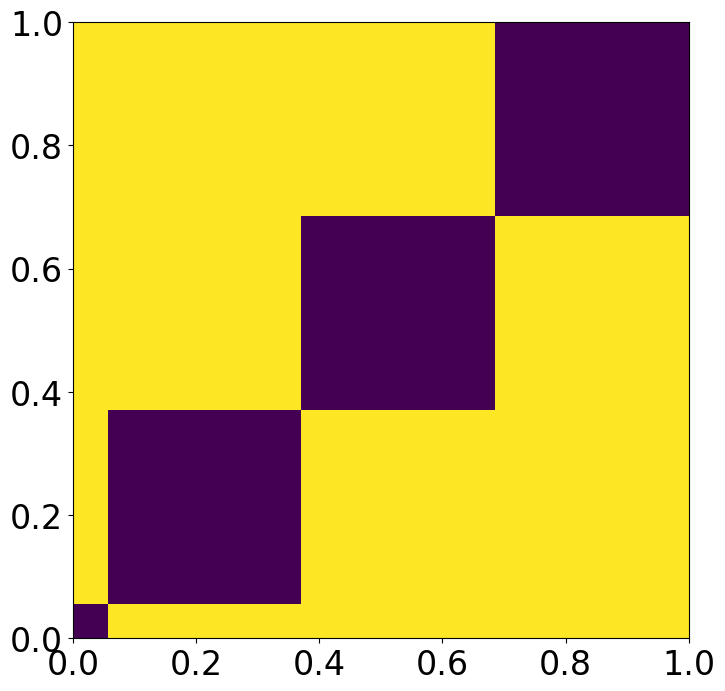}
  \end{subfigure}
  \begin{subfigure}{0.22\textwidth}
      \centering
      \includegraphics[width=\twoimagewidth]{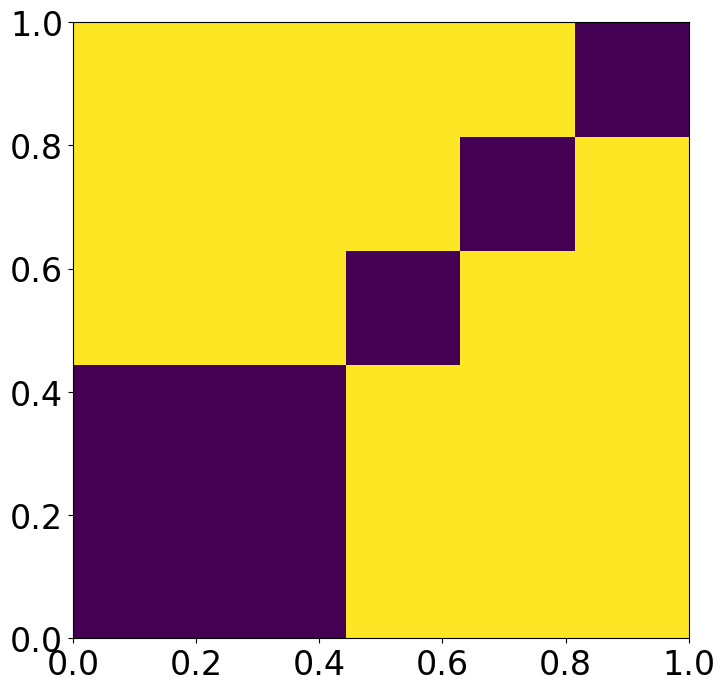}
  \end{subfigure}
  \begin{subfigure}{0.22\textwidth}
      \centering
      \includegraphics[width=\twoimagewidth]{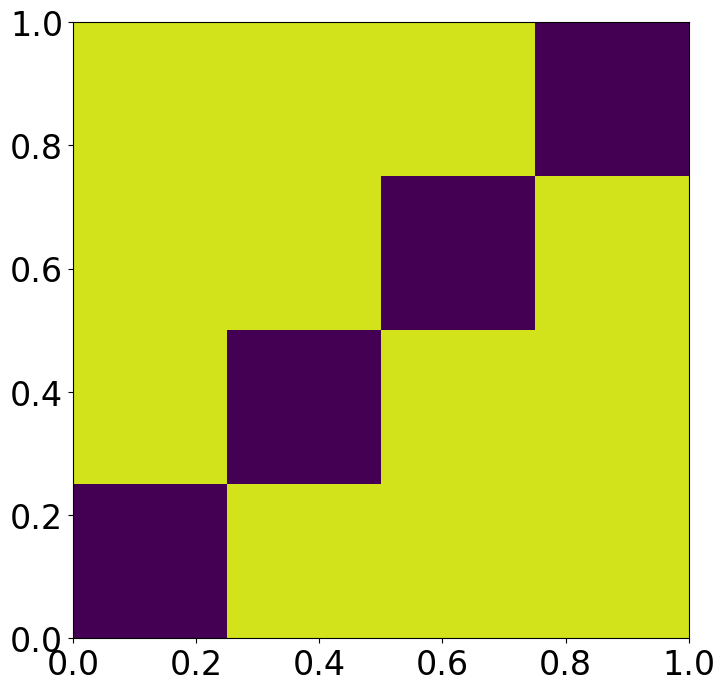}
  \end{subfigure}
  \begin{subfigure}{0.22\textwidth}
      \centering
      \includegraphics[width=\twoimagewidth]{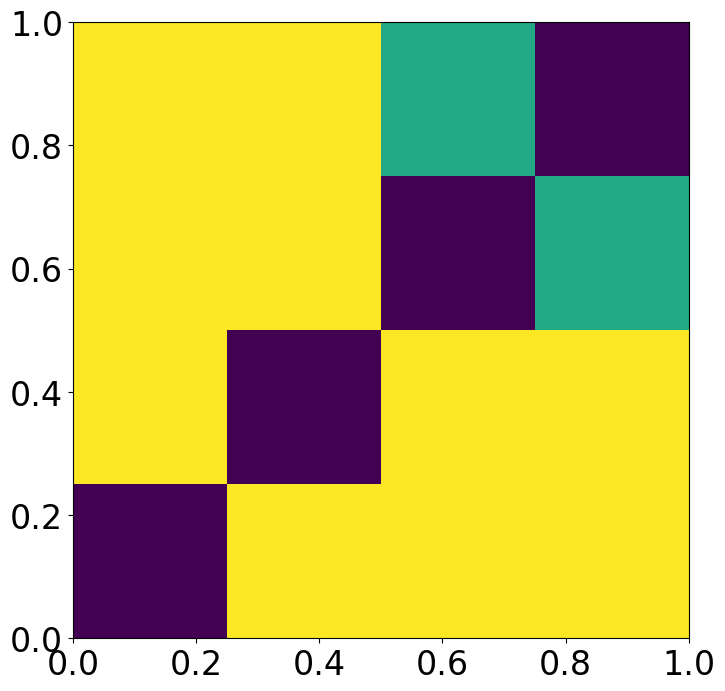}
  \end{subfigure}
  \caption{Hand-crafted graphons $W_1, W_2, W_3, W_4$ (left to right) with edge density $t(K_2, W)=0.7$.}
  \label{fig:handcraftedgraphons}
\end{figure}

\section{Ablation study, robustness, non-neural baselines, and multi-task variant}
\label{sec:ablation}

\paragraph{Ablation Study.}
We perform an ablation study on triangle-density minimisation under the constraint $t(K_2,W)=7/9$, which is not of the form $1-1/k$ and admits multiple types of optima.
We consider four variants: (1) replacing our architecture with a ResNet that uses sinusoidal input encoding only in the first layer; (2) using a constant per-layer input encoding scale $s$; (3) enforcing the constraint via a regulariser; and (4) using a constant learning rate instead of cosine-annealing scheduling.
For fair comparison, we performed hyperparameter searches for the constant value of $s$ in (2) and the strength of the regulariser in (3), and report the best run for each variant; for (4), we report the runs with a low and a high learning rate.
To separate the effect of the sinusoidal features from that of their per-layer injection, we also repeat our method and every variant with GELU in place of every sine function, keeping the per-layer injection; in the ResNet variant, the sinusoidal encoding of the first layer is kept.
As in \Cref{sec:addexperiments}, all values are computed after solving for the scalar bias $c$ on the discretised graphon.
The graphons found by these variants are shown in \Cref{fig:ablationstudy} and quantitative results in \Cref{tab:ablationstudy}.
With sinusoidal features, every variant reaches a lower $K_3$ density than with GELU. Since the GELU column also changes the activation $\phi$, we also compare with the ten runs in the robustness study below, which use sinusoidal features with $\phi = \mathrm{GELU}$: their mean $0.435661$ is also below the $0.436066$ of our method with GELU features.
The ResNet variant (1) and the constant-scale variant (2) often fail to recover sharp, straight boundaries, consistent with a bias toward smooth functions in (1) and toward high frequencies in (2).
With the penalty-based constraint (3), the learnt graphon often violates fine structural details (e.g., the size of the smallest block).
Without cosine-annealing scheduling (4), the learnt graphon either fails to capture the multi-block structure (low learning rates) or exhibits less stable boundaries (high learning rates), highlighting the benefit of alternating between higher and lower learning rates.

\begin{table}[t]
  \centering
  \caption{$K_3$ densities in graphons found by our method, ablations, and non-neural baselines for $K_3$ density minimisation under $t(K_2, W)=7/9$, with sinusoidal features or with GELU features in the per-layer encoding and activation. The optimum is $98/225 \approx 0.435556$.  \label{tab:ablationstudy}}
  {\small
  \begin{tabular}{lll}
  \toprule
                     & \multicolumn{2}{c}{$t(K_3, W)$} \\
  \cmidrule(lr){2-3}
                     & sin & GELU \\
  \midrule
  Ours                                     & $0.435712$ & $0.436066$ \\
  ResNet                                   & $0.435724$ & $0.436405$ \\
  Constant-scale $s$                       & $0.435689$ & $0.436347$ \\
  Regulariser                              & $0.437085$ & $0.437323$ \\
  Constant LR (low)                        & $0.437544$ & $0.442451$ \\
  Constant LR (high)                       & $0.436174$ & $0.436743$ \\
  \midrule
  SBM (256 blocks, fixed sizes)            & \multicolumn{2}{c}{$0.435817$} \\
  SBM (5 blocks, trainable sizes)          & \multicolumn{2}{c}{$0.470508$} \\
  SBM (64 blocks, trainable sizes)         & \multicolumn{2}{c}{$0.470508$} \\
  Decision Tree                            & \multicolumn{2}{c}{$0.455558$} \\
  Gradient Boosted Decision Tree           & \multicolumn{2}{c}{$0.454485$} \\
  \bottomrule
  \end{tabular}
  }
\end{table}

\begin{figure}[t]
  \centering
  \begin{subfigure}{0.25\textwidth}
    \centering
    \includegraphics[width=\textwidth]{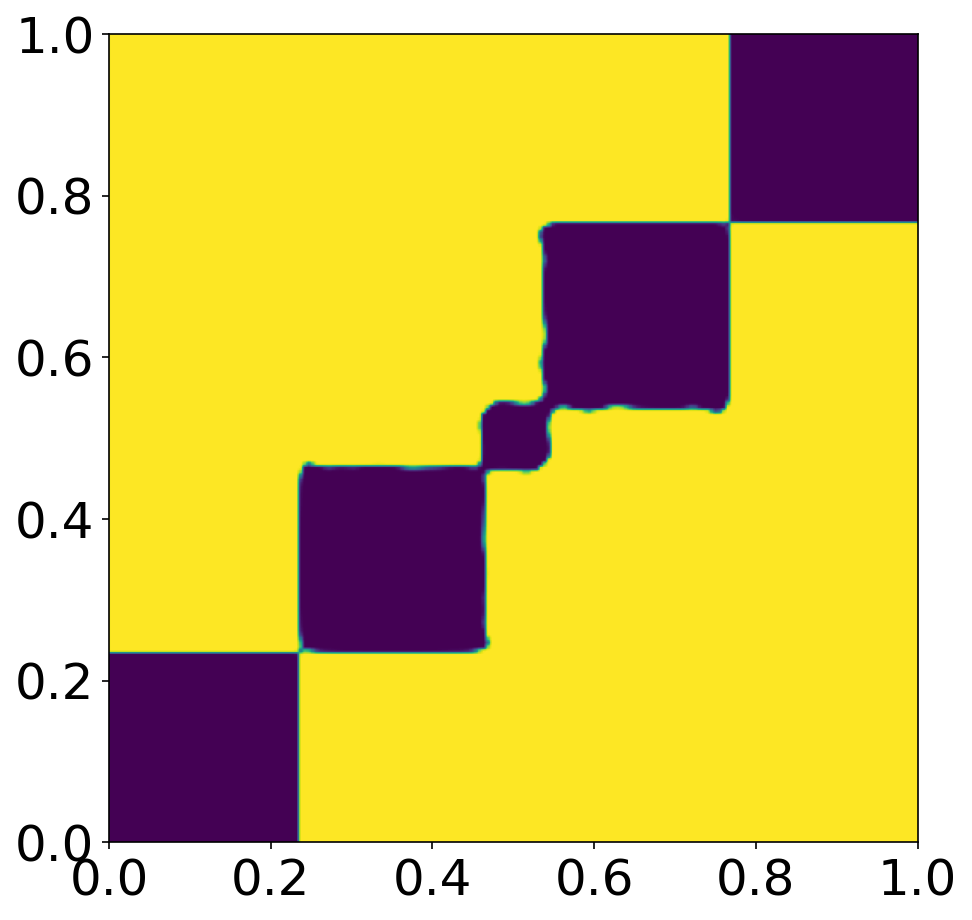}
    \caption{Our method.}
    \label{fig:baseline}
  \end{subfigure} \hfill
  \begin{subfigure}{0.25\textwidth}
    \centering
    \includegraphics[width=\textwidth]{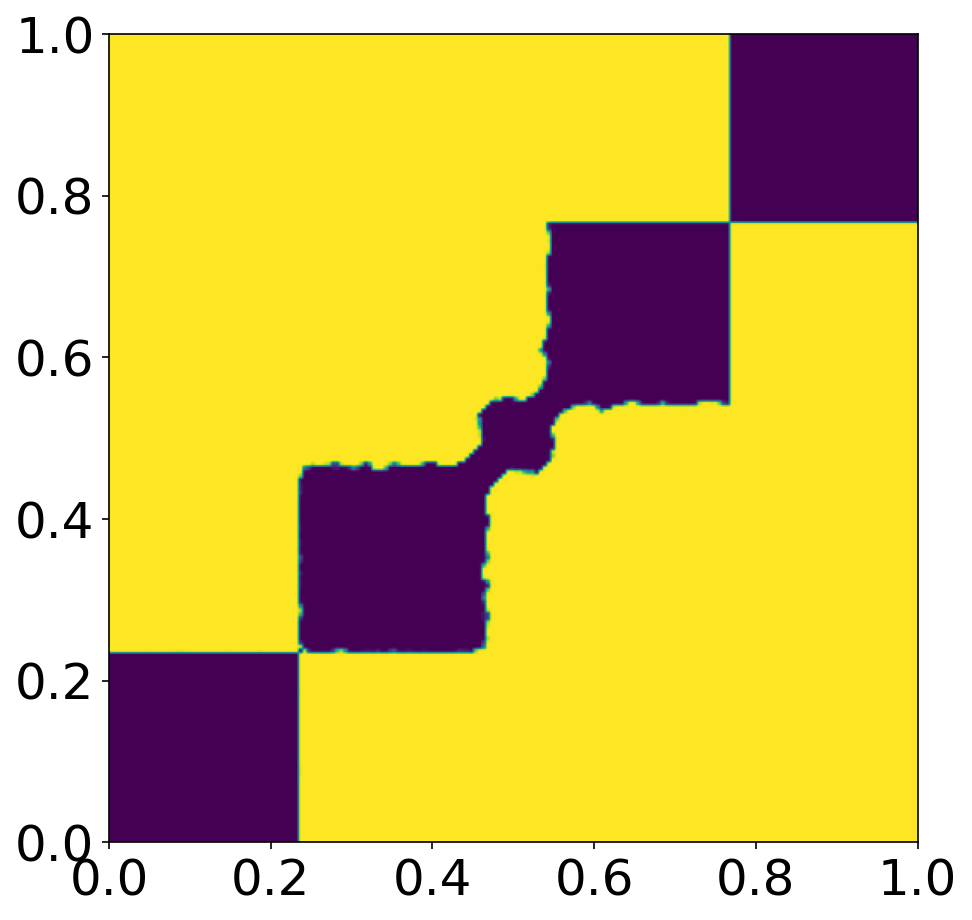}
    \caption{ResNet variant.}
    \label{fig:resnet}
  \end{subfigure} \hfill
  \begin{subfigure}{0.25\textwidth}
    \centering
    \includegraphics[width=\textwidth]{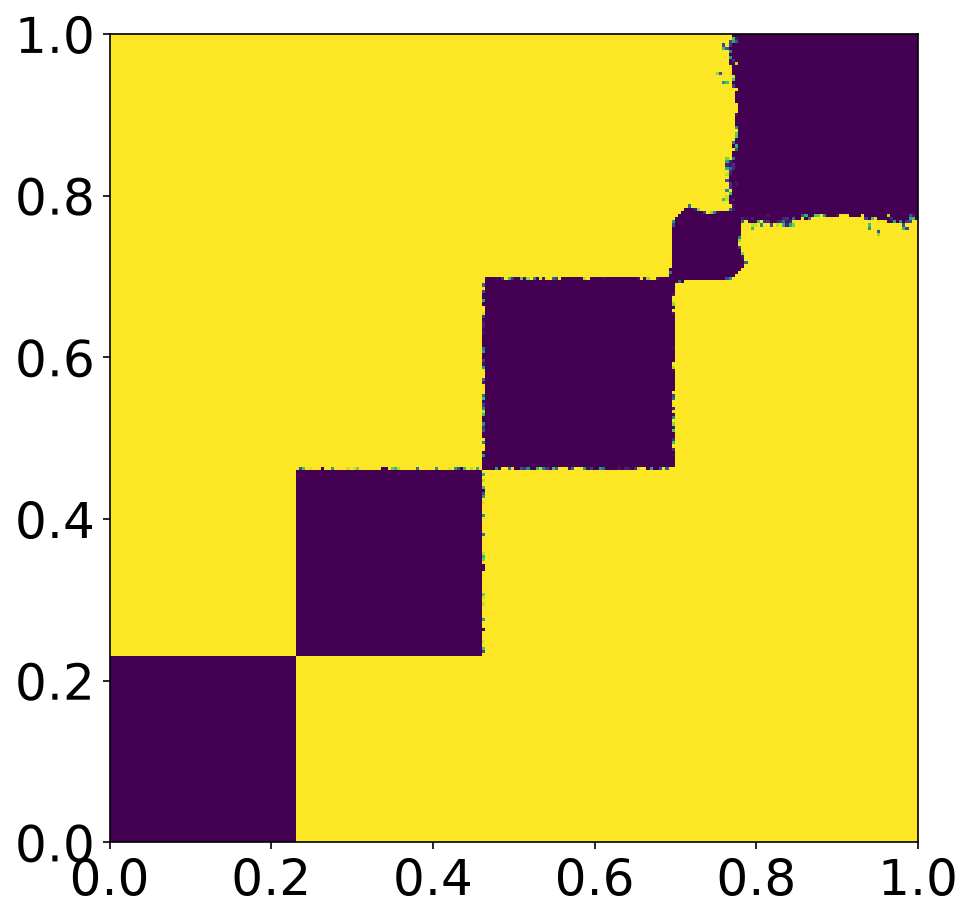}
    \caption{Constant-scale~$s$.}
    \label{fig:const2}
  \end{subfigure}

  \begin{subfigure}{0.25\textwidth}
    \centering
    \includegraphics[width=\textwidth]{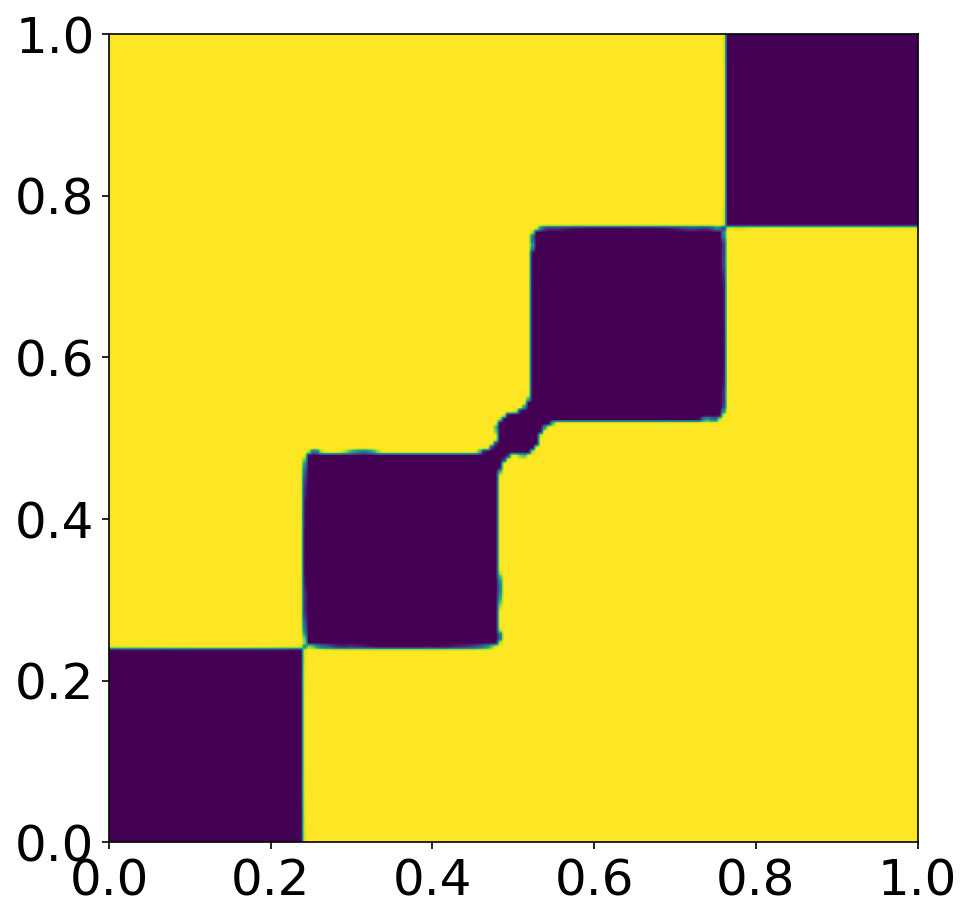}
    \caption{Regulariser.}
    \label{fig:noconsthighreg}
  \end{subfigure} \hfill
  \begin{subfigure}{0.25\textwidth}
    \centering
    \includegraphics[width=\textwidth]{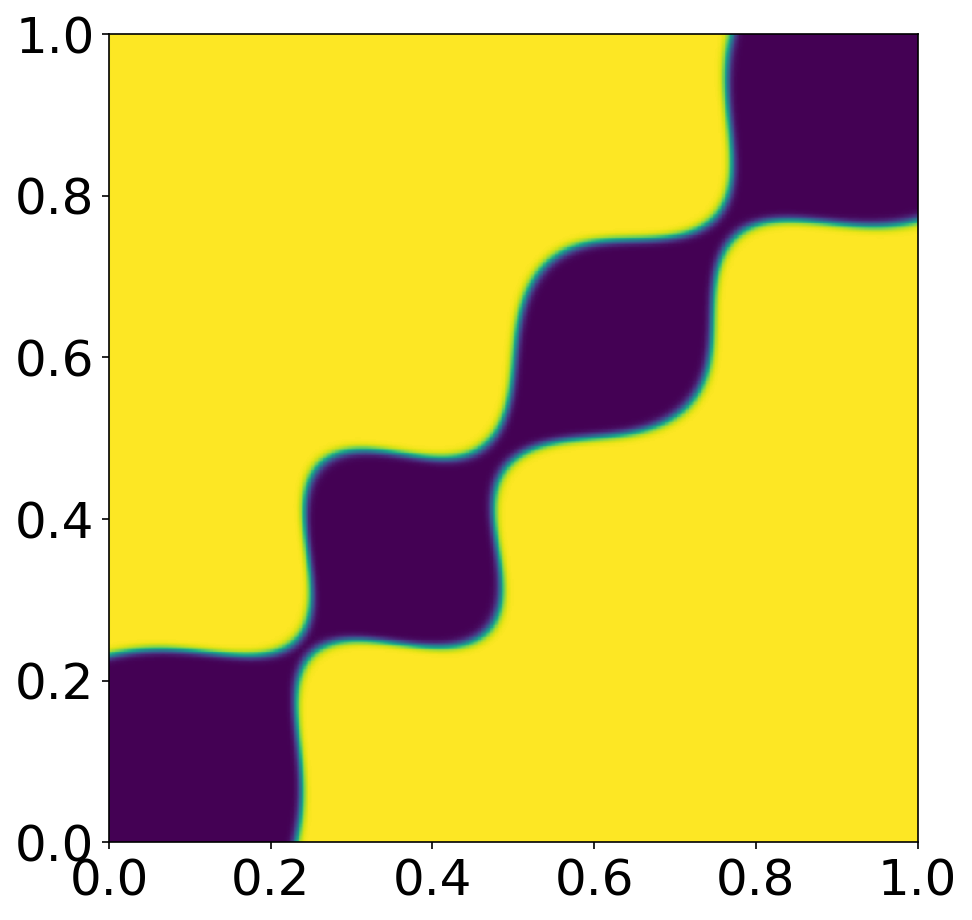}
    \caption{Const.~LR (low).}
    \label{fig:noscheduling-lowlr}
  \end{subfigure} \hfill
  \begin{subfigure}{0.25\textwidth}
    \centering
    \includegraphics[width=\textwidth]{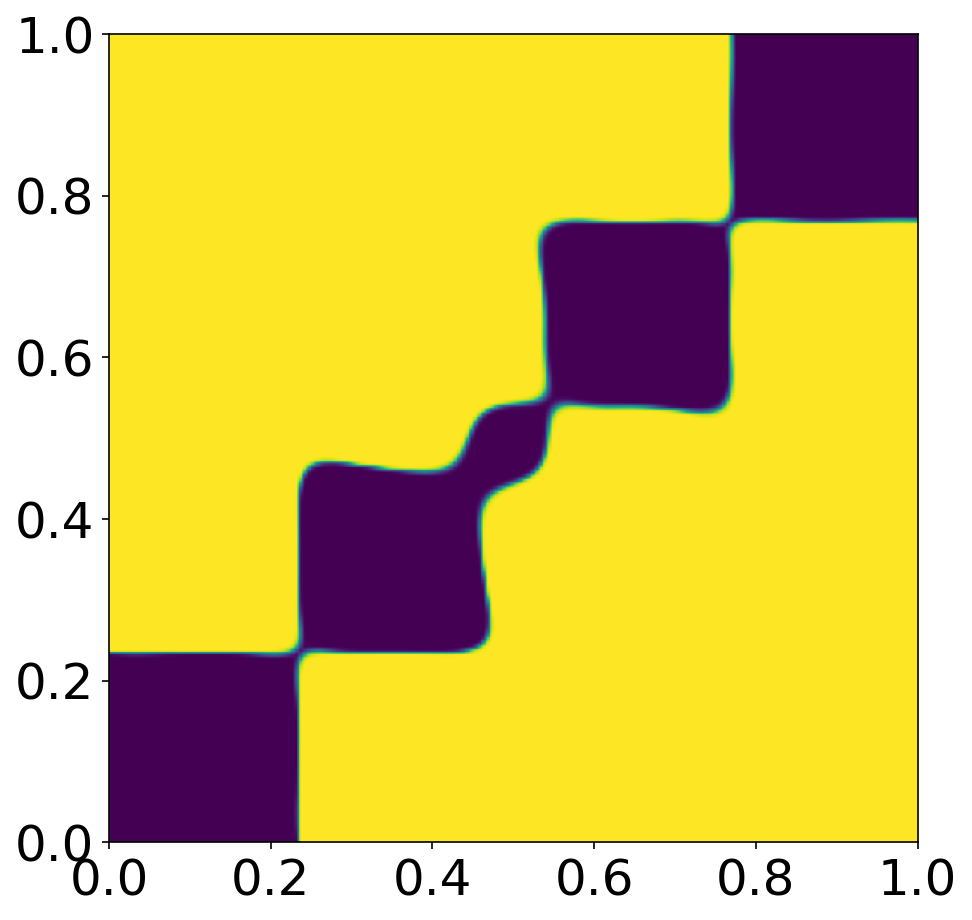}
    \caption{Const.~LR (high).}
    \label{fig:noscheduling-highlr}
  \end{subfigure}
  \caption{Graphons found by our method and its ablations, for $K_3$ density minimisation under the constraint $t(K_2, W) = 7/9$.}
  \label{fig:ablationstudy}
\end{figure}

\paragraph{Multi-scale encoding.}
The increasing scale $s(\ell)$ is motivated by the step-like structure of the optima: approximating a step boundary requires increasingly high spatial frequencies. Early layers, whose encodings have low frequencies, provide the coarse block geometry, while later layers, whose encodings have higher frequencies, refine the boundaries through the residual connections, so that earlier layers do not need to carry high-frequency information. The scale enters only through the initialisation of $U^{(\ell)}$, whose entries are drawn from $\mathrm{Unif}[-s(\ell)/\sqrt{2}, s(\ell)/\sqrt{2}]$, so $s(\ell)$ sets the initial spatial frequency of the encoding $\sin(U^{(\ell)}\mathbf{x})$ in layer $\ell$. Since $\sin$ is bounded and the remaining weights are initialised with scale $1/\sqrt{2d}$, increasing $s(\ell)$ gives later layers access to finer scales without increasing the magnitude of the hidden states. The constant-scale variant (2) above isolates the effect of increasing the scale. To see the contribution of each layer, \Cref{fig:multiscale} applies the output layer to the hidden state $g^{(\ell)}_\theta$ after each layer $\ell$ of a trained network. The first layer gives only a smooth gradient, the next layers form a band along the diagonal and split it into blocks, and the last layers sharpen the block boundaries into the five-block structure of the learnt graphon.

\begin{figure}[t]
  \centering
  \begin{subfigure}{0.155\textwidth}
    \centering
    \includegraphics[width=\linewidth]{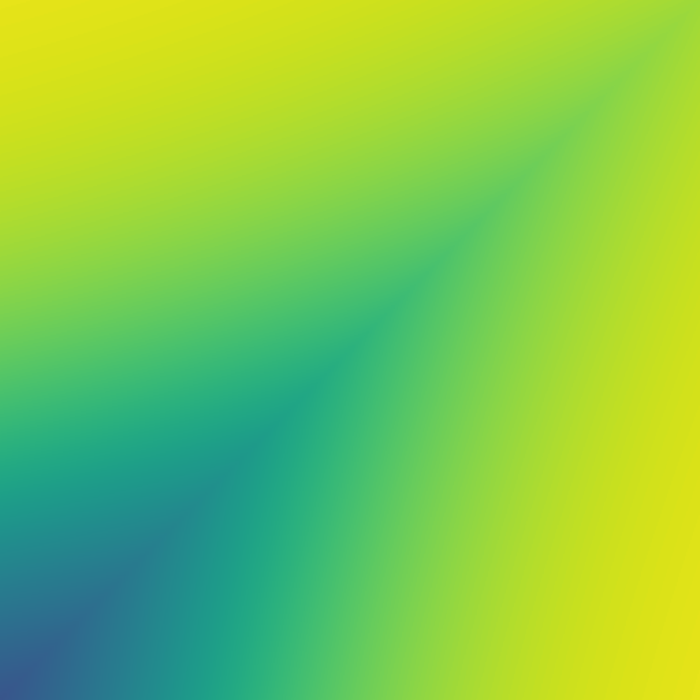}
  \end{subfigure} \hfill
  \begin{subfigure}{0.155\textwidth}
    \centering
    \includegraphics[width=\linewidth]{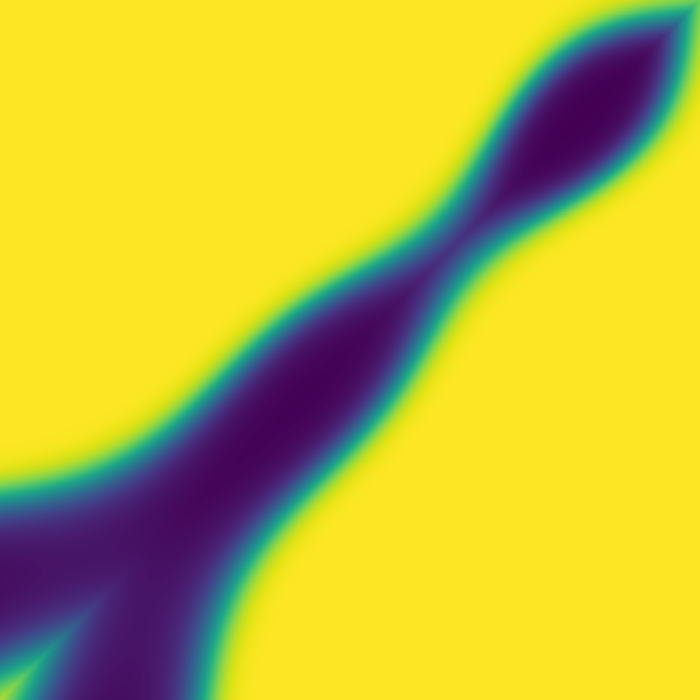}
  \end{subfigure} \hfill
  \begin{subfigure}{0.155\textwidth}
    \centering
    \includegraphics[width=\linewidth]{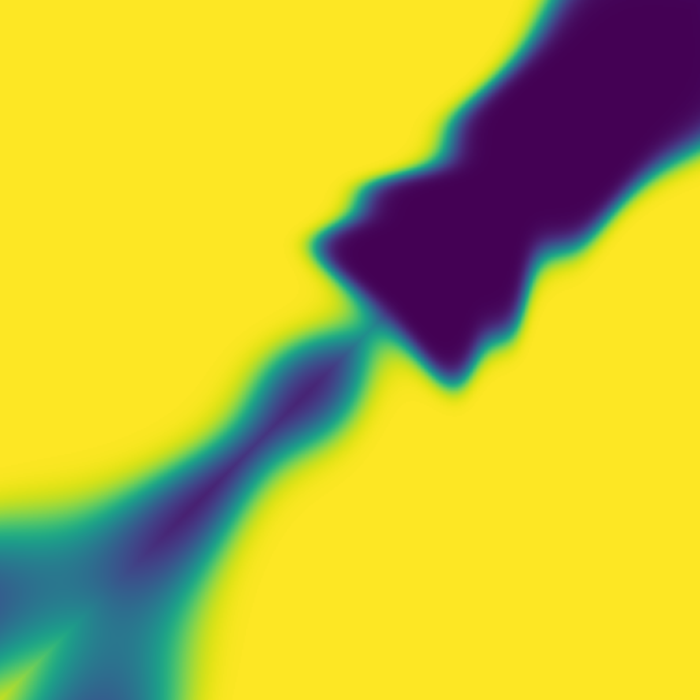}
  \end{subfigure} \hfill
  \begin{subfigure}{0.155\textwidth}
    \centering
    \includegraphics[width=\linewidth]{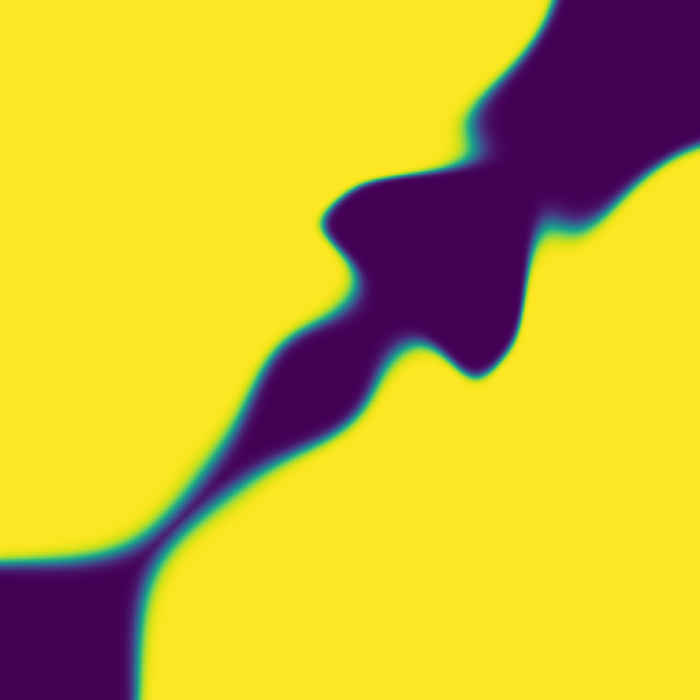}
  \end{subfigure} \hfill
  \begin{subfigure}{0.155\textwidth}
    \centering
    \includegraphics[width=\linewidth]{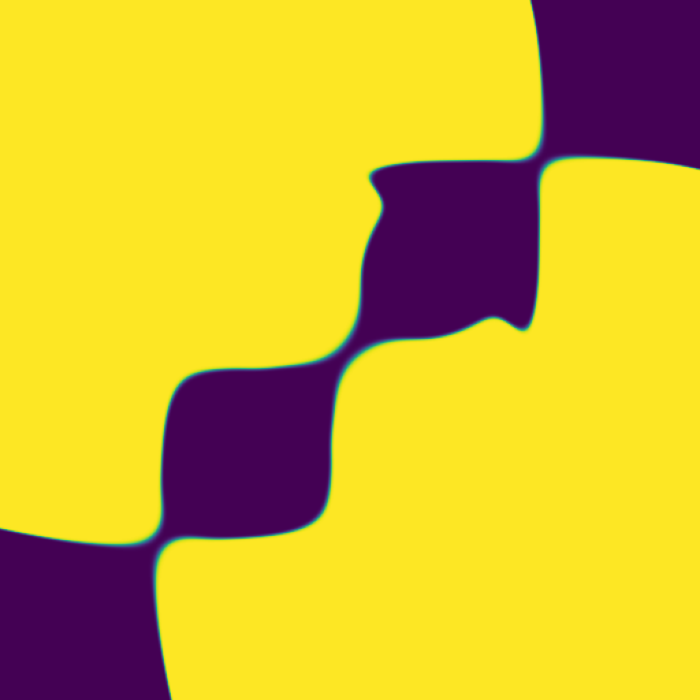}
  \end{subfigure} \hfill
  \begin{subfigure}{0.155\textwidth}
    \centering
    \includegraphics[width=\linewidth]{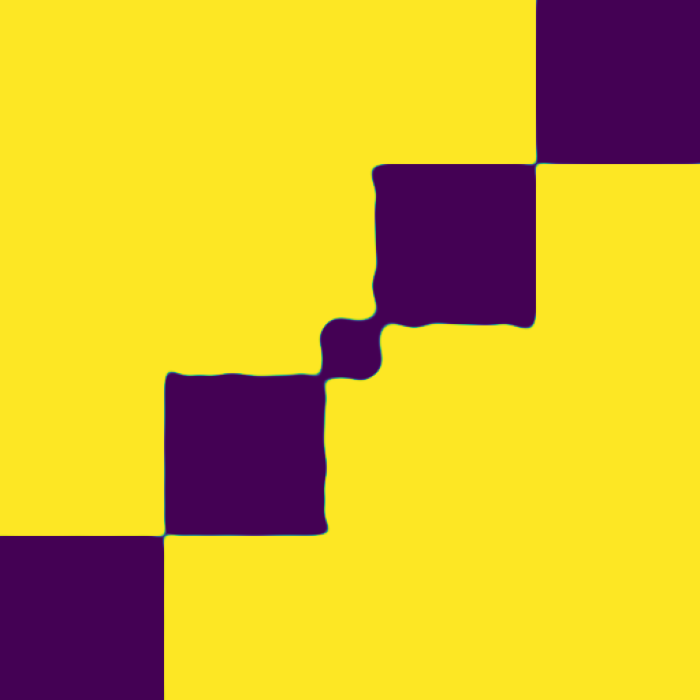}
  \end{subfigure}
  \caption{Graphons obtained by applying the output layer to the hidden state $g^{(\ell)}_\theta$ after layer $\ell$ of a trained network, for $K_3$ density minimisation under $t(K_2, W) = 7/9$, with the scalar bias $c$ solved separately for each $\ell$ so that $t(K_2, W) = 7/9$; $\ell = 1, \ldots, 6$ from left to right, and $\ell = 6$ is the learnt graphon.}
  \label{fig:multiscale}
\end{figure}

\paragraph{Robustness.}
Except for the exhaustive study, the results in the main text are the best of eight training runs over the activation and the learning rate (\Cref{sec:experiments}). To measure the variability across random initialisations separately, we repeated five instances ten times each under a single fixed configuration (GELU activation, learning rate $10^{-3}$), varying only the random initialisation, and report the mean with a $95\%$ confidence interval:
\begin{itemize}
  \item $K_3$ at $p = 7/9$: $0.435661$, $[0.435656, 0.435667]$, against the optimum $98/225 \approx 0.435556$.
  \item $C_5$ at $p = 3/5$: $0.066731$, $[0.066706, 0.066756]$, against the conjectured optimum $0.066616$.
  \item $C_7$ at $p = 3/5$: $0.026467$, $[0.026134, 0.026800]$, against our conjectured optimum $0.026261$. One run converged to the constant graphon, with density $p^7$; the other nine give $0.026297$, $[0.026279, 0.026315]$.
  \item $H_6$ at $p = 4/5$: $0.129877$, $[0.129851, 0.129902]$, against $0.13056$ for the $k$-partite graphon and $0.134218$ for the constant graphon.
  \item The upper-tail problem for $K_3$ at $(q, r) = (0.05, 0.5)$: $0.727215$, $[0.727210, 0.727220]$, against the lower bound $0.725284$ (\Cref{tab:largedeviation}).
\end{itemize}
The learnt structures are also consistent across seeds. For $K_3$, seven runs recover the clean five-block partition, and the other three produce five parts whose smallest part is not an exact block, with comparable objective values; since the optimum is not unique, this does not indicate suboptimality. For $C_5$, all ten runs recover the two-block structure of the bottom row of \Cref{fig:5cycleminimisation}, and for $C_7$, the nine nonconstant runs recover that of \Cref{fig:7cycleminimisation}. For $H_6$, eight runs recover the pattern of \Cref{fig:noveloptimum}, and the other two show the same band with an oscillating boundary, at objectives $0.04\%$ and $0.10\%$ above the average of the eight. For the upper-tail problem, all ten runs recover the two-block structure of \Cref{fig:k3upperdeviation}.
Across the 702 instances of the exhaustive study in \Cref{sec:experiments}, the three seeds have a median relative standard deviation of $1.4 \times 10^{-4}$ and differ by more than $1\%$ in $5.8\%$ of the instances.

We also checked the sensitivity to the other sources of randomness. For $C_5$, training with permutation sets $\mathcal{S}$ of size $1$, $2$, and $12$ in~\eqref{eqn:estimator} gives objectives that agree within $1.6 \times 10^{-4}$ at each of the six densities of \Cref{tab:c5result}. Finally, for the learnt graphon for $K_3$ at $p = 7/9$, solving for the scalar bias $c$ on $1000$ independent batches of $2^{16}$ samples, as in training, gives a constraint violation with standard deviation $1.1 \times 10^{-3}$ and maximum $3.4 \times 10^{-3}$; this error does not affect the reported values, which solve for $c$ on the discretised graphon, and refining the discretisation from $256^2$ to $512^2$ points changes $t(K_3, W)$ by $9.7 \times 10^{-5}$.

\paragraph{Non-neural Baselines.}
On the same triangle-density-minimisation task we evaluated four non-neural baselines tailored to step graphons: (a) a stochastic block model (SBM) with $256$ equal-sized blocks and trainable edge densities; (b) SBMs with $5$ blocks and $64$ blocks, with trainable edge densities and trainable block sizes; (c) a decision tree; and (d) a tree ensemble.
Each baseline was run four times from independent random initialisations, and we report the best objective value across runs in \Cref{tab:ablationstudy}.
The SBMs with trainable block sizes collapsed to nearly constant graphons, suggesting a difficult optimisation landscape; the fixed-block SBM recovered a noisy multi-block structure; the tree-based baselines produced structured but qualitatively different graphons. See \Cref{fig:baselines} for visualisations of the found graphons.

These baselines can represent the optimal graphon, at least approximately; the difference lies in optimisation. For the fixed-block SBM and the tree-based baselines, moving a block boundary is hard, since almost every configuration with values in $\{0, 1\}$ is a local optimum, so these methods rarely leave the geometry of their initialisation. Making the block sizes trainable does not help, as these SBMs collapse to nearly constant graphons. In contrast, the neural graphon adjusts the values and the boundaries of its regions jointly and continuously, and it does not require choosing in advance between multipartite, banded, circular-distance, or two-block structures.

\begin{figure}
  \centering
  \begin{subfigure}{0.19\textwidth}
    \centering
    \includegraphics[width=\textwidth]{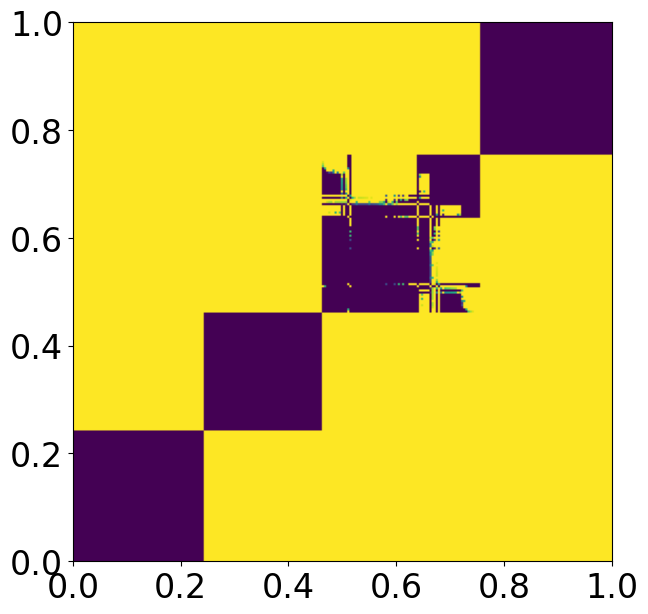}
  \end{subfigure} \hfill
  \begin{subfigure}{0.19\textwidth}
    \centering
    \includegraphics[width=\textwidth]{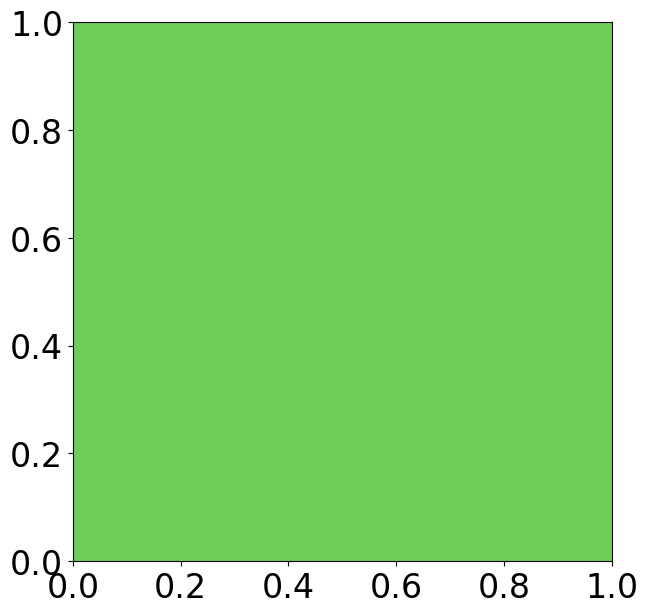}
  \end{subfigure} \hfill
  \begin{subfigure}{0.19\textwidth}
    \centering
    \includegraphics[width=\textwidth]{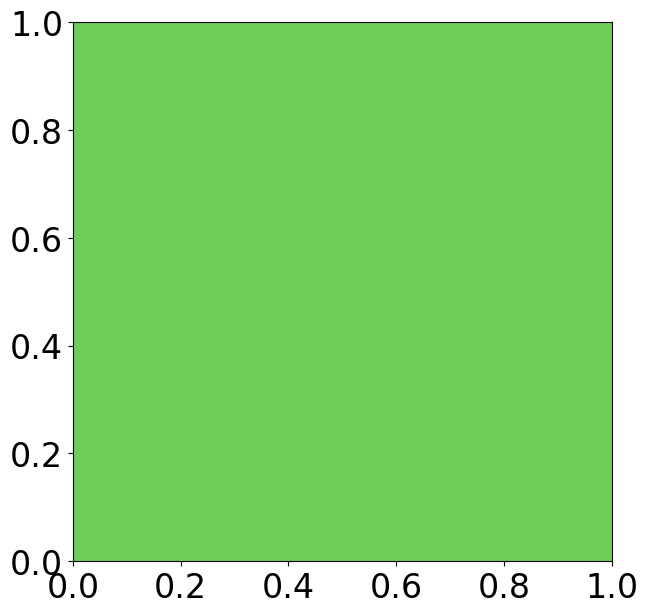}
  \end{subfigure} \hfill
  \begin{subfigure}{0.19\textwidth}
    \centering
    \includegraphics[width=\textwidth]{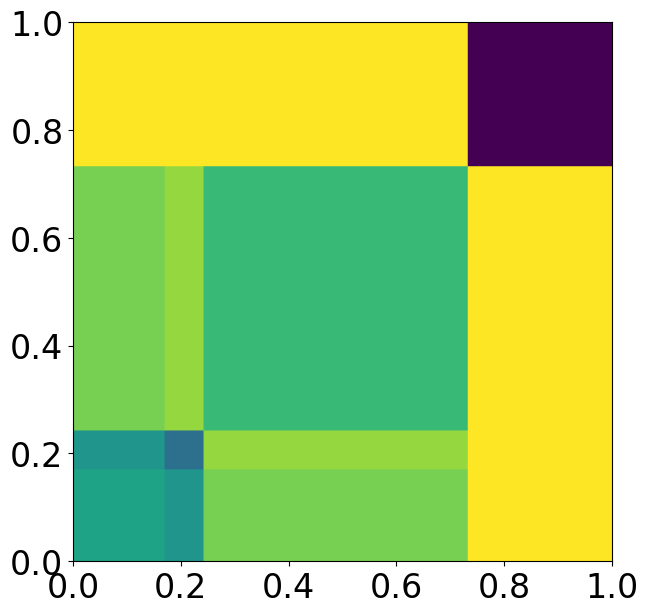}
  \end{subfigure} \hfill
  \begin{subfigure}{0.19\textwidth}
    \centering
    \includegraphics[width=\textwidth]{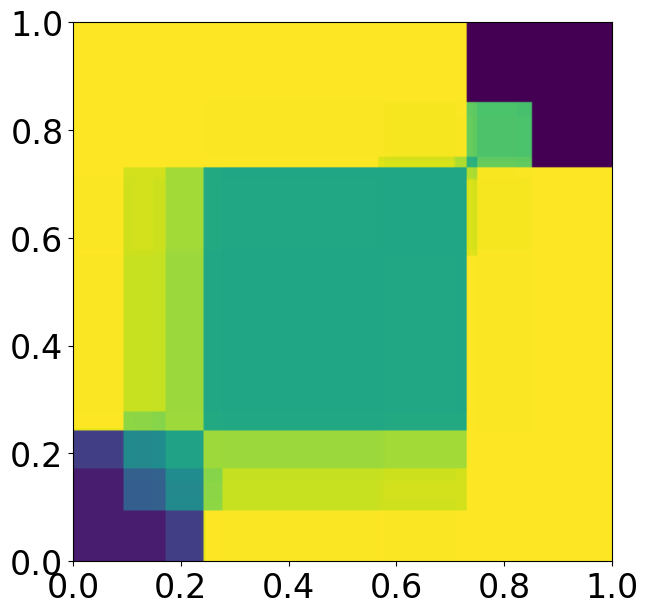}
  \end{subfigure}
  \caption{Graphons found by non-neural baselines, fixed-block SBM, trainable-block SBM with 5 / 64 blocks, decision tree, and tree ensemble (left to right), for $K_3$ density minimisation under the constraint $t(K_2, W) = 7/9$.}
  \label{fig:baselines}
\end{figure}

\paragraph{Multi-task Variant.}
The main paper trains one network per $(H, p)$. To probe whether shared training across related tasks is viable, we trained a single conditional network jointly over $p \in [1/2, 2/3]$ for $H = K_3$. Concretely, we condition the network on $p$ using FiLM-style feature-wise affine modulation of hidden activations~\citep{perez2018film}. The resulting conditional network achieves absolute error at most $2.09 \times 10^{-4}$ and average absolute error $3.12 \times 10^{-5}$ relative to the known optimum, for $p$ on a uniform grid of 100 points in $[1/2, 2/3]$. See \Cref{fig:parameterised_density} for the learnt density curve and its error, and \Cref{fig:parameterised_graphons} for the learnt graphons.
While limited to a single graph, this experiment indicates that shared training across related tasks can substantially reduce the per-instance burden of our framework, and we view a fuller multi-task / multi-graph treatment as a natural extension.

\begin{figure}
  \centering
  \begin{subfigure}{0.45\textwidth}
    \centering
    \includegraphics[width=\textwidth]{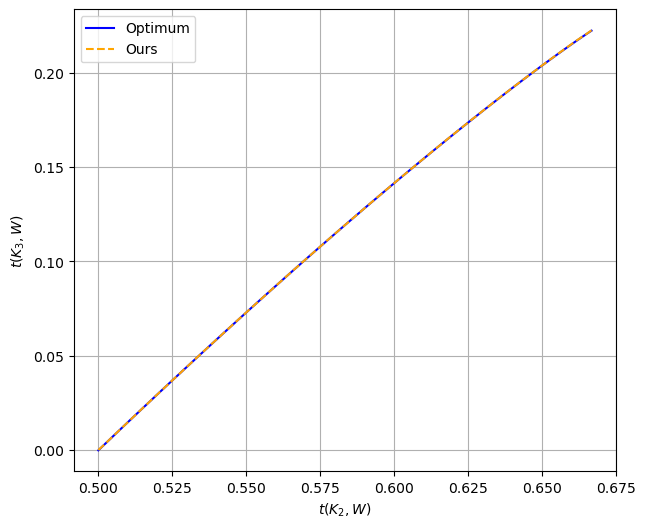}
  \end{subfigure} \hfill
  \begin{subfigure}{0.45\textwidth}
    \centering
    \includegraphics[width=\textwidth]{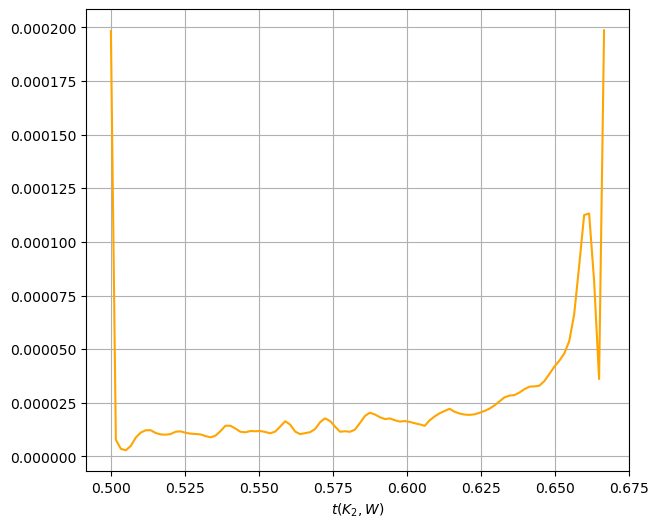}
  \end{subfigure}
  \caption{(Left) Triangle density as a function of $p$ for the conditional network, compared to the known optimum. (Right) Absolute error of the conditional network's triangle density relative to the known optimum.}
  \label{fig:parameterised_density}
\end{figure}

\begin{figure}
  \centering
  \begin{subfigure}{0.19\textwidth}
    \centering
    \includegraphics[width=\textwidth]{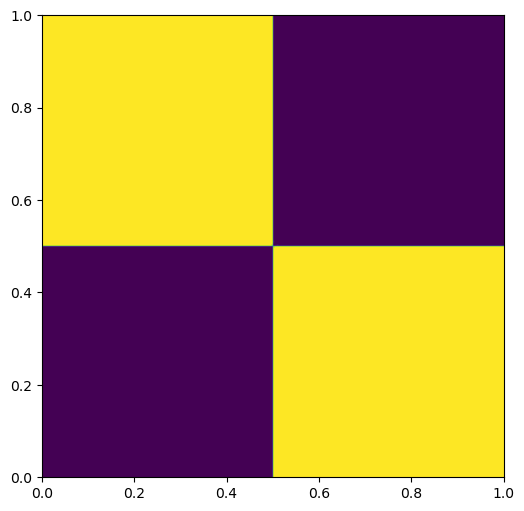}
  \end{subfigure} \hfill
  \begin{subfigure}{0.19\textwidth}
    \centering
    \includegraphics[width=\textwidth]{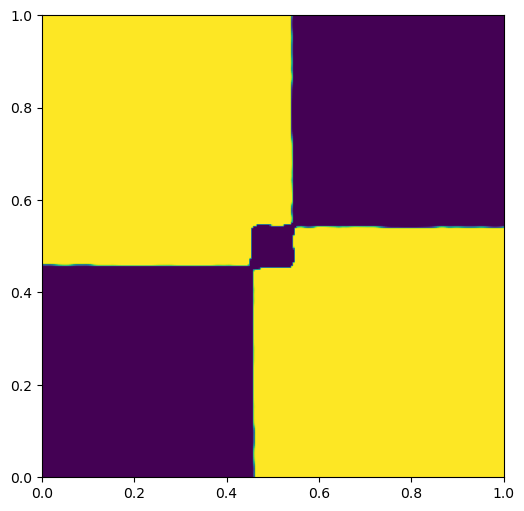}
  \end{subfigure} \hfill
  \begin{subfigure}{0.19\textwidth}
    \centering
    \includegraphics[width=\textwidth]{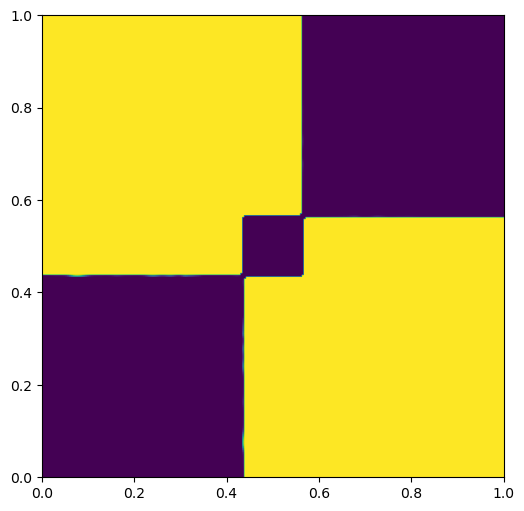}
  \end{subfigure} \hfill
  \begin{subfigure}{0.19\textwidth}
    \centering
    \includegraphics[width=\textwidth]{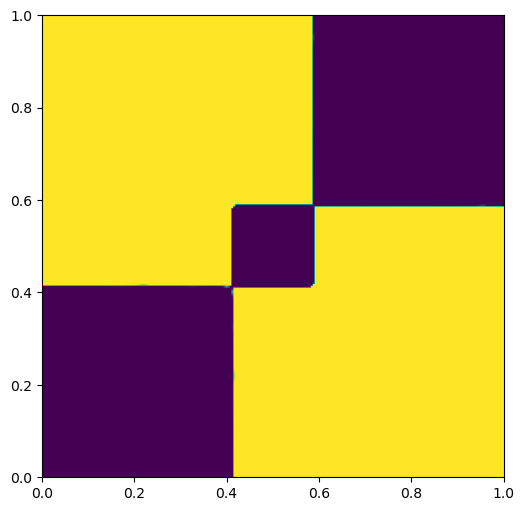}
  \end{subfigure} \hfill
  \begin{subfigure}{0.19\textwidth}
    \centering
    \includegraphics[width=\textwidth]{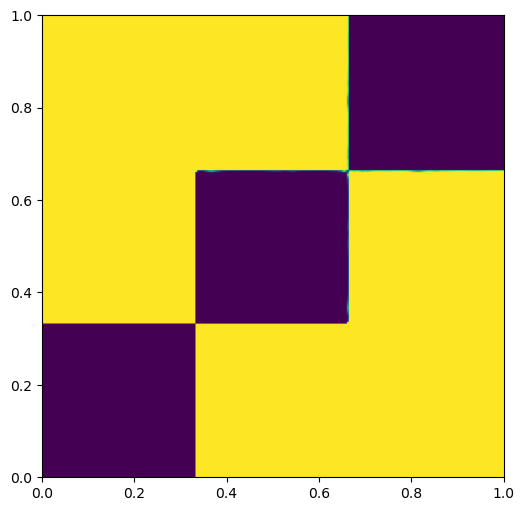}
  \end{subfigure}
  \caption{Learnt graphons for the conditional network at $p \in \{1/2, 4/7, 3/5, 5/8, 2/3\}$ (left to right).}
  \label{fig:parameterised_graphons}
\end{figure}

\end{document}